\documentclass{article}

\usepackage[preprint]{neurips_2026}

\usepackage[utf8]{inputenc}
\usepackage[T1]{fontenc}
\usepackage{hyperref}
\usepackage{url}
\usepackage{booktabs}
\usepackage{amsfonts}
\usepackage{amsmath,amssymb,amsthm,mathtools,bm,bbm}
\usepackage{microtype}
\usepackage{xcolor}
\usepackage{paralist}
\usepackage{enumitem}
\usepackage{algorithm}
\usepackage[noend]{algpseudocode}
\usepackage{multirow,array}

\usepackage{comment}
\usepackage{caption}
\usepackage{wrapfig}
\usepackage{minitoc}
\usepackage{longtable}
\usepackage{graphicx} 
\usepackage{float}    

\hypersetup{colorlinks=true,linkcolor=blue,citecolor=blue,urlcolor=blue}

\newtheorem{theorem}{Theorem}[section]
\newtheorem{proposition}[theorem]{Proposition}
\newtheorem{corollary}[theorem]{Corollary}
\newtheorem{lemma}[theorem]{Lemma}
\newtheorem{remark}[theorem]{Remark}
\newtheorem{assumption}{Assumption}[section]

\newcommand{\R}{\mathbb{R}}
\newcommand{\E}{\mathbb{E}}
\newcommand{\KL}{\mathrm{KL}}
\newcommand{\kl}{\mathrm{KL}}
\newcommand{\TV}{\mathrm{TV}}

\newcommand{\X}{\mathcal{X}}
\newcommand{\D}{\mathcal{D}}
\newcommand{\U}{\mathcal{U}}

\newcommand{\Pbb}{\mathbb{P}}
\newcommand{\F}{\mathcal{F}}

\newcommand{\dd}{\mathrm{d}}

\newcommand{\argmin}{\mathop{\mathrm{argmin}}}
\newcommand{\argmax}{\mathop{\mathrm{argmax}}}
\newcommand{\argtopk}{\mathop{\mathrm{argtopK}}}

\newcommand{\prior}{p_0}

\newcommand{\Xstar}{\X^{\star}}

\newcommand{\good}{\mathsf{G}}

\DeclareMathOperator{\supp}{supp}

\DeclareMathOperator{\Pdim}{Pdim}

\newcommand{\dist}{\operatorname{dist}}

\title{Regret Analysis of Guided Diffusion for\\Black-Box Optimization over Structured Inputs}
\author{
    Masaki Adachi, Anita Yang\\
    Lattice Lab\\
    Toyota Motor Corporation\\
    \And
    Yakun Wang, Song Liu\\
    School of Mathematics\\
    University of Bristol\\
}

\begin{document}
\maketitle

\begin{abstract}
Guided-diffusion black-box optimization (BO) has shown strong empirical performance on structured design problems such as molecules and crystals, but its regret behavior remains poorly understood. Existing BO regret analyses typically rely on maximum information gain, non-pretrained surrogate models, or exact acquisition maximization—assumptions that break down in modern diffusion-BO pipelines, where pretrained diffusion models serve as powerful priors over valid structures and acquisition maximization is replaced by approximate sampling over astronomically large discrete spaces.
We develop a \emph{first} certificate-based expected simple-regret framework for guided-diffusion BO that avoids maximum-information-gain bounds, RKHS assumptions, and exact acquisition maximization. The central quantity in our analysis is \emph{mass lift}: the increase in probability mass assigned to $\gamma$-optimal designs relative to the pretrained generator.  This view explains how exponential-looking finite-budget convergence and polynomial acceleration can all arise from the same mechanism. We also give practical diagnostics for estimating search exponents from finite candidate pools and a proposal-corrected resampling construction that provides a fully certified sampler instance.
\end{abstract}

\doparttoc 
\faketableofcontents 
\section{Introduction}\label{sec:intro}
\begin{wrapfigure}[15]{r}{0.45\textwidth}
    \vspace{-2em}
    \centering
    \includegraphics[width=0.45\textwidth]{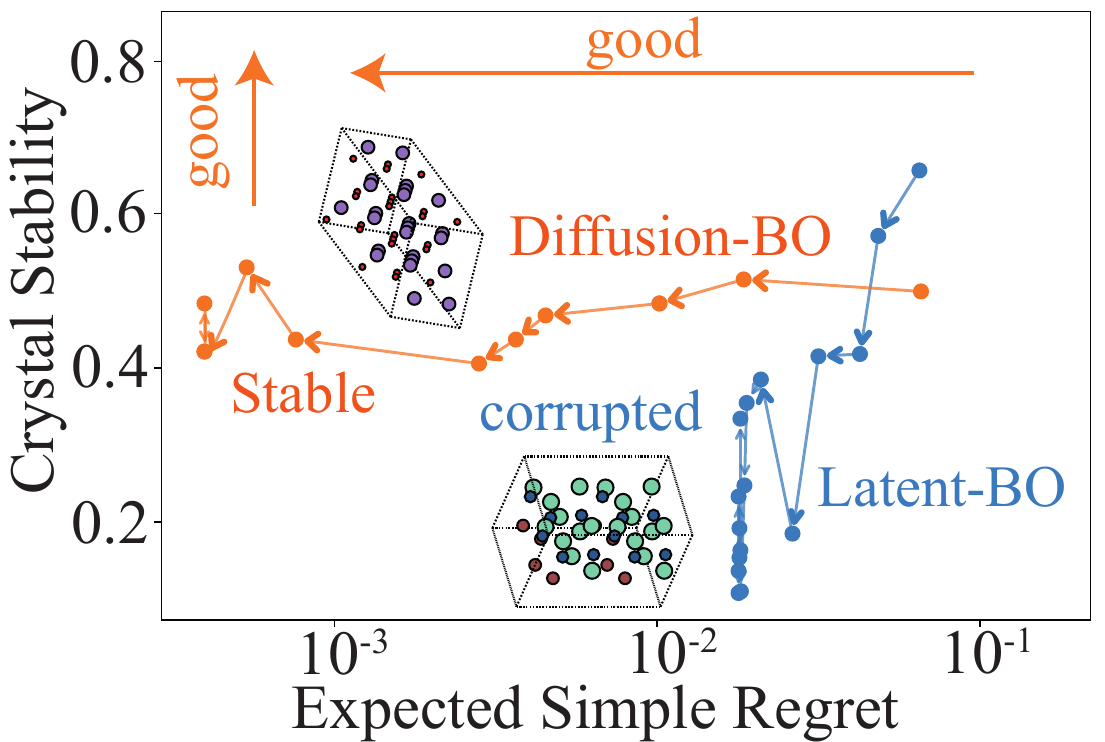}
    \caption{Diffusion BO produces stable, low-regret crystals during optimization, whereas Latent BO yields corrupted ones.}
    \label{fig:vae}
\end{wrapfigure}
\textbf{Structured-input optimization.}
Black-box optimization (BO~\citep{garnett2023bayesian}) is a natural tool for expensive scientific search. 
In structured domains such as molecules~\citep{gomez2018automatic,vignac2023digress}, proteins~\citep{gruver2023protein}, crystals~\citep{zeni2025mattergen}, and neural architectures~\citep{white2021bananas, ru2021interpretable}, however, the main bottleneck is often not surrogate fitting but \emph{acquisition maximization}: selecting the next query requires optimizing over a massive combinatorial space (e.g., molecular space is estimated to be of size $10^{60}$~\citep{reymond2015chemical}).
A common workaround is \emph{latent BO}, which embeds structured objects into a continuous space and performs optimization therein~\citep{gomez2018automatic, grosnit2021high, moss2025return}. This approach relies on approximately invertible encoder--decoder pairs, which can be fragile when extrapolating beyond the training distribution, often producing corrupted structures~\citep{maus2022local, lee2023advancing, santi2026verifierconstrained} (Figure~\ref{fig:vae}). Diffusion models offer a powerful alternative generative model~\citep{jiao2023crystal, zeni2025mattergen, okabe2025structural}, but their lack of an explicit latent space precludes directly applying latent BO.

\begin{figure}
    \vspace{-1em}
    \centering
    \includegraphics[width=1\linewidth]{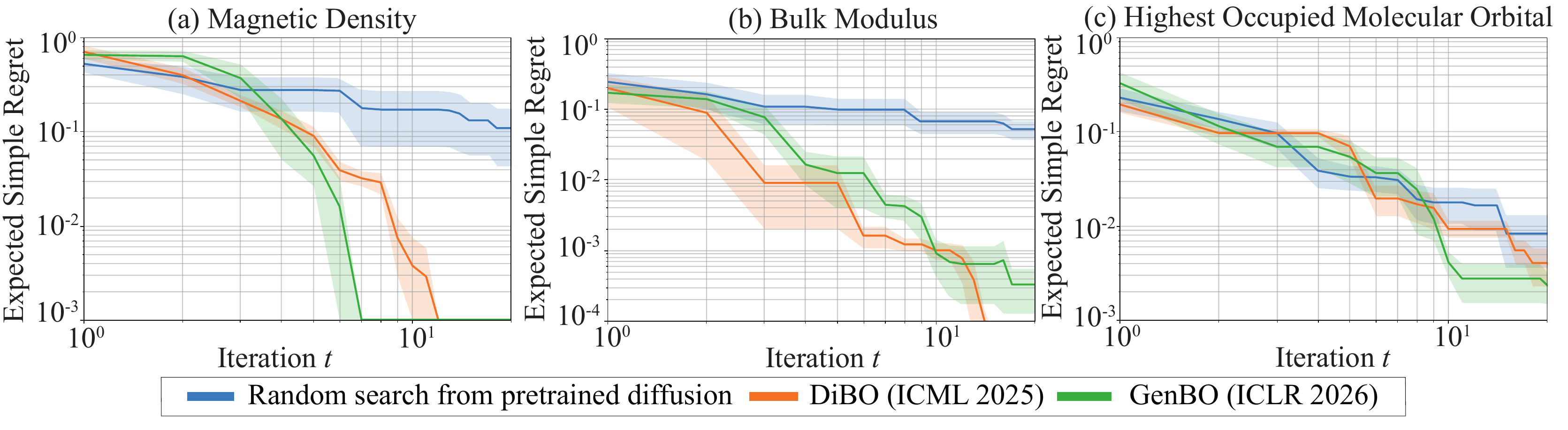}
    \caption{Log-log plots of regret of two crystal and one molecular optimization tasks illustrating three acceleration types: (a) exponential, (b) polynomial, and (c) constant-factor acceleration.}
    \label{fig:motivation}
\end{figure}

\textbf{Diffusion BO.}
Guided-Diffusion BO (GDBO~\citep{yun2025posterior, oliveira2025generative, sun2024fast}) recasts acquisition maximization as sampling: it fine-tunes a pretrained diffusion model so that samples concentrate in high-acquisition regions. This has shown strong performance in scientific discovery~\citep{chen2025matinvent,park2025guiding}, yet its optimization behavior is largely unexplained.
A central obstacle is that classical regret analyses rely on maximum information gain (MIG) and exact acquisition maximization~\citep{srinivas2009gaussian, adachi2025bayesian}. In GDBO, however, the diffusion model is a generative prior over inputs $p_0(x)$, the sampler targets an acquisition-tilted distribution $p_a(x\mid y)$\footnote{A Gibbs posterior using the acquisition function $a_t(x)$ as a reward or negative loss.}, and a Bayesian surrogate $p(y\mid x)$ is optional. Moreover, sampling from this guided distribution introduces approximation error. Existing analyses~\citep{yun2025posterior,oliveira2025generative,sun2024fast} show that a fine-tuned sampler can track a tilted distribution, but distribution tracking is not itself a regret guarantee.

\textbf{Three empirical regimes.}
Empirically, GDBO does not show a single form of acceleration. 
As shown in Figure~\ref{fig:motivation}, DiBO~\citep{yun2025posterior} and GenBO~\citep{oliveira2025generative} exhibit three distinct finite-budget regimes: exponential, polynomial, and only constant-factor improvement over random search.
The first exceeds what classical polynomial BO rates~\citep{srinivas2009gaussian, bull2011convergence, zhou2020neural} would predict; the second matches the usual expectation of BO acceleration but still requires explaining the source of the exponent gain; the third shows that guided diffusion does not automatically improve rates.
Since the same acquisition family can appear across these regimes, acquisition-function design alone cannot explain them.
A theory of GDBO must instead identify when guided sampling changes the convergence rate.

\begin{figure}
    \vspace{-1em}
    \centering
    \includegraphics[width=1\linewidth]{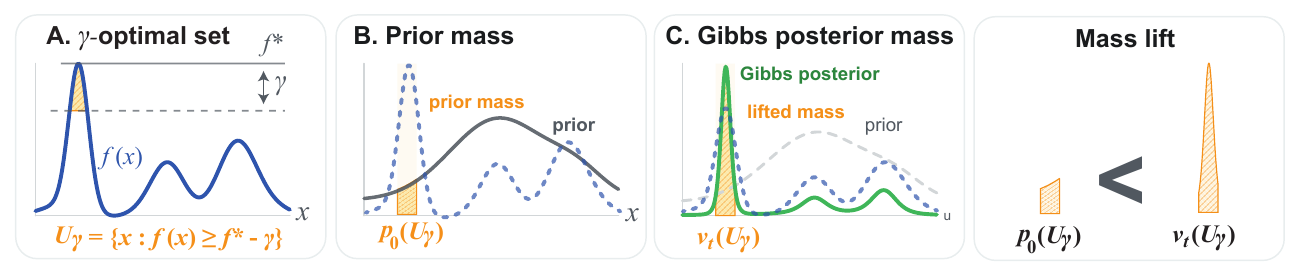}
    \caption{Concept of mass lift: posterior gain on the $\gamma$-optimal set over the prior.}
    \label{fig:mass_lift}
\end{figure}
\textbf{Our contributions.}
We provide the \emph{first} certificate-based expected simple-regret analysis for this class of GDBO algorithms without relying on MIG, RKHS assumptions, or exact acquisition maximization.
The key quantity is \emph{mass lift}: the increase in probability assigned to a \(\gamma\)-optimal set by the guided sampler relative to the pretrained prior.
This single mechanism explains the three regimes above: large threshold-set lift yields exponential-looking search acceleration, lift on local critical bands yields polynomial active-learning acceleration, and weak lift or dominant sampler/score errors yields only constant-factor improvement.
Technically, our analysis separates three roles that are entangled in existing methods: acquisition calibration, terminal sampler approximation, and finite-pool best-of-$N$ search. This separation gives a common interface for DiBO-style RTB samplers, GenBO-style weighted scoring objectives, and proposal-corrected resampling. Once an acquisition function is chosen, mass lift becomes a concrete design principle for generative BO: it determines the finite-budget search exponent through the probability of hitting the relevant threshold set.

\section{Setup and background}\label{sec:setup}
Let $\X:=\supp(p_0)$ be the accessible domain of a frozen pretrained generator $p_0$. We assume $\X$ is finite but potentially enormous. The objective $f:\X\to\R$ is expensive, and observations satisfy $Y=f(X)+\xi$ with independent noise $\xi$. We use expected simple regret~\footnote{Expectation is over all algorithmic randomness, e.g., diffusion sampling.}:
\begin{equation}
\label{eq:near_opt_set}
 r_T:=\E[f^\star-f(\widehat x_T)],
 \qquad
 f^\star:=\max_{x\in\X}f(x),
 \qquad
 \widehat x_T = \argmax_{x \in\D_T} \widehat{a}_T(x),
\end{equation}
where $\widehat{a}_T$ is the learned acquisition function. If $\lim_{T \to \infty} r_T = 0$, the algorithm is no-regret~\citep{srinivas2009gaussian} and asymptotically identifies a global optimizer in $\Xstar=\argmax_{x\in\X}f(x)$. 

\begin{wrapfigure}[12]{r}{0.54\textwidth}
  \vspace{-2.7em}
  \begin{minipage}{\linewidth}
    \begin{algorithm}[H]
    \caption{GDBO template}
    \label{alg:generic_gbo}
    \begin{algorithmic}[1]
    \For{$t=1,2,\ldots,T$}
        \State Train $\widehat a_t$ from the current data $\D_t$.
        \State Define the target $\widehat q_t(x)$.
        \State Estimate a terminal sampler $\nu_t \approx \widehat q_t$.
        \State Draw $N$ candidates $X_{t,1},\ldots,X_{t,N}\stackrel{\mathrm{i.i.d.}}{\sim}\nu_t$.
        \State $X_t^{(K)}=\argtopk_{x\in\{X_{t,1},\ldots,X_{t,N}\}}\widehat a_t(x)$
        \State Optionally $X_t^{(J)} := X_{t,1},\ldots,X_{t,J} \stackrel{\mathrm{i.i.d.}}{\sim} p_0$.
        \State Evaluate $X_t^{(K)} \cup X_t^{(J)}$ and update $\D_{t+1}$.
    \EndFor
    \State \Return $\widehat x_T = \argmax_{x \in\D_T} \widehat{a}_T(x)$.
    \end{algorithmic}
    \end{algorithm}
  \end{minipage}
\end{wrapfigure}

\subsection{Acquisition sampling}
Classical BO maximizes an acquisition $a_t(x)$, which is generally intractable over structured spaces. A common variational alternative~\citep{wild2023rigorous, zhang2024diffusion, oliveira2025generative, knoblauch2019generalized} is:
\[
q_t^\star
\in
\argmax_{q\in\Delta(\X)}
\left\{
\E_q\bigl[a_t(x)\bigr] - \frac{1}{\beta}\,\KL\bigl(q\,\|\,p_0\bigr)
\right\},
\]

\begin{wrapfigure}[18]{r}{0.45\textwidth}
    \vspace{-1.0em}
    \centering
    \includegraphics[width=0.45\textwidth]{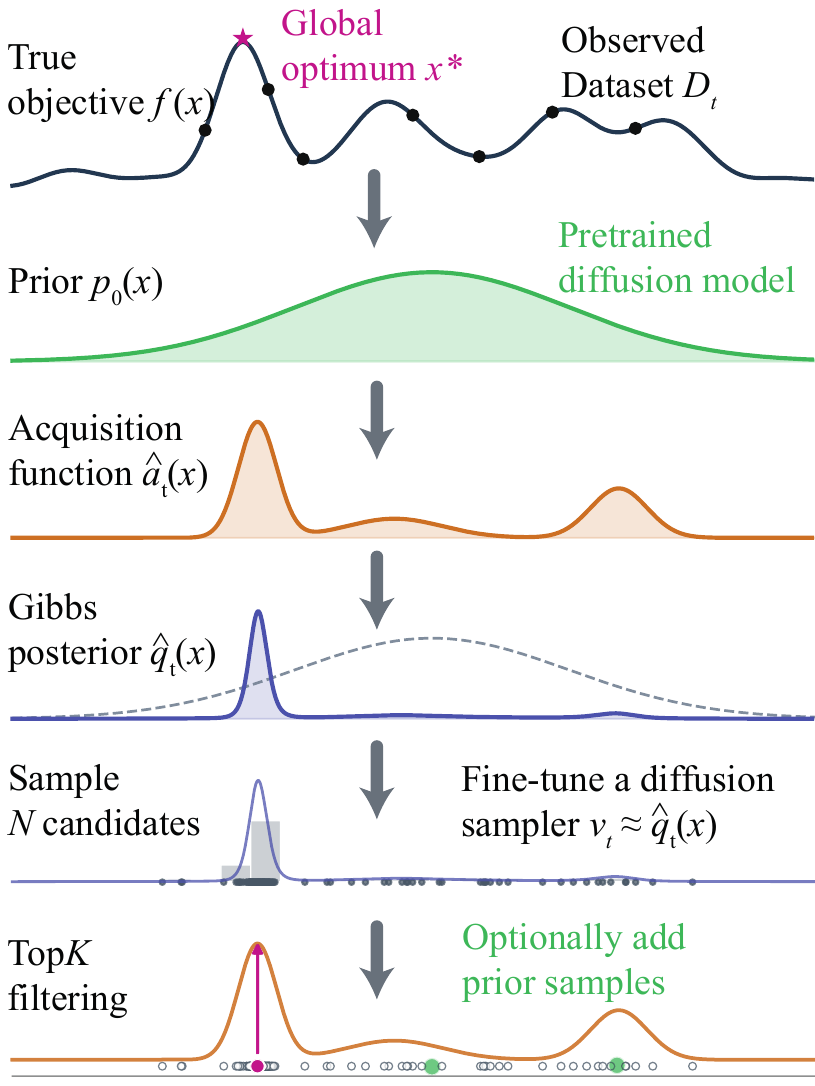}
    \caption{Visual explanation of Algorithm~\ref{alg:generic_gbo}.}
    \label{fig:gdbo}
\end{wrapfigure}
where $\beta>0$ controls regularization, and $\Delta(\X)$ is the probability simplex.
The optimizer is a Gibbs posterior~\citep{peters2007reinforcement,korbak2022reinforcement,rafailov2023direct};
\begin{equation}
\label{eq:temp_post_obj}
q_t^\star(x) \propto p_0(x)\exp\bigl(\beta a_t(x)\bigr).
\end{equation}
The KL term prevents the fine-tuned sampler from drifting too far from $p_0$, mitigating mode collapse~\citep{srivastava2017veegan}. With Probability of Improvement (PI; \citep{kushner1964new}), $a_t(x)=\log\Pr(Y\ge\tau_t\mid X=x)$ and $\beta=1$ recover the event posterior $p(x\mid Y\ge\tau_t)$~\citep{hvarfner2022pibo,adachi2024quadrature,adachi2024looping,xu2024principled}. The Expected Improvement (EI~\citep{mockus1975application}) in GenBO~\citep{oliveira2025generative} applies the same log-form $a_t(x) = \beta^{-1} \log \mathrm{EI}_t(x)$.

\textbf{Best-of-$N$ sampling.}
In practice, the target \eqref{eq:temp_post_obj} is approximated by a diffusion-based sampler:
\begin{align*}
&\hat q_t(x) \propto p_0(x) \exp\bigl(\beta \hat a_t(x)\bigr),\\
&X_{t,1},\ldots,X_{t,N}\stackrel{\mathrm{i.i.d.}}{\sim} \nu_t,\\
&\hat{x}_t \in \argmax_{x\in \{X_{t,1},\ldots,X_{t,N}\}} \, \hat a_t(x),
\end{align*}
where $\widehat a_t\approx a_t$, $\widehat q_t\approx q_t^\star$, and $\nu_t\approx\widehat q_t$. Top-$K$ is the batch analogue. Relative to exact maximization, GDBO therefore introduces both terminal distribution error $\nu_t\approx\widehat q_t$ and finite-pool best-of-$N$ error\footnote{i.e., it is not exact maximization $\argmax_{x\in \X} \, \hat a_t(x)$.}.

\subsection{Fine-tuning a diffusion sampler}
GDBO methods follow Algorithm~\ref{alg:generic_gbo}\footnote{The regret theorem assumes labeled i.i.d. prior-refresh samples for score learning (Assumption~\ref{ass:score_reservoir_app}) to avoid martingale arguments. In experiments, we train on all labels, including top-$K$ exploitation labels, which often improves performance; proving this requires adaptive local score-learning bounds. Appendix~\ref{app:unified_active_learning_floor} provides the needed local ranking decomposition.}; they differ in how they estimate $\widehat a_t$ and align the sampler with the Gibbs posterior. Unlike classical BO, they need not place a GP~\citep{rasmussen2003gaussian} on $f$.

\textbf{DiBO.}
DiBO~\citep{yun2025posterior} uses Upper Confidence Bound (UCB~\citep{srinivas2009gaussian}),
$a_t(x)=\mu_t(x)+\beta_{\mathrm{UCB}}\sigma_t(x)$, where $\mu_t$ and $\sigma_t$ come from a neural ensemble surrogate. It aligns the diffusion sampler with Eq.~\eqref{eq:temp_post_obj} using Relative Trajectory Balance (RTB) loss~\citep{venkatraman2024amortizing}.

\textbf{GenBO.}
GenBO~\citep{oliveira2025generative} uses EI as its acquisition function: $\mathbb{E}[(y-\tau_t)_+]$, where $\tau_t\in\mathbb{R}$ is a threshold and $(z)_+ := \max\{z,0\}$. It adopts a likelihood-free BO (LFBO; \citep{song2022general}) fine-tuning approach. Rather than placing a Bayesian surrogate on the objective $f$, LFBO directly learns an acquisition estimator $\widehat a_t \approx a_t$ using deterministic models, avoiding the need for ensembles.
Based on density-ratio estimation~\citep{sugiyama2012density, tiao2021bore, oliveira2022batch}, LFBO derives a weighted classification loss using binary labels
$z_i := \mathbbm{1}\{y_i \geq \tau_t\}$,
with weights $u(y_i) = (y_i-\tau_t)_+$. GenBO further extends this idea to a diffusion-alignment loss:\footnote{The original GenBO primarily uses Direct Preference Optimization (DPO)~\citep{rafailov2023direct, wallace2024diffusion, fujisawa2025scalable} with $S(q,x_i):=\log q_i(x)$. Here we focus on their extension to diffusion models in Section 3.3.}
$\mathcal{L}(q) := \sum_{(x_i,y_i)\in\D_t} \frac{p_0(x_i)}{\widehat q_{i-1}(x_i)} u(y_i) S(q,x_i)$,
where $S(q,x_i)$ is a proper scoring rule~\citep{gneiting2007strictly, adachi2025fixing}, such as the standard diffusion training loss.

\textbf{Algorithm flow.}
Algorithm~\ref{alg:generic_gbo} and Fig.~\ref{fig:gdbo} summarize the overall procedure. For a generic theory, we assume $J\geq 1$. The special case $J=0$ is treated separately in Appendix~\ref{app:unified_threshold_acceleration} as local score learning.

\section{A regret decomposition based on search and score learning}
\label{sec:mass_regret}

We analyze the GDBO family rather than a single algorithm. Let $a_t:\X\to\R$ be an \emph{ideal} acquisition used only for analysis, let $\widehat a_t$ be its learned version. We define the $\gamma$-optimal set:
\begin{align}
\label{eq:gamma-optimal_set}
    \U_\gamma:=\{x\in\X: f^\star-f(x)\le \gamma\}.
\end{align}

\subsection{Standing assumptions}
\label{subsec:generic_score_rate_main}
Formal assumptions are given in Appendix~\ref{app:formal_assumptions}; here we state the high-level content:
\begin{compactenum}
    \item[(A1)] We require $p_0(\U_\gamma)>0$. If this condition fails, then $p_0(\Xstar)=0$, so any sampler absolutely continuous with respect to $p_0$ assigns zero probability to querying a global optimum. This requirement can be viewed as Cromwell's rule for GDBO~\citep{jackman2009bayesian}.
    \item[(A2)] The ideal score ranks larger objective values higher. Noisy EI satisfies this condition (Proposition~\ref{prop:ei_calibrated_app} and Corollary~\ref{cor:log_ei_calibrated_app}).  Noiseless EI is not globally calibrated; it is used only through the setwise threshold certificate in Proposition~\ref{prop:noiseless_ei_setwise} and the local regret bound Corollary~\ref{cor:noiseless_ei_local_regret}.
    \item[(A3)] We assume the acquisition model lies in a Besov space and learn its parameters via clipped-ReLU empirical risk minimization (ERM). This assumption is standard in diffusion theory~\citep{oko2023diffusion,kawata2025direct}. Unlike RKHSes, Besov spaces capture inhomogeneous smoothness and singularities~\citep{suzuki2018adaptivity,imaizumi2019deep}, making them well suited to structured materials-discovery problems, which are often described as “needle-in-a-haystack” searches~\citep{hautier2019finding,siemenn2023fast}.
    \item[(A4)] The diffusion sampler may be imperfect: the realized law $\nu_t$ differs from $\widehat q_t$ through a setwise terminal mass error $d_{t,A}$.
\end{compactenum}

\subsection{A generic regret decomposition}
\label{subsec:main_regret_decomp}

\begin{theorem}[Regret decomposition, informal]
\label{thm:regret_decomp_informal}
Under Assumption~\ref{ass:calibrated_score_app}, for any $\gamma\ge0$,
\begin{equation}
\label{eq:main_regret_decomp_informal}
 r_T := \E[f^\star-f(\widehat x_T)]
 \le
 \underbrace{\gamma}_{\text{optimization accuracy}}
 +
 \underbrace{2L_a\,\E\|\widehat a_T-a_T\|_\infty}_{\text{score-learning term}}
 +
 \underbrace{R_f\,\Pbb(\D_T\cap \U_{\gamma_T}=\varnothing)}_{\text{search term}}.
\end{equation}
The formal statement and proof are in Theorem~\ref{thm:formal_regret_decomp}.
\end{theorem}

The score-learning term is obtained via $s_a$-smooth Besov functions after $\Psi:\X\to[0,1]^{d_a}$ features.

\begin{theorem}[Generic score-learning rate, informal]
\label{thm:generic_score_rate_informal}
Under Assumptions~\ref{ass:finite_alignment_app},\ref{ass:score_reservoir_app},\ref{ass:besov_score_app},\ref{ass:relu_erm_app}, fix any effective smoothness $0<\bar s_a<s_a-d_a/p$.
The balanced-width clipped-ReLU ERM satisfies
\begin{equation}
\label{eq:main_score_rate_informal}
    \E\|\widehat a_t-a_t\|_\infty
    =
    O\!\left(t^{-\vartheta_a}(\log t)^{c_a}\right),
    \qquad
    \vartheta_a=\frac{\bar s_a}{2\bar s_a+d_a}.
\end{equation}
Moreover, for every fixed $\kappa_a>0$ there is an envelope
$\bar\varepsilon_t=C t^{-\vartheta_a}(\log t)^{c_a}$ such that
\begin{equation}
\label{eq:main_score_hp_informal}
    \Pbb\{\|\widehat a_t-a_t\|_\infty>\bar\varepsilon_t\}
    \le u_t := C t^{-1-\kappa_a}.
\end{equation}
The formal result is Theorem~\ref{thm:generic_score_learning_rate}. If a direct $L_\infty$ ReLU approximation theorem with exponent $s_a$ is available for the chosen score class, then $\bar s_a$ can be replaced by $s_a$.
\end{theorem}

For exact optimization set $\gamma_T=0$ and $\U_{\gamma_T}=\Xstar$. Given slack $\omega_{t,\Xstar}>0$ and $p_0^\star:=p_0(\Xstar)$, define
\begin{equation}
\label{eq:main_alpha_def}
    m_t^\star:=[\ell_{t,\Xstar}-\omega_{t,\Xstar}]_+,
    \quad
    \lambda_{{\rm BO}, t} :=J\Lambda(p_0^\star)+N\Lambda(m_t^\star),
    \quad
    \Lambda(u):=-\log(1-u).
\end{equation}

\begin{theorem}[Search term, informal]
\label{thm:end_to_end_informal}
Suppose top-$K$ selection uses $K\ge1$. After the score activation time $T_\mathrm{act}$ in Eq.~\eqref{eq:activation_time}, the miss probability satisfies
\begin{align}
\label{eq:main_end_to_end_miss_informal}
    \Pbb(\D_T\cap\Xstar=\varnothing)
    &\le
    \exp\!\left(-\sum_{t=T_\text{act}}^T \lambda_{\rm BO, t}\right)
    +
    \sum_{t=T_\text{act}}^T
    \exp\!\left(-\sum_{s=t+1}^T\lambda_{\rm BO, s}\right)
    \left( u_t+\frac{d_{t,\Xstar}^2}{\omega_{t,\Xstar}^2}\right),
\end{align}
The formal statement is Theorem~\ref{thm:formal_end_to_end_miss}.
\end{theorem}

This is the point where a concrete sampler enters: it must bound $d_{t,A}$ and certify enough mass $\ell_{t,A}$. 

\begin{corollary}[Clean rate after a sampler-specific floor, informal]
\label{cor:clean_rate_from_sampler_informal}
Suppose that after a burn-in time $T_0$, $J\ge1$, $N\ge1$, $\ell_{t,\Xstar}\ge\ell_\star>0$, and the sampler-specific analysis gives
$\frac{d_{t,\Xstar}^2}{(\ell_\star/2)^2} \le C_{\rm alg}t^{-1-\kappa_{\rm alg}}$.
\[
    r_T
    =
    O\!\left(
    \underbrace{T^{-\vartheta_a}(\log T)^{c_a}}_\text{score-learning rate}
    +\underbrace{T^{-1-\kappa_a}}_\text{activation failure probability}
    +\underbrace{T^{-1-\kappa_{\rm alg}}}_\text{sampler approximation error}
    +\underbrace{\exp(-\lambda_{\rm BO}T)}_\text{$\Xstar$ miss search probability}
    \right).
\]
The formal proof is Corollary~\ref{cor:formal_clean_rate}.
\end{corollary}

\textbf{Interpretation.}
Corollary~\ref{cor:clean_rate_from_sampler_informal} contains both exponential search terms and polynomial terms. Asymptotically, the slowest polynomial term dominates; at finite budgets, however, search, activation, or sampler floors can be the bottleneck. We next explain how this produces the regimes in Figure~\ref{fig:motivation}.

\subsection{Problem-specific acceleration via threshold mass lift}
\label{subsec:mass_lift_window_main}

\begin{figure}
    \centering
    \vspace{-1em}
    \includegraphics[width=0.8\linewidth]{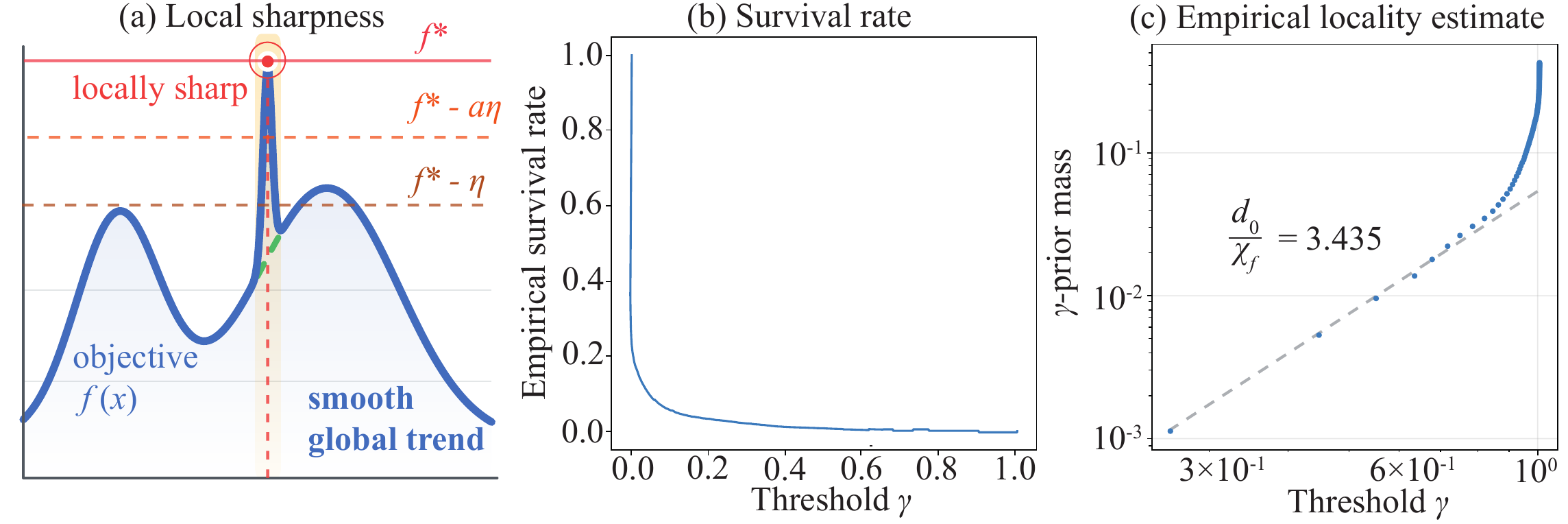}
    \caption{
    Local sharpness and threshold mass lift:
    (a) visual explanation;
    (b) the empirical survival curve is highly skewed;
    (c) a rapidly shrinking \(\gamma\)-optimal set.
    }
    \label{fig:local}
\end{figure}

\begin{wraptable}[7]{r}{0.59\textwidth}
\centering
\vspace{-1.5em}
\caption{Three finite-budget regimes.}
\label{tab:rate}
\resizebox{0.58\textwidth}{!}{%
\begin{tabular}{lll}
\toprule
Acceleration
& Relevant set
& Mass-lift condition\\
\midrule
Exponential
& \(\U_{a\eta}\)
& \(\nu_t(\U_{a\eta}) \gg p_0(\U_{a\eta})\)\\
Polynomial
& \(\mathcal C_{T,\eta}\)
& \(Q_T^{\rm act}(\mathcal C_{T,\eta}) \gg p_0(\mathcal C_{T,\eta})\)\\
Constant-factor
& \(\U_{a\eta},\,\mathcal C_{T,\eta}\)
& no substantial lift\\
\bottomrule
\end{tabular}
}
\end{wraptable}
The exact-optimizer bound above is an asymptotic no-regret statement: it controls the probability of ever hitting $\Xstar$. Finite-budget behavior is different. Before the search reaches $\Xstar$, regret is governed by moving near-optimal sets $\U_\gamma$, with the relevant threshold $\gamma$ shrinking as the budget grows. In particular, after the current best gap is $\eta$\footnote{$\eta=\min_{x_i\in\D_t} f^\star-f(x_i)$, the current simple regret.}, achieving a multiplicative improvement by $a\in(0,1)$ requires querying $\U_{a\eta}$ (Fig.~\ref{fig:local}(a)). Random search hits this progress set with probability $p_0(\U_{a\eta})$, whereas GDBO hits it with probability $\nu_t(\U_{a\eta})$. We call the regime $\nu_t(\U_{a\eta})\gg p_0(\U_{a\eta})$ \emph{threshold mass lift}. Table~\ref{tab:rate} summarizes the resulting finite-budget regimes.

\textbf{Local sharpness and random search.}
Random-search difficulty is controlled by how fast $p_0(\U_\gamma)$ vanishes as $\gamma\downarrow0$. A useful local model is
\begin{equation}
\label{eq:local_geometry_main}
    f^\star-f(x)
    \asymp
    \operatorname{dist}_\Psi(x,\Xstar)^{\chi_f},
    \qquad
    p_0\{x:\operatorname{dist}_\Psi(x,\Xstar)\le r\}
    \asymp
    r^{d_0},
\end{equation}
so $p_0(\U_\gamma)\asymp\gamma^{d_0/\chi_f}$. Thus, although fixed-threshold hits are exponential, $\Pbb\{\Gamma_n>\gamma\}=\exp\{-n\Lambda(p_0(\U_\gamma))\}$, the expected best gap looks polynomial because the relevant $\gamma$ shrinks with budget: $\E[\Gamma_{BT}]\asymp T^{-\chi_f/d_0}$. The fitted exponent in Figure~\ref{fig:local}(c) is therefore an effective near-optimal mass exponent: it measures how rapidly the prior mass of \(\U_\gamma\) vanishes as the threshold becomes stricter.  This is consistent with the ``needle-in-the-haystack'' structure in materials
discovery~\citep{hautier2019finding, siemenn2023fast}.

\begin{wrapfigure}[11]{r}{0.4\textwidth}
    \vspace{-1.0em}
    \centering
    \includegraphics[width=0.4\textwidth]{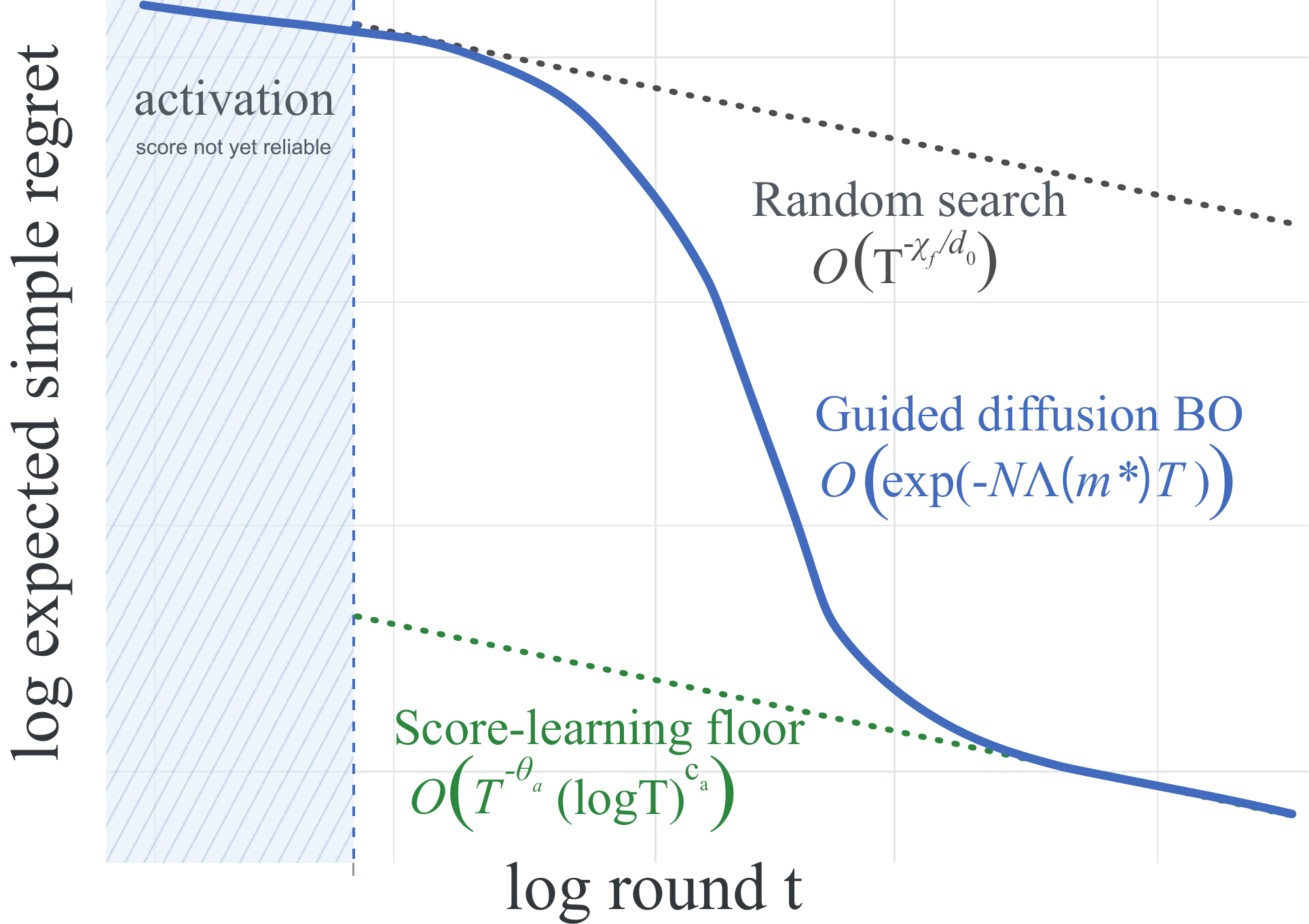}
    \caption{Exponential acceleration.}
    \label{fig:exponential}
\end{wrapfigure}
\textbf{Guided diffusion lifts threshold mass.}
Threshold acquisitions such as EI naturally target progress sets. In the noiseless population case, current gap $\eta$ gives utility $I_\eta(x)=[\eta-(f^\star-f(x))]_+$ and tilted target $q_\eta^{\rm EI}(x)\propto p_0(x)[\eta-(f^\star-f(x))]_+$\footnote{This threshold target can be viewed as the Gibbs posterior in Eq.~\eqref{eq:temp_post_obj} applied to the log-utility $a_t := \log I_\eta(x)$.}. Under Eq.~\eqref{eq:local_geometry_main},
\begin{equation}
\label{eq:ei_constant_progress_main}
    q_\eta^{\rm EI}(\U_{a\eta})
    \gtrsim
    (1-a)\frac{p_0(\U_{a\eta})}{p_0(\U_\eta)}
    \asymp
    1.
\end{equation}
Hence random search hits $\U_{a\eta}$ with vanishing probability, whereas the guided sampler can hit it with constant probability. After score activation and sampler control, the best threshold gap contracts exponentially until it reaches a floor:
\begin{equation}
\label{eq:threshold_contraction_main}
    \E(\eta_T-\eta_{\min})_+
    \le
    (\eta_{T_0}-\eta_{\min})_+
    \exp\{-p_{\rm prog}(1-a)(T-T_0)\},
\end{equation}
where $p_{\rm prog}$ lower bounds the per-round probability of hitting $\U_{a\eta}$.

\begin{wrapfigure}[12]{r}{0.4\textwidth}
    \vspace{-1.0em}
    \centering
    \includegraphics[width=0.4\textwidth]{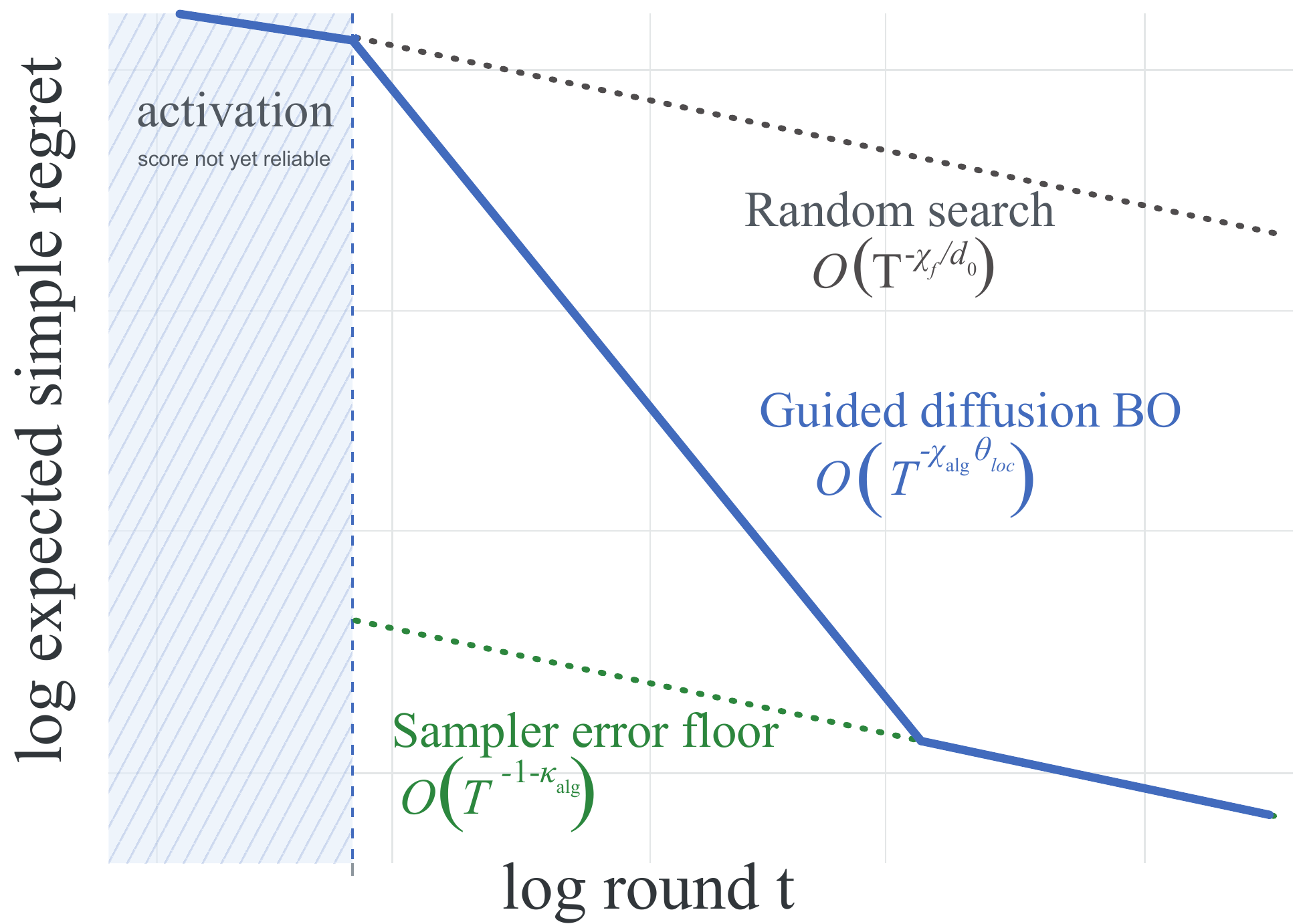}
    \caption{Polynomial acceleration.}
    \label{fig:polynomial}
\end{wrapfigure}
\textbf{Active-learning effect.}
Mass lift can also accelerate score learning. Once the search phase reaches a floor, the only remaining errors come from a local critical band $\mathcal C_{T,\eta}$: points whose ideal scores are close enough to those of near-optimal points that the learned acquisition model may still confuse them. If guided sampling allocates substantially more labeled data to this band than the prior does, $Q_T^{\rm act}(\mathcal C_{T,\eta})\gg p_0(\mathcal C_{T,\eta})$, score learning becomes a form of local active learning. Here, $Q_T^{\rm act}$ denotes the empirical distribution of labeled data used to train the score model.
Appendix~\ref{app:unified_threshold_acceleration} shows that if $Q_T^{\rm act}(\mathcal C_{T,\eta})\gtrsim \eta^{\beta_{\rm act}}$, the floor scales as $\widetilde O\!\left(T^{-\chi_a\vartheta_{\rm loc} / (1+\chi_a\beta_{\rm act}\vartheta_{\rm loc})}\right)$; if the guided sampler keeps constant critical-band mass ($\beta_{\rm act}=0$), this becomes $\widetilde O(T^{-\chi_a\vartheta_{\rm loc}})$. GDBO therefore yields polynomial acceleration when this local rate beats the global rate $\vartheta_a$.

\textbf{Discussion: GDBO vs. GP-UCB.}
GP-UCB typically uses a globally homogeneous smoothness model, often summarized by a lengthscale $\theta_\ell$. For locally sharp functions this can be pessimistic: complexity scales with $(1/\theta_\ell)^d$~\citep{ziomek2024bayesian}, and the roughest local region can dominate. GDBO instead uses the pretrained generator and acquisition-guided sampling to place mass on the threshold sets that matter, turning cursed local needle-in-the-haystack geometry into a possible source of acceleration.

\section{How Algorithmic Choices Enter the Bound}
\label{sec:algorithm_interface}

The analysis reduces a concrete GDBO algorithm to certificates. For each search-relevant set $A\subseteq\X$---$\Xstar$ for no-regret, $\U_{a\eta_t}$ for exponential progress, or $\mathcal C_{T,\eta}$ for active learning---the sampler must certify enough terminal mass on $A$; it need not solve exact acquisition maximization.

\subsection{Diffusion-sampler certificates}
\begin{table}[t]
\centering
\caption{
Sufficient sampler certificates for the terminal-mass interface.  Here
\(A\) denotes any search-relevant set, such as \(\Xstar\) or
\(\U_{a\eta_t}\).  \(M_t\) denotes an internal sampler-training, particle, or
optimization budget, and \(b_{t,A}\) denotes residual decoder or optimization
bias.
}
\label{tab:sampler_floor_summary}
\small
\begin{tabular}{p{0.16\linewidth}p{0.49\linewidth}p{0.27\linewidth}}
\toprule
Sampler family & Sufficient certificate & Consequence for \(d_{t,A}^2\) \\
\midrule
Exact
&
\(\nu_t=\widehat q_t\)
&
\(d_{t,A}=0\).
\\

Generic TV
&
\(\E[\TV(\nu_t,\widehat q_t)^2\mid\F_{t-1}]\le\delta_t^2\)
&
\(d_{t,A}^2\le\delta_t^2\).
\\

Generic KL
&
\(\E[\KL(\widehat q_t\|\nu_t)\mid\F_{t-1}]\le\varepsilon_t\)
&
\(d_{t,A}^2\le\varepsilon_t/2\).
\\

DiBO/RTB~\citep{yun2025posterior}
&
$\E[\KL(P_t^\star\|P_{\theta_t})\mid\F_{t-1}]
\le
\mathcal E_t^{\rm RTB}+b_{t,A,{\rm dec}}^2$,
with \((\Pi_t)_\#P_t^\star=\widehat q_t\) and
\((\Pi_t)_\#P_{\theta_t}=\nu_t\).
&
$d_{t,A}^2
\le
\tfrac12\{\mathcal E_t^{\rm RTB}
+b_{t,A,{\rm dec}}^2\}$.
\\

GenBO~\citep{oliveira2025generative}
&
$\E[\mathcal L_t^S(\nu_t)-\mathcal L_t^S(\widehat q_t)\mid\F_{t-1}]
\le
\mathcal E_t^S+b_{t,{\rm opt}}^2$,
with \(D_S(\widehat q_t,q)\ge c_S\TV(\widehat q_t,q)^2\).
&
$d_{t,A}^2
\le
c_S^{-1}\{\mathcal E_t^S+b_{t,{\rm opt}}^2\}.
$
For the log score, \(c_S=2\).
\\[0.2em]

GenBO + \citep{oko2023diffusion}
&
$\E[\TV(\nu_t,\widehat q_t)^2\mid\F_{t-1}]
\le
\widetilde O(M_{t,{\rm eff}}^{-s/(2s+d_x)})
+b_{t,{\rm opt}}^2$.
&
Same for \(d_{t,A}^2\).  Choose
\(M_{t,{\rm eff}}\) to bound
\(O(t^{-1-\kappa_{\rm alg}})\).
\\
Prior-resampling wrapper
&
Draw $M_t$ proposals from $p_0$, reweight by
$w_t(x)=\exp\{\beta \widehat a_t(x)\}$, and resample with normalized $w_t(x)$.
&
$d_{t,A}^2\le \exp(2\beta A_{\rm osc})/M_t$ for every $A$ with $\operatorname{osc}(\widehat a_t)\le A_{\rm osc}$.\\
\bottomrule
\end{tabular}
\end{table}
Once the learned acquisition $\widehat a_t$ is frozen, the sampler interface is
\[
\widehat q_t(A)\ge \ell_{t,A},
\qquad
\mathbb E\!\left[
\{\widehat q_t(A)-\nu_t(A)\}_+^2
\,\middle|\,\mathcal F_{t-1}
\right]
\le d_{t,A}^2 .
\]
The first inequality is a target-mass certificate and the second is a terminal approximation certificate.  Together they give a miss exponent of order $N_t\Lambda([\ell_{t,A}-\omega_{t,A}]_+)$, plus the sampler floor $d_{t,A}^2/\omega_{t,A}^2$.  Therefore a concrete diffusion algorithm enters the regret bound only through $(\ell_{t,A},d_{t,A})$. 
\begin{remark}[Certificate interface]
\label{rem:certificate-interface}
Table~\ref{tab:sampler_floor_summary} should be read as a sampler-certificate interface. Any training analysis that certifies terminal TV, terminal KL, path-space KL, or a strongly proper scoring divergence to the frozen Gibbs target $\widehat q_t$ can be plugged into the regret bound through $(\ell_{t,A},d_{t,A})$. The table gives sufficient routes for DiBO/RTB, GenBO-style weighted scoring, diffusion distribution learning, and proposal-corrected resampling. Appendix~\ref{sec:theory-experiment-interface} summarizes which routes are used as formal certificates and which are used as empirical mass-lift diagnostics.
\end{remark}

\subsection{Acquisition-function certificates}

The acquisition function enters in three ways. First, it must provide a ranking certificate for the final recommendation; the global version is Assumption~\ref{ass:calibrated_score_app}. Second, it determines the target-mass lower bound $\ell_{t,A}$, the route to exponential-looking acceleration. EI-style threshold utilities can keep $\ell_{t,\U_{a\eta}}$ bounded away from zero even when $p_0(\U_{a\eta})$ is tiny. Third, the acquisition learner controls the score-learning floor. The global rate in Theorem~\ref{thm:generic_score_rate_informal} depends on the score class, whereas polynomial acceleration requires a local statement: the labeled design must place sufficient mass on $\mathcal C_{T,\eta}$. Thus acquisition functions can be plugged into the framework whenever these ranking, mass, and learning certificates can be verified; Appendix~\ref{app:ei_calibration} gives EI calibration and Appendix~\ref{app:mean_ucb_certificates} covers others.

\subsection{Related work}
\textbf{Diffusion black-box optimizer.}
Diffusion BBO has been studied mostly offline~\citep{krishnamoorthy2023diffusion,li2024diffusion} and recently in online GDBO~\citep{li2024diffusion,yun2025posterior,oliveira2025generative}, evolutionary search~\citep{zhang2024diffusion}, and reinforcement learning~\citep{chen2025matinvent}. We use DiBO and GenBO as motivating instances.\\
\textbf{BO for structured inputs.}
Structured surrogate methods include graph-kernel GPs~\citep{ru2021interpretable,griffiths2023gauche}, string kernels~\citep{moss2020boss}, and neural surrogates~\citep{zhou2020neural,white2021bananas}. They can provide regret guarantees but usually abstract away acquisition-maximization error. Latent BO~\citep{gomez2018automatic,grosnit2021high} optimizes in continuous embeddings, but relies on valid decoder extrapolation, which often fails in practice~\citep{moss2025return}.\\
\textbf{Priors over the design space.}
Prior-informed BO injects human experts' beliefs over $x$ into a Bayesian surrogate~\citep{ramachandran2020incorporating,souza2021bayesian,hvarfner2022pibo,hvarfner2024a,adachi2024looping,xu2024principled}, but it does not adapt a diffusion prior using observed data. Trust regions~\citep{eriksson2019scalable,eriksson2021high} can also be seen as a shrinking prior over $x$, yet impose hard constraints and therefore violate Assumption~\ref{ass:finite_alignment_app} unless made probabilistic. \\
\textbf{Approximation-aware BO.}
Approximation-aware BO has been studied in several forms, including hyperparameter estimation error~\citep{berkenkamp2019no, ziomek2024bayesian, ziomek2025time}, sparse variational approximations in GP modeling~\citep{maus2024approximation}, and quadrature error in batch sampling~\citep{adachi2024quadrature, adachi2024adaptive, osselin2025natural}. However, the approximation error induced by diffusion-based sampling has not yet been theoretically characterized.

\section{Experiments}\label{sec:experiments}
\textbf{Baselines and generators.}
We use DiBO~\citep{yun2025posterior} as the main GDBO baseline because it supports multiple acquisitions, and GenBO~\citep{oliveira2025generative} as a supporting baseline. For diffusion models, we use MatterGen~\citep{zeni2025mattergen} for crystals and DiGress~\citep{vignac2023digress} for molecules. Both are pretrained priors $p_0$ over realistic structures and are not conditioned on the target property.\\
\textbf{Evaluation and protocol.}
We evaluate crystal properties with ALIGNN~\citep{choudhary2021atomistic} and molecular properties with GFN2-xTB~\citep{bannwarth2019gfn2xtb}. For all tasks we convert the raw oracle value into a normalized reward $\bar f(x)\in[0,1]$ and report the normalized target regret $R_t^{\rm exp} = 1-\max_{x\in D_t}\bar f(x)$. We estimate guided threshold mass by Monte Carlo,
$\widehat\nu_t(S_\tau) = \frac1m\sum_{i=1}^m \mathbbm{1} \{\bar f(Z_{t,i})\ge \tau\}$, where $Z_{t,i}\sim\nu_t$ and $\tau=0.8$. We measure crystal stability with the quantum-chemical metrics from MatterGen~\citep{zeni2025mattergen}. All methods start from 32 prior samples; each round generates $N=128$ candidates, evaluates the top $K=32$, and repeats for 20 rounds. We use 5 seeds. Experiments ran on eight NVIDIA H200 GPUs (141GB each) with an Intel Xeon Platinum 8480+ CPU. For MatterGen, GenBO takes approximately 2.6 hours per run, while DiBO takes 13 hours. For DiGress, GenBO takes approximately 0.5 hours per run, while DiBO takes 1.6 hours. Code is available at \url{https://anonymous.4open.science/r/Diffusion-BO-E260}.

\begin{wrapfigure}[17]{r}{0.35\textwidth}
    \vspace{-4.5em}
    \centering
    \includegraphics[width=0.35\textwidth]{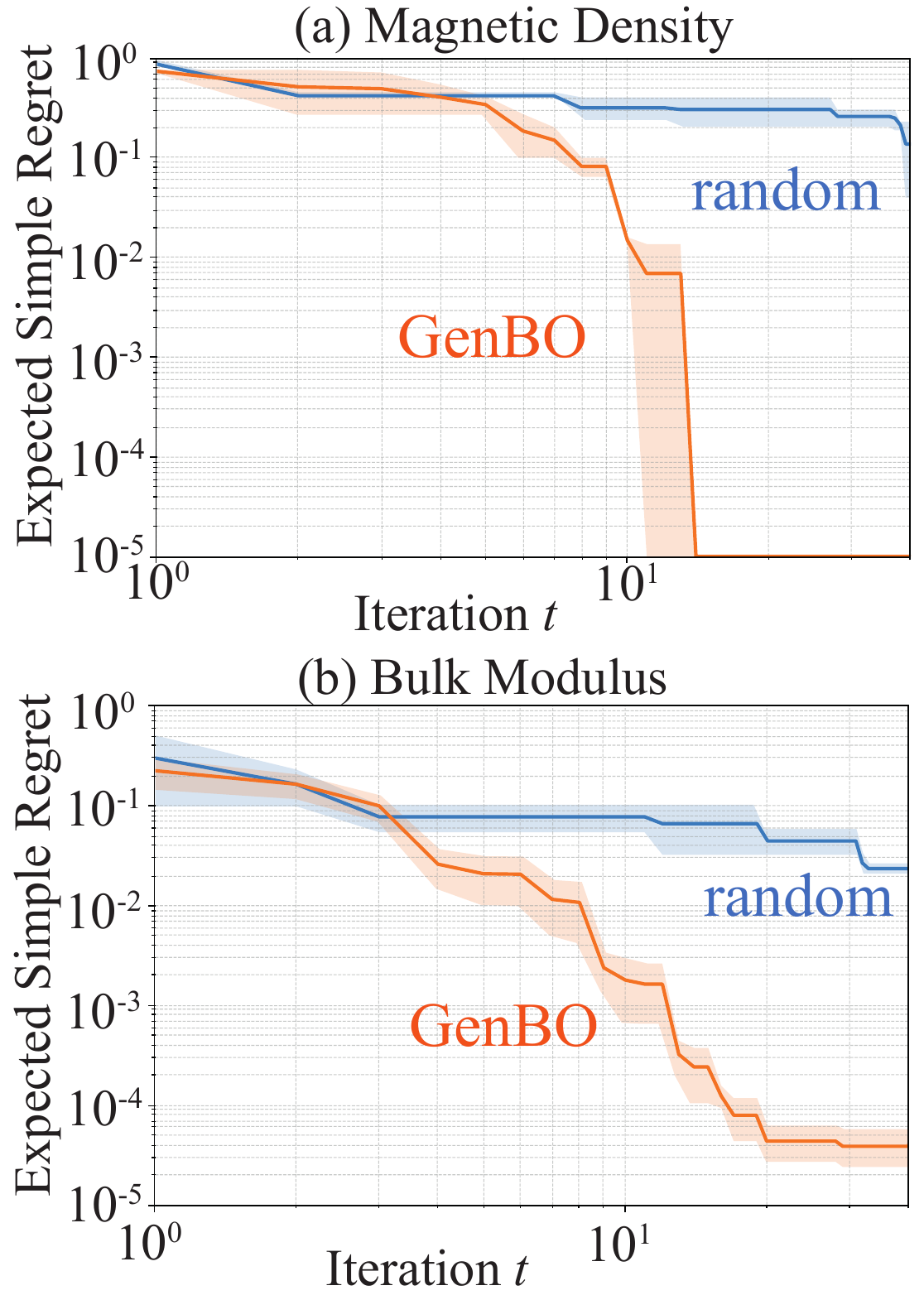}
    \caption{Longer iterations.}
    \label{fig:longer}
\end{wrapfigure}
\subsection{Design choice prediction}
\begin{figure}
    \centering
    \includegraphics[width=0.9\linewidth]{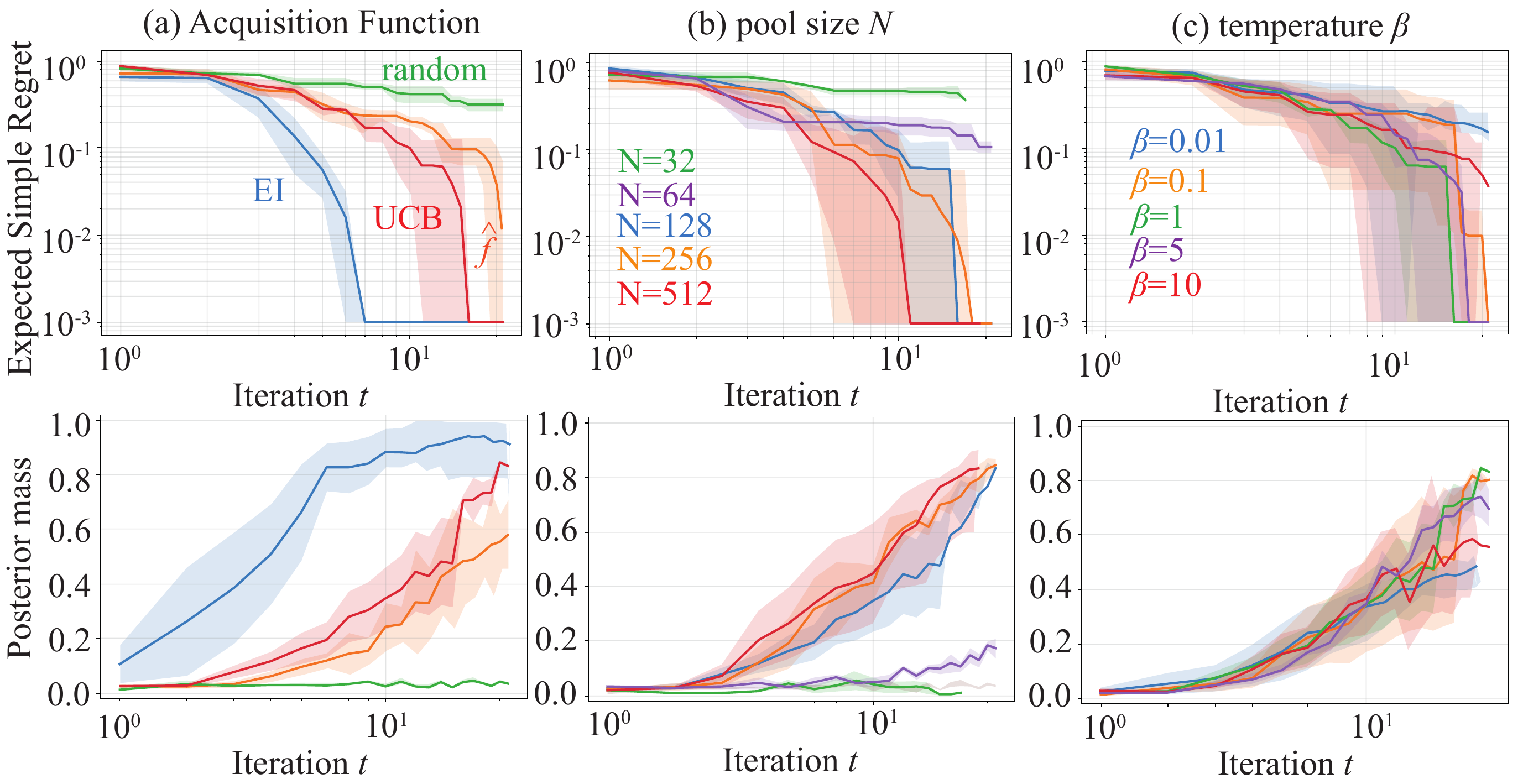}
    \caption{
    Theory-alignment diagnostics on magnetic-density optimization. Top row: normalized target regret. Bottom row: empirical guided threshold mass $\widehat\nu_t(S_{0.8})$. The trends match the search-exponent prediction: convergence is faster when (a) the acquisition produces stronger threshold mass lift and (b) the candidate pool size $N$ is larger, yet (c) temperature $\beta$ does not induce visible change.}
    \label{fig:hypers}
\end{figure}
We test whether design choices affect convergence as predicted. This subsection uses magnetic density, which exhibits the exponential acceleration in Figure~\ref{fig:motivation}(a). We vary the acquisition function $a_t$, $\beta$, $N$, and prior mass $p_0(\Xstar)$.

\textbf{Longer iterations.}
Corollary~\ref{cor:clean_rate_from_sampler_informal} predicts that finite-budget acceleration eventually hits the slowest polynomial floor. Figure~\ref{fig:longer} supports this: in the exponential case all seeds find the ground truth before the floor is visible, while in the polynomial case the curve flattens, consistent with a sampler-approximation floor. Similar sigmoid-like floor effects appear in GenBO~\citep{oliveira2025generative}, e.g., its Fig. 1(a) and Fig. 2(a).\\
\textbf{Acquisition function.}
Our framework requires ranking and mass certificates, not uncertainty estimation per se. We therefore compare EI, UCB, and the no-uncertainty plug-in score $\widehat a_t(x)=\widehat y_t(x)$. These fit the interface in different ways: EI provides a threshold-progress certificate through the EI-weighted/log-EI target, while UCB and the plug-in score are treated as setwise perturbations of the ideal score $a=f$ as in Appendix~\ref{app:mean_ucb_certificates}. Figure~\ref{fig:hypers}(a) shows that all three can yield exponential-looking acceleration when they produce sufficient threshold mass lift.\\
\textbf{Parameters.}
The search term depends on the finite-pool hit exponent $N\Lambda(\nu_t(A))$. For exact search, Proposition~\ref{prop:gibbs_target_mass_app} gives the Gibbs-mass
lower bound $\widehat q_t(\mathcal X^\star) \ge
\left[ 1+ \frac{1-p_0^\star}{p_0^\star} \exp\left\{-\frac{\beta_t\Delta_f^\star}{2L_a}\right\} \right]^{-1}$, which increases with both $\beta_t$ and prior mass. For finite-budget threshold search, the analogous diagnostic is the empirical mass $\widehat\nu_t(S_\tau)$. 
We therefore vary the candidate-pool size N and the inverse-temperature parameter $\beta$. Figure~\ref{fig:hypers}(b) shows that increasing $N$ can accelerate convergence, whereas changing $\beta$ within the tested range has no visible effect. Because larger $\beta$ also increases the fraction of corrupted crystals, a larger candidate pool $N$ is preferable.\\
\textbf{Posterior mass.}
Across settings, expected regret tracks posterior mass lift closely, supporting the mass-lift explanation.

\vspace{-0.5em}
\subsection{Theory-experiment alignment}
\begin{wrapfigure}[6]{r}{0.35\textwidth}
    \vspace{-4em}
    \centering
    \includegraphics[width=0.35\textwidth]{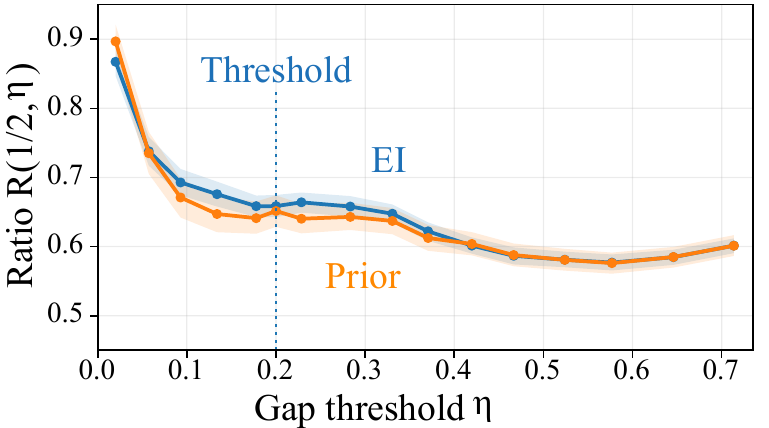}
    \caption{Reverse doubling.}
    \label{fig:reverse}
\end{wrapfigure}
\textbf{Reverse-doubling and mass-lift diagnostics.}
We separate two acceleration effects predicted by the theory.  First, we audit the exponential acceleration effect in the local threshold geometry by estimating the reverse-doubling ratio $R_s(a,\eta) := \frac{M_s(a\eta)}{M_s(\eta)}$, where $M_s(\eta):=\Pr_s(x\in\mathcal U_\eta)$.
For the prior \(s=p_0\), Proposition~\ref{prop:ei_progress_mass} implies the ideal EI lower bound $q_\eta^{\rm EI}(\mathcal U_{a\eta}) \ge (1-a)R_0(a,\eta)$. Thus a ratio bounded away from zero over the tested threshold window certifies that EI can make multiplicative progress with constant probability, conditional on accurately sampling the EI-tilted target (See §\ref{subsec:mass_lift_window_main}).  This diagnostic is a statement about the shape of the near-optimal tail. In our experiments the reverse-doubling curves remain bounded away from zero, supporting the constant-progress geometry.

\begin{wrapfigure}[14]{r}{0.5\textwidth}
    \vspace{-2.5em}
    \centering
    \includegraphics[width=0.5\textwidth]{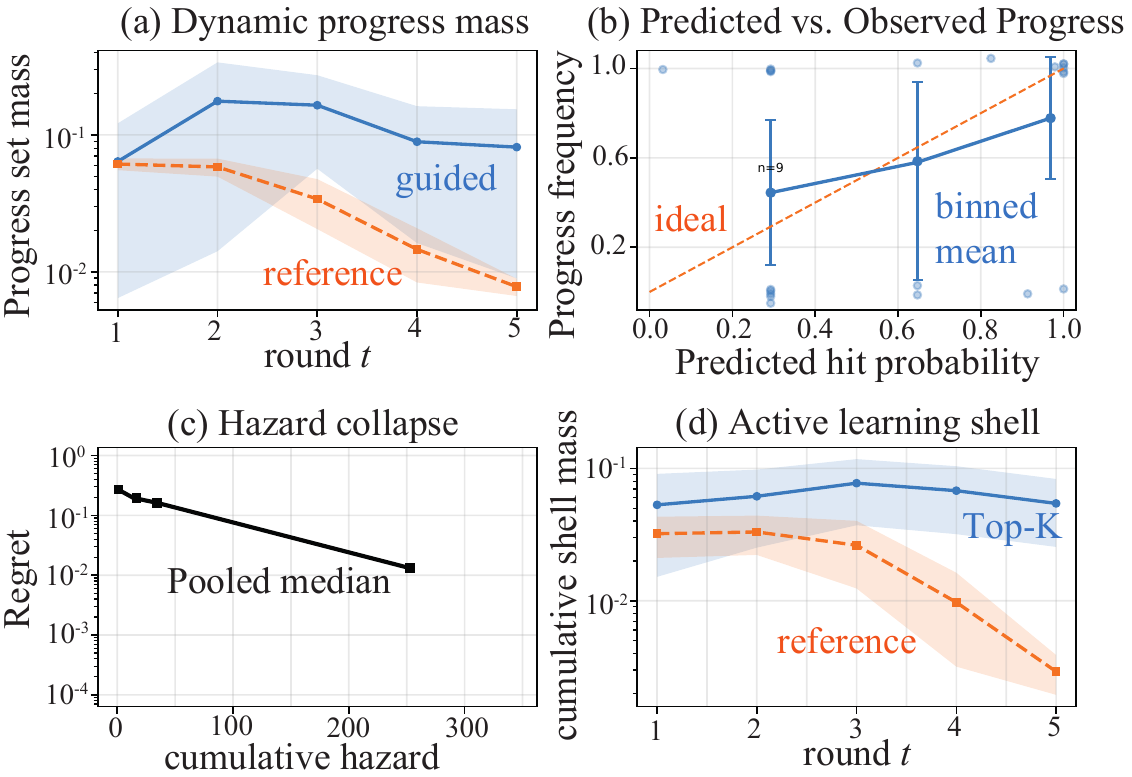}
    \caption{EI mass lift diagnosis.}
    \label{fig:ei_mass_lift}
\end{wrapfigure}
\textbf{Mass-lift diagnostics.}
Second, we audit the polynomial effect. We post-process the EI runs using the dynamic progress set
$\mathcal U_{a\eta_{t-1}}$, where $\eta_{t-1}=1-\max_{x\in\mathcal D_{t-1}} f(x)$ and $a=0.75$.
For each round, we estimate the guided terminal mass
$\widehat\nu_t(\mathcal U_{a\eta_{t-1}})$ from an independent audit pool and convert it into the predicted pool-hit probability $1-(1-\widehat\nu_t)^N$.
The resulting diagnostic shows that (a) the estimated progress mass increases over the active phase relative to an early-round reference audit distribution, and (b) the posterior-predictive pool-hit probability is directionally aligned with observed multiplicative progress. (c) Plotting regret against the cumulative hazard $\sum_t N\Lambda(\widehat\nu_t)$ further collapses the search phase, supporting the interpretation that EI accelerates optimization by increasing the probability of sampling from the moving progress set. (d) The final panel reports the cumulative mass of evaluated top-$K$ points in the
local shell $\mathcal U_{\eta}\setminus\mathcal U_{0.75\eta}$, which is a practical proxy for the active critical band used in the local score-learning analysis.

\section{Conclusion}
We introduced a \emph{first} certificate-based regret framework for guided-diffusion black-box optimization over structured inputs and identified \emph{mass lift} as the mechanism controlling finite-budget acceleration. The analysis explains how the same GDBO template can exhibit exponential-looking search phases, faster polynomial active-learning phases, or only constant-factor gains, depending on how much probability the guided sampler assigns to the relevant near-optimal, progress, or critical-band sets.
Several directions follow naturally. On the algorithmic side, one can design acquisitions that directly maximize threshold lift, adapt $\beta$, $N$, and $K$ using estimated exponent gains, and develop diffusion or flow-matching samplers with finite-sample terminal TV/KL certificates. 
Our framework also naturally extends beyond diffusion models to other generative samplers, including flow matching~\citep{lipman2023flow} and auto-regressive models such as large language models~\citep{yang2026verbalizing, adachi2026open}.
The main limitation is the finite-domain assumption. In continuous domains, finite-$N$ best-of-$N$ sampling can certify only $\gamma$-optimality unless $N\to\infty$ as $\gamma\downarrow0$. This is acceptable in settings such as materials discovery, where the structured search space is astronomically large but finite. 
The asymptotic rate is also not tight, since our learning rate is not tied to any particular acquisition function; sharper acquisition-specific analyses may improve it. Our goal is instead to explain the distinctive convergence regimes of GDBO, and the mass-lift framework provides this explanation.
The main message is that guided diffusion changes BO not by exactly maximizing an acquisition, but by reshaping the probability of sampling the sets that matter.

\bibliographystyle{unsrt}
\bibliography{references} 


\appendix
\newpage
\addcontentsline{toc}{section}{Appendix} 
\part{Appendix} 
\parttoc 

\section{Notation}\label{app:notation}

Table~\ref{tab:notation} summarizes the main notation used throughout the paper.
Symbols are grouped by role. Some symbols are local to a theorem or appendix; when a symbol is reused locally, the relevant context is stated explicitly.

\begingroup
\small
\setlength{\tabcolsep}{4pt}
\renewcommand{\arraystretch}{1.12}

\begin{longtable}{@{}>{\raggedright\arraybackslash}p{0.24\textwidth}
                  >{\raggedright\arraybackslash}p{0.70\textwidth}@{}}
\caption{Notation used in the paper.}
\label{tab:notation}\\
\toprule
Symbol & Meaning \\
\midrule
\endfirsthead

\toprule
Symbol & Meaning \\
\midrule
\endhead

\midrule
\multicolumn{2}{r}{\emph{continued on next page}}\\
\endfoot

\bottomrule
\endlastfoot

\multicolumn{2}{@{}l}{\textbf{Domain, prior, objective, and observations}}\\
\addlinespace[0.2em]
\(\X=\supp(p_0)\) & Finite accessible structured design domain induced by the pretrained generator. \\
\(x\in\X\), \(X\) & A design object and its random counterpart. \\
\(p_0\), \(\prior\) & Frozen pretrained generator or prior distribution over \(\X\). \\
\(\supp(p_0)\) & Support of the prior \(p_0\). \\
\(f:\X\to\R\) & Expensive black-box objective. Larger values are better. \\
\(f^\star\) & Global optimum value, \(f^\star=\max_{x\in\X}f(x)\). \\
\(\Xstar\) & Global optimizer set, \(\Xstar=\argmax_{x\in\X}f(x)\). \\
\(p_0^\star\) & Prior mass of the optimizer set, \(p_0^\star=p_0(\Xstar)\). \\
\(Y=f(X)+\xi\) & Noisy observation model. \\
\(\xi\), \(\varepsilon\) & Observation noise. The appendix often writes the additive noise as \(\varepsilon\). \\
\(\D_t\) & Evaluated dataset available after round \(t\). \\
\(\widehat x_T\) & Final recommendation after budget \(T\). \\
\(R_f\) & Objective range, \(R_f=f^\star-\min_{x\in\X}f(x)\). \\
\(\rho(x)\) & Objective gap, \(\rho(x)=f^\star-f(x)\), used mainly in Appendix~\ref{app:unified_threshold_acceleration}. \\
\(\bar f(x)\) & Normalized experimental score in \([0,1]\). \\

\addlinespace[0.4em]
\multicolumn{2}{@{}l}{\textbf{Regret and near-optimal sets}}\\
\addlinespace[0.2em]
\(r_T\) & Expected simple regret, \(r_T=\E[f^\star-f(\widehat x_T)]\). \\
\(\gamma\), \(\gamma_T\) & Near-optimality or recommendation-accuracy level. \\
\(\U_\gamma\) & \(\gamma\)-optimal set, \(\U_\gamma=\{x\in\X:f^\star-f(x)\le\gamma\}\). \\
\(\Gamma_n\) & Best true gap after \(n\) random-search samples, \(\Gamma_n=\min_{1\le i\le n}\rho(X_i)\). \\
\(M_0(\gamma)\) & Prior near-optimal mass profile, \(M_0(\gamma)=p_0(\U_\gamma)\). \\
\(\Lambda(u)\) & Bernoulli miss exponent, \(\Lambda(u)=-\log(1-u)\). \\

\addlinespace[0.4em]
\multicolumn{2}{@{}l}{\textbf{Acquisition functions and Gibbs targets}}\\
\addlinespace[0.2em]
\(a_t:\X\to\R\) & Ideal acquisition or score used for analysis. \\
\(\widehat a_t\) & Learned acquisition or score used by the algorithm. \\
\(L_a\) & Inverse-link constant in the calibrated-score assumption. \\
\(\Delta(\X)\) & Probability simplex over the finite domain \(\X\). \\
\(\beta\), \(\beta_t\) & Gibbs inverse-temperature or acquisition-tilt strength. \\
\(q_t^\star\) & Ideal Gibbs target, \(q_t^\star(x)\propto p_0(x)\exp\{\beta a_t(x)\}\). \\
\(\widehat q_t\) & Learned Gibbs target, \(\widehat q_t(x)\propto p_0(x)\exp\{\beta_t\widehat a_t(x)\}\). \\
\(\nu_t\) & Realized terminal distribution of the guided sampler. \\
\(\tau_t\), \(\tau\) & Improvement threshold. \\
\(\mathrm{EI}_t(x)\), \(a_\tau(x)\) & Expected Improvement score, \(a_\tau(x)=\E[(Y-\tau)_+\mid X=x]\). \\
\(I_\eta(x)\) & Noiseless EI/progress utility, \(I_\eta(x)=(\eta-\rho(x))_+\). \\
\(\mu_t(x)\), \(\sigma_t(x)\) & Surrogate predictive mean and uncertainty used in UCB-style scores. \\
\(\beta_{\rm UCB}\), \(b_t\) & UCB exploration multiplier. \\
\(c_\phi(x)\), \(c_\tau^\star(x)\) & LFBO classifier and its population minimizer. \\
\(z_i\) & Threshold label, \(z_i=\mathbbm{1}\{y_i\ge\tau_t\}\). \\
\(u(y_i)\), \(w_i(\tau)\) & Improvement weights such as \(u(y_i)=(y_i-\tau_t)_+\) and \(w_i(\tau)=(Y_i-\tau)_+\). \\
\(\mathcal L_\tau\), \(\mathcal L(q)\), \(\mathcal L_t^S(q)\) & LFBO, diffusion-alignment, or proper-scoring objective. \\
\(S(q,x_i)\) & Proper scoring rule used in GenBO-style alignment. \\

\addlinespace[0.4em]
\multicolumn{2}{@{}l}{\textbf{Algorithmic sampling quantities}}\\
\addlinespace[0.2em]
\(T\) & Number of BO rounds. \\
\(N\), \(N_t\) & Number of terminal guided candidates sampled at a round. \\
\(K\), \(K_t\) & Number of top-ranked guided candidates evaluated. \\
\(J\), \(J_{t,a}\) & Number of prior-refresh samples evaluated. \\
\(B\) & Batch size of a matched no-filtering random-search baseline, often \(B=J+K\). \\
\(C_t\) & Terminal guided candidate pool at round \(t\). \\
\(X_{t,1},\ldots,X_{t,N_t}\) & Guided candidate samples from \(\nu_t\). \\
\(\argtopk\) & Operator returning the top-\(K\) candidates under a score. \\
\(\F_{t-1}\) & Filtration containing all information before round \(t\). \\
\(M_t\) & Proposal, resampling, particle, or internal sampler-training budget. \\
\(M_{t,{\rm eff}}\) & Effective sample size in sampler-learning certificates. \\

\addlinespace[0.4em]
\multicolumn{2}{@{}l}{\textbf{Terminal-mass interface and search exponents}}\\
\addlinespace[0.2em]
\(A\subseteq\X\) & Generic search-relevant set, e.g., \(\Xstar\), \(\U_\gamma\), \(\U_{a\eta}\), or \(\mathcal C_{T,\eta}\). \\
\(\ell_{t,A}\) & Certified lower bound on target mass \(\widehat q_t(A)\). \\
\(d_{t,A}\) & RMS terminal set-mass error between \(\widehat q_t(A)\) and \(\nu_t(A)\). \\
\(\omega_{t,A}\) & Slack used to convert target-mass and sampler-error bounds into terminal-mass events. \\
\(m_t^\star\) & Slacked terminal optimizer mass, \(m_t^\star=[\ell_{t,\Xstar}-\omega_{t,\Xstar}]_+\). \\
\(\lambda_{{\rm BO},t}\), \(\alpha_t^\star\) & Per-round BO search exponent, typically \(J_{t,a}\Lambda(p_0^\star)+N_t\Lambda(m_t^\star)\). \\
\(\good_t\), \(\good_{t,A}\) & Good event on which the learned score gives the required ranking or target-mass certificate. \\
\(u_t\) & Failure probability of the good event, \(u_t=\Pbb(\good_t^c)\). \\
\(T_{\rm act}\) & Score activation time after which the deterministic score-learning envelope is below the required margin. \\
\(\Delta_f^\star\) & Finite-domain optimizer gap, \(f^\star-\max_{x\notin\Xstar}f(x)\). \\
\(\underline q_\star^{\rm Gibbs}(\beta_t)\) & Lower bound on \(\widehat q_t(\Xstar)\) under a Gibbs tilt and a valid score gap. \\

\addlinespace[0.4em]
\multicolumn{2}{@{}l}{\textbf{Score-learning notation}}\\
\addlinespace[0.2em]
\(\Psi:\X\to[0,1]^{d_a}\) & Representation map used by the score learner. \\
\(d_a\) & Dimension of the score-learning feature representation. \\
\(\zeta_t\) & Acquisition index determining the score slice \(a_t\). \\
\(\mathcal Z\) & Compact index set for acquisition slices. \\
\(G\) & Link mapping a primitive score to the acquisition score. \\
\(g_\zeta^\star\), \(\widehat g_t\) & Primitive target score slice and its clipped-ReLU ERM estimate. \\
\(B^{s_a}_{p,q}\) & Besov function space for primitive score slices. \\
\(s_a\), \(\bar s_a\) & Besov smoothness and effective smoothness used in the learning rate. \\
\(n_{t,a}\) & Number of labeled reservoir samples available to the acquisition learner. \\
\(\bar\varepsilon_t\) & Deterministic score-learning envelope. \\
\(\vartheta_a\) & Global score-learning exponent, \(\vartheta_a=\bar s_a/(2\bar s_a+d_a)\). \\
\(c_a\) & Generic logarithmic exponent or constant in score-learning rates; its value may change from line to line. \\
\(P_\Psi\) & Pushforward law of \(p_0\) under \(\Psi\). \\
\(p_{\min}(P_\Psi)\) & Minimum atom of the finite feature law \(P_\Psi\). \\
\(\Pdim(\mathcal G_W)\) & Pseudodimension of the ReLU sieve class. \\
\(\mathcal G_W\), \(W\), \(W_n\) & Clipped-ReLU network class, sieve size, and balanced sieve size. \\

\addlinespace[0.4em]
\multicolumn{2}{@{}l}{\textbf{Sampler-certification notation}}\\
\addlinespace[0.2em]
\(\TV(\cdot,\cdot)\), \(\KL(\cdot\|\cdot)\) & Total variation distance and Kullback--Leibler divergence. \\
\(P_t^\star\), \(P_{\theta_t}\) & Target and learned path laws in DiBO/RTB-style path-space certificates. \\
\(\Pi_t\) & Terminal-state or decoder map from path or latent space to \(\X\). \\
\(\mathcal E_t^{\rm RTB}\) & RTB path-space divergence certificate or its upper bound. \\
\(b_{t,A,{\rm dec}}\) & Residual decoder or path-to-terminal bias. \\
\(R_t(x)\) & Positive tilt defining a GenBO-style target, with \(\widehat q_t(x)\propto p_0(x)R_t(x)\). \\
\(Z_t\) & Normalizing constant for a positive tilt \(R_t\). \\
\(\rho_t\) & Proposal distribution in the GenBO log-score certificate; distinct from the gap \(\rho(x)\). \\
\(W_t(x)\) & Importance weight used in proposal correction. \\
\(D_S\), \(c_S\) & Proper-scoring divergence and its TV-strong-properness constant. \\
\(\mathcal E_t^S\), \(b_{t,{\rm opt}}\) & Proper-score optimization/statistical error and residual optimization bias. \\
\(\bar q_t\), \(\bar\nu_t\) & Lifted target and learned distributions before decoding. \\
\(r_t\) & Roundwise distribution-learning error rate in sampler certificates. \\
\(g_t\) & Generic proposal distribution in the proposal-corrected guided sampler. \\
\(\nu_t^M\), \(\nu_t^{M,g}\) & Self-normalized proposal-resampling terminal laws. \\
\(A_{\rm osc}\) & Uniform oscillation bound for the frozen learned acquisition. \\
\(W_{\min}\), \(W_{\max}\) & Lower and upper bounds on proposal-correction weights. \\

\addlinespace[0.4em]
\multicolumn{2}{@{}l}{\textbf{Threshold mass lift and local geometry}}\\
\addlinespace[0.2em]
\(\eta\), \(\eta_t\) & Threshold gap or best-so-far true gap, \(\eta_t=\min_{x\in\D_t}\rho(x)\). \\
\(\eta_{\min}\) & Irreducible floor for threshold contraction. \\
\(a\in(0,1)\) & Multiplicative progress factor; hitting \(\U_{a\eta}\) improves gap \(\eta\) to at most \(a\eta\). \\
\(\zeta_t=(\eta_t-\eta_{\min})_+\) & Threshold-excess variable used locally in Appendix~\ref{app:threshold_contraction_theorem}; distinct from the acquisition index \(\zeta_t\). \\
\(\dist_\Psi(x,\Xstar)\) & Distance from \(x\) to the optimizer set in representation space. \\
\(\chi_f\) & Objective sharpness exponent near \(\Xstar\). \\
\(d_0\) & Local prior small-ball dimension near \(\Xstar\). \\
\(\alpha_0=d_0/\chi_f\) & Near-optimal prior-mass exponent. \\
\(q_\eta^{\rm EI}\) & Ideal EI-tilted threshold target, \(q_\eta^{\rm EI}(x)\propto p_0(x)(\eta-\rho(x))_+\). \\
\(Z_0(\eta)\) & Normalizing constant for \(q_\eta^{\rm EI}\). \\
\(m_{\rm EI}(\eta,a)\) & EI-guided mass on the multiplicative progress set \(\U_{a\eta}\). \\
\(p_{\rm prog}\) & Lower bound on the conditional probability of making multiplicative threshold progress. \\
\(\lambda_{\rm rand}(\eta,a)\), \(\lambda_{\rm BO}(\eta,a)\) & Threshold-level random-search and BO search exponents. \\

\addlinespace[0.4em]
\multicolumn{2}{@{}l}{\textbf{Local active-learning notation}}\\
\addlinespace[0.2em]
\(\mathcal C_{T,\eta}\) & Local critical score band outside \(\U_\eta\). \\
\(Q_T^{\rm act}\) & Average labeled design distribution used for local score learning. \\
\(\beta_{\rm act}\) & Active-mass exponent defined by \(Q_T^{\rm act}(\mathcal C_{T,\eta})\gtrsim \eta^{\beta_{\rm act}}\). \\
\(s_{\rm loc}\), \(d_{\rm loc}\) & Local smoothness and local effective dimension for score learning. \\
\(\vartheta_{\rm loc}\) & Local score-learning exponent, \(\vartheta_{\rm loc}=s_{\rm loc}/(2s_{\rm loc}+d_{\rm loc})\). \\
\(\chi_a\) & Score-margin sharpness exponent linking score error to regret. \\
\(c_\Delta\), \(c_{\mathcal C}\) & Local margin and critical-band constants. \\
\(S_T(\eta)\) & Search miss probability for \(\U_{\eta/2}\). \\
\(b_T(\eta)\) & Coarse-screening failure envelope. \\
\(u_T^{\rm loc}(\eta)\), \(v_T^{\rm act}(\eta)\) & Local learning and active-mass failure envelopes. \\

\addlinespace[0.4em]
\multicolumn{2}{@{}l}{\textbf{Experimental diagnostics}}\\
\addlinespace[0.2em]
\(\widehat\nu_t(\U_\gamma)\) & Monte Carlo estimate of posterior or guided mass on \(\U_\gamma\). \\
\(m\) & Audit sample size used to estimate \(\widehat\nu_t(\U_\gamma)\). \\
\(\xi\) & In experiments, also used as a diffusion-weight interpolation parameter; this is overloaded with the observation-noise notation. Consider renaming it to \(\alpha_{\rm prior}\) or \(\xi_{\rm interp}\) for clarity. \\

\end{longtable}
\endgroup

\section{Extended background}

\subsection{Likelihood-free BO.}\label{app:lfbo}
While classical BO estimates acquisition function $\hat a_t$ through Bayesian surrogate models on $f$ such as Gaussian Process (GP; \citep{rasmussen2003gaussian}), likelihood-free BO (LFBO; \citep{song2022general}) directly models $\hat a_t \approx a_t$ via weighted classification reformulation. This is particularly useful for non-Bayesian surrogate models such as neural network~\citep{lecun2015deep}. 
For any threshold $\tau\in\R$, define weights $w_i(\tau):=(Y_i-\tau)_+$ and train a classifier $c_\phi(x)\in(0,1)$ by minimizing
\begin{align}
\mathcal{L}_\tau(\phi)
:=
\frac1n\sum_{i=1}^n
\Bigl[
- w_i(\tau)\log c_\phi(X_i)
-\log(1-c_\phi(X_i))
\Bigr].
\label{eq:ei_logistic_loss}
\end{align}
At the population level, the unique minimizer satisfies
\begin{align}
c_\tau^\star(x)=\frac{a_\tau(x)}{1+a_\tau(x)},
\qquad
a_\tau(x):=\E[(Y-\tau)_+\mid X=x] = \frac{c_\tau^\star(x)}{(1-c_\tau^\star(x))},
\label{eq:lfbo_population_identity}
\end{align}
so the odds transform recovers Expected Improvement (EI; \citep{mockus1975application}) exactly. LFBO can handle both EI and PI, and we employ EI mainly.

\subsection{Diffusion models in one paragraph}
A diffusion model starts from data $X_0\sim p_0$, gradually corrupts it into an easier reference distribution, and then learns to approximately reverse that corruption. In continuous-state diffusion models the forward process is typically Gaussian and the reverse dynamics are parameterized through a score network. In discrete domains the same idea is implemented with a continuous-time Markov chain rather than with Gaussian perturbations. In both cases, the role of the pretrained backbone is the same: it learns how to sample from the unconditional data distribution $p_0$.

\subsection{Continuous-time discrete diffusion}
For structured discrete objects we use a continuous-time Markov chain (CTMC) $\{X_s\}_{s\in[0,1]}$ on $\X$ with time-dependent generator $Q_s$. Starting from $X_0\sim p_0$, the forward process produces noisy marginals
\[
\nu_s(x):=\Pr(X_s=x),
\qquad s\in[0,1].
\]
At large noise levels the law $\nu_s$ becomes easy to sample, while at $s=0$ it coincides with the data distribution $p_0$. If $\nu_s(x)>0$, the exact reverse-time CTMC has off-diagonal rates
\begin{equation}
\label{eq:primer_reverse_rates}
R_s^0(x,y)
=
Q_s(y,x)\frac{\nu_s(y)}{\nu_s(x)},
\qquad x\neq y.
\end{equation}
A pretrained discrete diffusion model approximates these reverse rates, or an equivalent reverse transition operator, so that starting from a noisy sample near $\nu_1$ and simulating the reverse chain approximately recovers $X_0\sim p_0$.

\section{Formal Assumptions}\label{app:formal_assumptions}
This appendix gives the formal assumptions used by the main theorems. We keep the assumptions modular so that different acquisition learners and different samplers can be substituted independently.

\begin{assumption}[Finite aligned domain]\label{ass:finite_alignment_app}
The domain $\X=\supp(p_0)$ is finite. For the exact-optimization results, $p_0^\star:=p_0(\Xstar)>0$. For a near-optimal set $A=\U_\gamma$, we assume $p_0(A)>0$.
\end{assumption}

\begin{assumption}[Calibrated ideal score]\label{ass:calibrated_score_app}
For every round $t$, the ideal score $a_t:\X\to\R$ is order-preserving with $f$ and has a uniform inverse link: there exists $L_a<\infty$ such that, for all $x,x'\in\X$ with $f(x)\ge f(x')$,
\begin{equation}
\label{eq:calibrated_score_app}
    a_t(x)\ge a_t(x'),
    \qquad
    f(x)-f(x')\le L_a\{a_t(x)-a_t(x')\}.
\end{equation}
\end{assumption}

\begin{assumption}[Score reservoir]\label{ass:score_reservoir_app}
At round $t$, the score learner has access to a labeled reservoir
\begin{equation}
\label{eq:score_reservoir_app}
    \D^{\rm explr}_{t,a}:=\{(X_{i,a},Y_{i,a})\}_{i=1}^{n_{t,a}},
    \qquad
    X_{i,a}\stackrel{\rm i.i.d.}{\sim}p_0.
\end{equation}
The acquisition index $\zeta_t$ used to define the ideal score is measurable with respect to the past and independent of the current reservoir sample conditional on the past. Eventually $n_{t,a}\ge c_a^{\rm samp}t$ for some constant $c_a^{\rm samp}>0$.
\end{assumption}

\begin{assumption}[Primitive Besov score model]\label{ass:besov_score_app}
There is a fixed representation $\Psi:\X\to[0,1]^{d_a}$ and a compact index set $\mathcal Z$ such that
\begin{equation}
\label{eq:score_slice_app}
    a_t(x)=G\{g^\star_{\zeta_t}(\Psi(x))\},
    \qquad
    \zeta_t\in\mathcal Z.
\end{equation}
The link $G:I_G\to\R$ is $L_G$-Lipschitz on a compact interval $I_G$. The slice family satisfies, for some $s_a>d_a/p$,
\begin{equation}
\label{eq:besov_ball_app}
    g^\star_\zeta\in B^{s_a}_{p,q}([0,1]^{d_a})\cap L_\infty,
    \qquad
    \sup_{\zeta\in\mathcal Z}\|g^\star_\zeta\|_{B^{s_a}_{p,q}}
    \le C_{\rm Bes},
\end{equation}
and the range of every $g^\star_\zeta$ lies in $I_G$.
\end{assumption}

\begin{assumption}[Clipped ReLU ERM]\label{ass:relu_erm_app}
For each sieve budget $W\ge1$, let $\mathcal G_W$ be the class of clipped ReLU networks $g:[0,1]^{d_a}\to I_G$ with depth at most $C_L\log(eW)$, width at most $C_UW$, total number of trainable parameters at most $C_SW$, and parameter magnitudes bounded by $W^{C_B}$.
Given $n$ reservoir samples, the learner computes an empirical risk minimizer
\begin{equation}
\label{eq:generic_erm_app}
    \widehat g_{t,W}\in\argmin_{g\in\mathcal G_W}
    \frac1{n_{t,a}}\sum_{i=1}^{n_{t,a}}
    \ell_{\zeta_t}\{Y_{i,a},g(\Psi(X_{i,a}))\}.
\end{equation}
For each $\zeta$, the population risk $\mathcal R_\zeta(g):=\E[\ell_\zeta\{Y,g(\Psi(X))\}]$ is uniquely minimized at $g^\star_\zeta$. Uniformly over $\zeta\in\mathcal Z$, the excess risk has local quadratic curvature in $L_2(P_\Psi)$, and the loss has a sub-exponential Lipschitz envelope on $I_G$, so that the localized finite-pseudodimension ERM oracle inequality in Lemma~\ref{lem:generic_oracle_ineq_app} applies.
\end{assumption}

\begin{assumption}[Target-relative weak terminal mass consistency]
\label{ass:weak_terminal_mass_app}
After $\widehat a_t$ is frozen, define the Gibbs target
\begin{equation}
\label{eq:gibbs_target_formal_app}
    \widehat q_t(x)
    :=
    \frac{p_0(x)\exp\{\beta_t\widehat a_t(x)\}}
    {\sum_{z\in\X}p_0(z)\exp\{\beta_t\widehat a_t(z)\}} .
\end{equation}
The exploitation sampler returns a candidate pool
$C_t=\{X_{t,1},\ldots,X_{t,N_t}\}$.
There exists a random terminal law $\nu_t\in\Delta(\X)$ such that, conditional on $\nu_t$ and the frozen score information, the candidates are i.i.d. from $\nu_t$.

For each set $A\subseteq\X$ used in the analysis, there is a deterministic envelope $d_{t,A}\ge0$ such that
\begin{equation}
\label{eq:weak_rms_mass_app}
    \E\left[
    \{\widehat q_t(A)-\nu_t(A)\}_+^2
    \mid \F_{t-1}
    \right]
    \le d_{t,A}^2 .
\end{equation}
Equivalently, it is sufficient to verify the stronger conditional statement after freezing $\widehat a_t$. No deterministic lower bound on $\nu_t(A)$ is assumed directly; target-mass lower bounds are proved separately on good score events.
\end{assumption}

\section{Generic Score Learning from Primitive Besov Assumptions}\label{app:generic_score_learning}
This section proves the generic score-learning rate used in Theorem~\ref{thm:generic_score_rate_informal}. The key point is that the neural-network capacity and approximation bounds are consequences of Assumptions~\ref{ass:besov_score_app} and~\ref{ass:relu_erm_app}, not assumptions.

\begin{proposition}[ReLU capacity and Besov-to-H\"older approximation]
\label{prop:relu_capacity_approx_app}
Under Assumptions~\ref{ass:besov_score_app} and~\ref{ass:relu_erm_app}, fix any
\[
    0<\bar s_a<s_a-d_a/p .
\]
Then there exist constants $C_{\rm vc},C_{\rm app}>0$ and $b_a\ge0$ such that for all $W\ge2$,
\begin{equation}
\label{eq:derived_capacity_approx_app}
    \Pdim(\mathcal G_W)
    \le
    C_{\rm vc}W(\log W)^2,
    \qquad
    \sup_{\zeta\in\mathcal Z}\inf_{g\in\mathcal G_W}
    \|g-g^\star_\zeta\|_{L_2(P_\Psi)}
    \le
    C_{\rm app}W^{-\bar s_a/d_a}(\log W)^{b_a}.
\end{equation}
If a direct $L_\infty$ ReLU approximation theorem with exponent $s_a$ is available for the chosen score class, then $\bar s_a$ may be replaced by $s_a$ throughout.
\end{proposition}

\begin{proof}
The pseudodimension bound follows from standard capacity estimates for ReLU networks with $O(W)$ trainable parameters, logarithmic depth, and polynomially bounded weights. Clipping does not increase pseudodimension up to constants.

For approximation, Assumption~\ref{ass:besov_score_app} gives
$g^\star_\zeta\in B^{s_a}_{p,q}$ with $s_a>d_a/p$. By the Besov--H\"older embedding, for every
$0<\bar s_a<s_a-d_a/p$ the family is uniformly bounded in a H\"older ball of smoothness $\bar s_a$. Standard ReLU approximation of H\"older functions gives a network
$\widetilde g_{\zeta,W}$ with the stated architecture such that
\[
    \|\widetilde g_{\zeta,W}-g^\star_\zeta\|_{L_\infty([0,1]^{d_a})}
    \le
    C W^{-\bar s_a/d_a}(\log W)^{b_a}.
\]
Clipping to $I_G$ is $1$-Lipschitz and leaves the target unchanged because $g^\star_\zeta$ takes values in $I_G$. Since $P_\Psi$ is a probability measure, the same bound holds in $L_2(P_\Psi)$. Taking the infimum over $\mathcal G_W$ and the supremum over $\zeta$ proves the claim.
\end{proof}

\begin{lemma}[Localized ERM oracle inequality]\label{lem:generic_oracle_ineq_app}
Under Assumption~\ref{ass:relu_erm_app}, there exists a constant $C_{\rm orc}>0$ such that for every fixed $W\ge2$, every $x>0$, and every $t$, with probability at least $1-2e^{-x}$ conditional on $\zeta_t$,
\begin{align}
\label{eq:generic_oracle_ineq_app}
    \|\widehat g_{t,W}-g^\star_{\zeta_t}\|_{L_2(P_\Psi)}^2
    &\le
    C_{\rm orc}
    \left[
    \inf_{g\in\mathcal G_W}\|g-g^\star_{\zeta_t}\|_{L_2(P_\Psi)}^2
    +
    \frac{W(\log W)^2(\log n_{t,a}+x)}{n_{t,a}}
    \right].
\end{align}
\end{lemma}

\begin{proof}
This is the standard localized ERM oracle inequality for finite-pseudodimension classes under local quadratic curvature and a sub-exponential Lipschitz envelope. Assumption~\ref{ass:relu_erm_app} supplies the curvature and envelope conditions; Proposition~\ref{prop:relu_capacity_approx_app} supplies the pseudodimension bound. Applying the localized inequality to the excess-loss class and absorbing constants yields Eq.~\eqref{eq:generic_oracle_ineq_app}.
\end{proof}

\begin{proposition}[Derived balanced sieve size]
\label{prop:balanced_width}
Fix $0<\bar s_a<s_a-d_a/p$ and define the deterministic bias--variance envelope
\begin{equation}
\label{eq:bias_variance_envelope_app}
    \mathfrak B_n(W)
    :=
    W^{-2\bar s_a/d_a}(\log W)^{b_1}
    +
    \frac{W(\log W)^{b_2}\log n}{n},
    \qquad W\ge2,
\end{equation}
where $b_1,b_2\ge0$ are fixed constants. Then there exists a deterministic choice
\[
    W_n=\widetilde O\!\left(n^{d_a/(2\bar s_a+d_a)}\right)
\]
such that
\begin{equation}
\label{eq:balanced_width_rate_app}
    \mathfrak B_n(W_n)
    \le
    C n^{-2\bar s_a/(2\bar s_a+d_a)}(\log n)^{c_a}.
\end{equation}
\end{proposition}

\begin{proof}
Ignoring logarithmic factors, the two terms in Eq.~\eqref{eq:bias_variance_envelope_app} balance when
$W^{-2\bar s_a/d_a}\asymp W/n$, i.e.
$W\asymp n^{d_a/(2\bar s_a+d_a)}$. Taking
$W_n=\lfloor n^{d_a/(2\bar s_a+d_a)}\rfloor\vee2$ and absorbing logarithmic factors gives the displayed upper bound.
\end{proof}

\begin{lemma}[Finite-support $L_2$ to $L_\infty$ conversion]\label{lem:finite_support_l2_linf_app}
Let $P$ be a probability measure supported on a finite set $\mathcal Z_0$ and let $p_{\min}(P):=\min_{z\in\mathcal Z_0}P(\{z\})>0$. Then every function $h:\mathcal Z_0\to\R$ satisfies
\begin{equation}
\label{eq:l2_linf_conversion_app}
    \|h\|_{L_\infty(\mathcal Z_0)}
    \le
    p_{\min}(P)^{-1/2}\|h\|_{L_2(P)}.
\end{equation}
\end{lemma}

\begin{proof}
Let $z^\star$ maximize $|h(z)|$. Then
\[
\|h\|_{L_2(P)}^2
=\sum_z |h(z)|^2P(\{z\})
\ge |h(z^\star)|^2P(\{z^\star\})
\ge p_{\min}(P)\|h\|_{L_\infty}^2.
\]
Taking square roots proves the claim.
\end{proof}

\begin{lemma}[Positive atoms of the feature law]
\label{lem:positive_atoms_feature}
Under Assumption~\ref{ass:finite_alignment_app}, the pushforward law
$P_\Psi:=\Psi_\#p_0$ has finite support and strictly positive minimum atom:
\[
p_{\min}(P_\Psi)
:=
\min_{z\in{\rm supp}(P_\Psi)}P_\Psi(\{z\})
>0.
\]
Moreover,
\[
p_{\min}(P_\Psi)
\ge
\min_{x\in\mathcal X}p_0(x).
\]
\end{lemma}

\begin{proof}
Since $\mathcal X=\operatorname{supp}(p_0)$ is finite,
\[
p_{0,\min}:=\min_{x\in\mathcal X}p_0(x)>0.
\]
The set $\Psi(\mathcal X)$ is finite.  For any
$z\in\Psi(\mathcal X)$,
\[
P_\Psi(\{z\})
=
\sum_{x:\Psi(x)=z}p_0(x)
\ge
p_{0,\min}.
\]
This proves the claim.
\end{proof}

\begin{theorem}[Generic Besov score-learning rate]
\label{thm:generic_score_learning_rate}
Under Assumptions~\ref{ass:finite_alignment_app}, \ref{ass:score_reservoir_app}, \ref{ass:besov_score_app}, and~\ref{ass:relu_erm_app}, fix any
\[
    0<\bar s_a<s_a-d_a/p .
\]
Let $\widehat g_t:=\widehat g_{t,W_{n_{t,a}}}$, where $W_{n_{t,a}}$ is the balanced width from Proposition~\ref{prop:balanced_width}, and define
$\widehat a_t(x):=G\{\widehat g_t(\Psi(x))\}$.
Then there exists $c_a\ge0$ such that
\begin{equation}
\label{eq:generic_score_expectation_app}
    \E\|\widehat a_t-a_t\|_\infty
    =
    O\!\left(n_{t,a}^{-\vartheta_a}(\log n_{t,a})^{c_a}\right),
    \qquad
    \vartheta_a=\frac{\bar s_a}{2\bar s_a+d_a}.
\end{equation}
Moreover, for every fixed $\kappa_a>0$ there are constants $C_\varepsilon,C_u>0$ such that
\begin{equation}
\label{eq:generic_score_hp_app}
    \bar\varepsilon_t:=C_\varepsilon n_{t,a}^{-\vartheta_a}(\log n_{t,a})^{c_a}
    \quad\Longrightarrow\quad
    \Pbb\{\|\widehat a_t-a_t\|_\infty>\bar\varepsilon_t\}
    \le
    u_t,
    \qquad
    u_t\le C_u n_{t,a}^{-1-\kappa_a}.
\end{equation}
The constants may depend on
\[
    p_{\min}(P_\Psi)
    :=
    \min_{z\in\supp(P_\Psi)}P_\Psi(\{z\}),
\]
which is positive because $\X$ is finite.
\end{theorem}

\begin{proof}
By Lemma~\ref{lem:generic_oracle_ineq_app} and Proposition~\ref{prop:relu_capacity_approx_app}, with probability at least $1-2e^{-x}$,
\[
\|\widehat g_t-g^\star_{\zeta_t}\|_{L_2(P_\Psi)}^2
\le
C\left[
W^{-2\bar s_a/d_a}(\log W)^{b_1}
+
\frac{W(\log W)^{b_2}(\log n_{t,a}+x)}{n_{t,a}}
\right]
\]
with $W=W_{n_{t,a}}$. Proposition~\ref{prop:balanced_width} gives
\[
\|\widehat g_t-g^\star_{\zeta_t}\|_{L_2(P_\Psi)}
\le
C n_{t,a}^{-\vartheta_a}(\log n_{t,a}+x)^{c_a}
\]
after increasing $c_a$ if necessary. 
By Lemma~\ref{lem:positive_atoms_feature}, $P_\Psi$ has finite support with
$p_{\min}(P_\Psi)>0$.  Lemma~\ref{lem:finite_support_l2_linf_app} therefore gives
\[
\|\widehat g_t-g^\star_{\zeta_t}\|_{L_\infty(\Psi(\mathcal X))}
\le
p_{\min}(P_\Psi)^{-1/2}
\|\widehat g_t-g^\star_{\zeta_t}\|_{L_2(P_\Psi)} .
\]
Choosing $x=(2+\kappa_a)\log n_{t,a}$ proves the high-probability envelope. Integrating the bounded tail gives the expectation bound.
\end{proof}

\begin{remark}[Size of the finite-support $L_2\!\to\!L_\infty$ constant]
\label{rem:l2_linf_constants}
The finite-support conversion in Lemma~\ref{lem:finite_support_l2_linf_app} is a formal global certificate.  Its hidden constant contains $p_{\min}(P_\Psi)^{-1/2}$.  In large structured domains this constant can be very large, because $\min_{x\in\mathcal X}p_0(x)$ may be exponentially small. Therefore Corollary~\ref{cor:formal_clean_rate} should be read as a conservative global no-regret/safety guarantee, not as the main explanation of finite-budget empirical acceleration.  The finite-budget acceleration mechanism is instead the set-mass term: if the sampler places substantial mass on the relevant $\gamma$-optimal or threshold-progress set, then best-of-$N$ search improves even when a global sup-norm learning bound is loose.  The local threshold and active learning results in Appendix~\ref{app:unified_threshold_acceleration} are designed to avoid using a global $p_{\min}(P_\Psi)^{-1/2}$ constant.
\end{remark}

\section{Score Calibration and Acquisition-Learning Instances}\label{app:acq_learning_instances}
\subsection{Generic regret decomposition}\label{app:regret_decomp_proof}
\begin{lemma}[Near-optimal hit implies low regret]\label{lem:near_opt_generic_app}
Under Assumption~\ref{ass:calibrated_score_app}, if $\D_T\cap\U_{\gamma_T}\neq\varnothing$, then
\begin{equation}
\label{eq:near_opt_generic_app}
    f^\star-f(\widehat x_T)
    \le
    \gamma_T+2L_a\|\widehat a_T-a_T\|_\infty.
\end{equation}
\end{lemma}

\begin{proof}
Let $x_\gamma\in\D_T\cap\U_{\gamma_T}$ and $\epsilon_T:=\|\widehat a_T-a_T\|_\infty$. Because $\widehat x_T$ maximizes $\widehat a_T$ over $\D_T$,
$\widehat a_T(\widehat x_T)\ge\widehat a_T(x_\gamma)$, hence
$a_T(\widehat x_T)\ge a_T(x_\gamma)-2\epsilon_T$.
If $f(x_\gamma)-f(\widehat x_T)>2L_a\epsilon_T$, then Assumption~\ref{ass:calibrated_score_app} gives
$a_T(x_\gamma)-a_T(\widehat x_T)>2\epsilon_T$, a contradiction. Thus $f(x_\gamma)-f(\widehat x_T)\le2L_a\epsilon_T$. Since $f^\star-f(x_\gamma)\le\gamma_T$, the claim follows.
\end{proof}

\begin{theorem}[Formal regret decomposition]\label{thm:formal_regret_decomp}
Under Assumptions~\ref{ass:finite_alignment_app} and~\ref{ass:calibrated_score_app}, for any $\gamma_T\ge0$,
\begin{equation}
\label{eq:formal_regret_decomp_app}
    r_T
    \le
    \gamma_T+2L_a\E\|\widehat a_T-a_T\|_\infty
    +R_f\Pbb(\D_T\cap\U_{\gamma_T}=\varnothing),
    \qquad
    R_f:=f^\star-\min_{x\in\X}f(x).
\end{equation}
\end{theorem}

\begin{proof}
Let $E_T:=\{\D_T\cap\U_{\gamma_T}\neq\varnothing\}$. On $E_T$, apply Lemma~\ref{lem:near_opt_generic_app}. On $E_T^c$, use the trivial bound $f^\star-f(\widehat x_T)\le R_f$. Taking expectations proves the result.
\end{proof}

\subsection{Expected Improvement calibration}\label{app:ei_calibration}
\begin{assumption}[EI noise band]\label{ass:ei_noise_band_app}
The oracle obeys $Y=f(X)+\varepsilon$, where $\varepsilon$ is independent of $X$ with survival function $\bar F_\varepsilon(u):=\Pbb(\varepsilon\ge u)$. Thresholds lie in a compact interval $\mathcal T=[\tau_{\min},\tau_{\max}]$. Let $f_{\min}:=\min_x f(x)$ and $f_{\max}:=\max_x f(x)$. Assume
\begin{equation}
\label{eq:ei_survival_band_app}
    A_{\max}:=\sup_{u\in[\tau_{\min}-f_{\max},\tau_{\max}-f_{\min}]}
    \int_u^\infty\bar F_\varepsilon(v)\,\dd v<\infty,
\end{equation}
and
\begin{equation}
\label{eq:ei_survival_lower_app}
    \underline s_{\mathcal T}:=
    \inf_{u\in[\tau_{\min}-f_{\max},\tau_{\max}-f_{\min}]}
    \bar F_\varepsilon(u)>0.
\end{equation}
\end{assumption}

\begin{proposition}[EI is calibrated]\label{prop:ei_calibrated_app}
Under Assumption~\ref{ass:ei_noise_band_app}, for $a_\tau(x):=\E[(Y-\tau)_+\mid X=x]$, the score $a_\tau$ has the same argmax set as $f$, and for all $x,x'\in\X$,
\begin{equation}
\label{eq:ei_calibrated_app}
    \underline s_{\mathcal T}|f(x)-f(x')|
    \le
    |a_\tau(x)-a_\tau(x')|
    \le
    |f(x)-f(x')|.
\end{equation}
Thus Assumption~\ref{ass:calibrated_score_app} holds with $L_a=1/\underline s_{\mathcal T}$.
\end{proposition}

\begin{proof}
By the additive noise model,
\[
 a_\tau(x)=\int_{\tau-f(x)}^\infty \bar F_\varepsilon(v)\,\dd v.
\]
Assume $f(x)\ge f(x')$. Then
\[
 a_\tau(x)-a_\tau(x')=
 \int_{\tau-f(x)}^{\tau-f(x')}\bar F_\varepsilon(v)\,\dd v.
\]
The integration interval lies in the EI band. Using $\underline s_{\mathcal T}\le\bar F_\varepsilon(v)\le1$ gives Eq.~\eqref{eq:ei_calibrated_app}. The argmax claim follows from strict monotonicity.
\end{proof}

\begin{corollary}[Floored log-EI is calibrated]
\label{cor:log_ei_calibrated_app}
Under Assumption~\ref{ass:ei_noise_band_app}, fix
\(\varepsilon_{\rm EI}>0\) and define
\[
    g_\tau(x)
    :=
    \log\{a_\tau(x)+\varepsilon_{\rm EI}\},
    \qquad
    a_\tau(x)=\E[(Y-\tau)_+\mid X=x].
\]
Then \(g_\tau\) is order-preserving with \(f\). Moreover, if
\[
    A_{\mathcal T,\varepsilon}
    :=
    \sup_{\tau\in\mathcal T,x\in\X}
    \{a_\tau(x)+\varepsilon_{\rm EI}\}<\infty,
\]
then for all \(x,x'\in\X\) with \(f(x)\ge f(x')\),
\[
    f(x)-f(x')
    \le
    \frac{A_{\mathcal T,\varepsilon}}{\underline s_{\mathcal T}}
    \{g_\tau(x)-g_\tau(x')\}.
\]
Thus Assumption~\ref{ass:calibrated_score_app} holds for the floored log-EI
score with
\[
    L_a
    =
    \frac{A_{\mathcal T,\varepsilon}}{\underline s_{\mathcal T}}.
\]
\end{corollary}

\begin{proof}
Since \(z\mapsto\log(z+\varepsilon_{\rm EI})\) is strictly increasing,
\(g_\tau\) is order-preserving whenever \(a_\tau\) is. For
\(f(x)\ge f(x')\), Proposition~\ref{prop:ei_calibrated_app} gives
\[
    a_\tau(x)-a_\tau(x')
    \ge
    \underline s_{\mathcal T}\{f(x)-f(x')\}.
\]
Also, for \(u\ge v>0\),
\[
    \log u-\log v
    \ge
    \frac{u-v}{u}.
\]
Apply this with
\(u=a_\tau(x)+\varepsilon_{\rm EI}\) and
\(v=a_\tau(x')+\varepsilon_{\rm EI}\). Since
\(u\le A_{\mathcal T,\varepsilon}\),
\[
    g_\tau(x)-g_\tau(x')
    \ge
    \frac{a_\tau(x)-a_\tau(x')}{A_{\mathcal T,\varepsilon}}
    \ge
    \frac{\underline s_{\mathcal T}}{A_{\mathcal T,\varepsilon}}
    \{f(x)-f(x')\}.
\]
Rearranging proves the claim.
\end{proof}

\begin{remark}[Noiseless EI is not globally calibrated]
\label{rem:noiseless_ei_not_global}
If the observation noise is identically zero, then
\[
a_\tau(x)=(f(x)-\tau)_+ .
\]
For thresholds $\tau>\min_x f(x)$, all points with $f(x)\le \tau$ receive the
same score zero.  Hence there is generally no finite constant $L_a$ such that
\[
f(x)-f(x')\le L_a\{a_\tau(x)-a_\tau(x')\}
\]
for all pairs with $f(x)\ge f(x')$.  Thus the global calibration assumption used
in Theorem~\ref{thm:formal_regret_decomp} is not satisfied by noiseless EI, except in
the degenerate case where the threshold lies below all objective values.  The
correct deterministic-EI statement is local and threshold-set based, as in
Proposition~\ref{prop:noiseless_ei_setwise}.
\end{remark}

\begin{proposition}[Setwise calibration of noiseless EI]
\label{prop:noiseless_ei_setwise}
Let
\[
\rho(x):=f^\star-f(x),
\qquad
\mathcal U_\gamma:=\{x:\rho(x)\le \gamma\}.
\]
For a threshold gap $\eta>0$, define the noiseless EI/progress utility
\[
I_\eta(x)
:=
(f(x)-(f^\star-\eta))_+
=
(\eta-\rho(x))_+ .
\]
Then, for every $0\le \gamma<\eta$,
\[
\inf_{x\in\mathcal U_\gamma} I_\eta(x)
-
\sup_{y\notin\mathcal U_\eta} I_\eta(y)
\ge
\eta-\gamma.
\]
Moreover, if $\mathcal X^\star\neq\mathcal X$ and
\[
\Delta_f^\star
:=
f^\star-\max_{x\notin\mathcal X^\star}f(x)>0,
\]
then
\[
\inf_{x^\star\in\mathcal X^\star} I_\eta(x^\star)
-
\sup_{y\notin\mathcal X^\star} I_\eta(y)
\ge
\min\{\eta,\Delta_f^\star\}.
\]
\end{proposition}

\begin{proof}
If $x\in\mathcal U_\gamma$, then $\rho(x)\le\gamma$, so
\[
I_\eta(x)=(\eta-\rho(x))_+\ge \eta-\gamma.
\]
If $y\notin\mathcal U_\eta$, then $\rho(y)>\eta$, so
\[
I_\eta(y)=(\eta-\rho(y))_+=0.
\]
This proves the first display.

For the exact optimizer set, $I_\eta(x^\star)=\eta$ for all
$x^\star\in\mathcal X^\star$.  If $y\notin\mathcal X^\star$, then
$\rho(y)\ge\Delta_f^\star$, hence
\[
I_\eta(y)\le (\eta-\Delta_f^\star)_+.
\]
Therefore
\[
I_\eta(x^\star)-I_\eta(y)
\ge
\eta-(\eta-\Delta_f^\star)_+
=
\min\{\eta,\Delta_f^\star\}.
\]
\end{proof}

\begin{corollary}[Local regret bound for noiseless EI]
\label{cor:noiseless_ei_local_regret}
Fix $0\le\gamma<\eta$ and suppose the final recommendation is
\[
\widehat x_T\in\argmax_{x\in\mathcal D_T}\widehat a_T(x).
\]
Let $I_\eta$ be the noiseless EI utility from
Proposition~\ref{prop:noiseless_ei_setwise}.  Then
\[
r_T
\le
\eta
+
R_f\,\Pr(\mathcal D_T\cap\mathcal U_\gamma=\varnothing)
+
R_f\,
\Pr\!\left(
\sup_{z\in \mathcal D_T\cap(\mathcal U_\gamma\cup\mathcal U_\eta^c)}
|\widehat a_T(z)-I_\eta(z)|
\ge
\frac{\eta-\gamma}{2}
\right).
\]
\end{corollary}

\begin{proof}
On the event $\mathcal D_T\cap\mathcal U_\gamma\neq\varnothing$, choose
$x_\gamma\in\mathcal D_T\cap\mathcal U_\gamma$.  If the displayed local error is
less than $(\eta-\gamma)/2$, then for every
$y\in\mathcal D_T\setminus\mathcal U_\eta$,
Proposition~\ref{prop:noiseless_ei_setwise} gives
\[
\widehat a_T(x_\gamma)-\widehat a_T(y)
\ge
I_\eta(x_\gamma)-I_\eta(y)
-
2\sup_{z\in \mathcal D_T\cap(\mathcal U_\gamma\cup\mathcal U_\eta^c)}
|\widehat a_T(z)-I_\eta(z)|
>0.
\]
Thus no maximizer of $\widehat a_T$ over $\mathcal D_T$ can lie outside
$\mathcal U_\eta$, so $f^\star-f(\widehat x_T)\le\eta$.  On the complement, use
the trivial bound $R_f$ and take expectations.
\end{proof}

\subsection{LFBO-EI score learning}\label{app:lfbo_ei_learning}
For a fixed threshold $\tau$, define the feature-projected EI
\begin{equation}
\label{eq:feature_projected_ei_app}
    a^\Psi_\tau(z)
    :=
    \E[(Y-\tau)_+\mid \Psi(X)=z],
    \qquad z\in\Psi(\X),
\end{equation}
and the true EI
\[
    a_\tau(x):=\E[(Y-\tau)_+\mid X=x].
\]
The representation bias is
\begin{equation}
\label{eq:lfbo_rep_bias_app}
    b_{\Psi,\tau}
    :=
    \sup_{x\in\X}
    |a^\Psi_\tau(\Psi(x))-a_\tau(x)|.
\end{equation}
This bias is zero whenever EI is $\Psi$-measurable; for example, under additive noise it is zero if $f(x)$ is constant on the fibers of $\Psi$.

LFBO trains a classifier $c:[0,1]^{d_a}\to[\underline c,\bar c]$ with weighted loss
\begin{equation}
\label{eq:lfbo_loss_app}
    \ell_\tau(y,q)=-(y-\tau)_+\log q-\log(1-q).
\end{equation}
At the population level, the minimizer satisfies
\begin{equation}
\label{eq:lfbo_identity_app}
    c^\Psi_\tau(z)=\frac{a^\Psi_\tau(z)}{1+a^\Psi_\tau(z)}.
\end{equation}
The EI estimate is
\[
    \widehat a_t(x)=\frac{\widehat c_t(\Psi(x))}
    {1-\widehat c_t(\Psi(x))}.
\]

\begin{assumption}[LFBO-EI smooth classifier slices]
\label{ass:lfbo_ei_smooth_app}
For every $\tau\in\mathcal T$, the classifier slice $c^\Psi_\tau$ belongs to a common Besov ball
$B^{s_f}_{p,q}([0,1]^{d_a})\cap L_\infty$ with $s_f>d_a/p$. Moreover
$0<\underline c\le c^\Psi_\tau(z)\le\bar c<1$ for all $(z,\tau)$, and the weights $(Y-\tau)_+$ have a uniform sub-exponential envelope over $\tau\in\mathcal T$.
\end{assumption}

\begin{lemma}[LFBO-EI excess-risk identity]
\label{lem:lfbo_kl_identity_app}
For every measurable $c:[0,1]^{d_a}\to[\underline c,\bar c]$,
\begin{equation}
\label{eq:lfbo_kl_identity_app}
    \mathcal R_\tau(c)-\mathcal R_\tau(c^\Psi_\tau)
    =
    \E\left[
    \frac{\kl(c^\Psi_\tau(Z),c(Z))}{1-c^\Psi_\tau(Z)}
    \right],
    \qquad Z:=\Psi(X).
\end{equation}
\end{lemma}

\begin{proof}
Condition on $Z=z$ and write $a=a^\Psi_\tau(z)$,
$p=c^\Psi_\tau(z)=a/(1+a)$, and $q=c(z)$. The conditional risk is
$r_z(q)=-a\log q-\log(1-q)$. Direct algebra gives
\[
 r_z(q)-r_z(p)
 =
 a\log\frac{p}{q}
 +
 \log\frac{1-p}{1-q}
 =
 \frac{1}{1-p}\kl(p,q).
\]
Taking expectation over $Z$ proves the claim.
\end{proof}

\begin{corollary}[LFBO-EI rate with representation bias]
\label{cor:lfbo_ei_rate_app}
Under Assumptions~\ref{ass:finite_alignment_app}, \ref{ass:score_reservoir_app}, \ref{ass:ei_noise_band_app}, and~\ref{ass:lfbo_ei_smooth_app}, fix any
$0<\bar s_f<s_f-d_a/p$. The LFBO-EI estimator with the balanced clipped-ReLU sieve satisfies
\begin{equation}
\label{eq:lfbo_ei_rate_app}
    \E\|\widehat a_t-a_{\tau_t}\|_\infty
    =
    O\!\left(n_{t,a}^{-\vartheta_f}(\log n_{t,a})^{c_f}\right)
    +
    \E b_{\Psi,\tau_t},
    \qquad
    \vartheta_f=\frac{\bar s_f}{2\bar s_f+d_a}.
\end{equation}
It also satisfies an analogous high-probability envelope around the feature-projected EI
$a^\Psi_{\tau_t}\circ\Psi$, with tail $O(n_{t,a}^{-1-\kappa_f})$.
In the feature-sufficient case $b_{\Psi,\tau_t}=0$, this gives the stated score-learning rate for the true EI.
\end{corollary}

\begin{proof}
Lemma~\ref{lem:lfbo_kl_identity_app} and Pinsker's inequality give local quadratic curvature between excess risk and
$L_2(P_\Psi)$ error on the clipped interval. The LFBO loss has a sub-exponential Lipschitz envelope by Assumption~\ref{ass:lfbo_ei_smooth_app}. Therefore the generic score-learning theorem applies to the classifier with link
$G(q)=q/(1-q)$, which is Lipschitz on $[\underline c,\bar c]$. This yields the displayed rate for
$\|\widehat a_t-a^\Psi_{\tau_t}\circ\Psi\|_\infty$. The triangle inequality adds the representation bias
$b_{\Psi,\tau_t}$.
\end{proof}

\subsection{Other acquisition functions}\label{app:mean_ucb_certificates}

This appendix explains how the plug-in predictor score and the UCB score fit the
calibrated-score framework used in the main regret analysis.  We do not claim
that every UCB score is globally calibrated for every finite sample size.  Rather,
we use the calibrated ideal score
\[
    a_t(x) := f(x),
\]
which satisfies Assumption~\ref{ass:calibrated_score_app} with \(L_a=1\), and
show that the learned scores
\[
    \widehat a_t^{\rm mean}(x) := \mu_t(x),
    \qquad
    \widehat a_t^{\rm UCB}(x) := \mu_t(x)+b_t\sigma_t(x),
\]
where \(b_t=\beta_{\rm UCB}\), satisfy the required setwise ranking and
Gibbs-mass certificates once the prediction error and the relevant uncertainty
terms are sufficiently small.

\subsubsection{Setwise margins for plug-in and UCB scores}

For two disjoint sets \(A,B\subseteq\X\), define the objective margin
\begin{equation}
\label{eq:setwise_objective_margin_mean_ucb}
    \Delta_f(A,B)
    :=
    \inf_{u\in A,\;v\in B}
    \{f(u)-f(v)\}.
\end{equation}
We also define the empirical score margin
\begin{equation}
\label{eq:setwise_score_margin_mean_ucb}
    \widehat\Delta_t^s(A,B)
    :=
    \inf_{u\in A}\widehat a_t^s(u)
    -
    \sup_{v\in B}\widehat a_t^s(v),
    \qquad
    s\in\{\mathrm{mean},\mathrm{UCB}\}.
\end{equation}
Finally, write
\[
    e_t(A\cup B)
    :=
    \sup_{x\in A\cup B}|\mu_t(x)-f(x)|,
\]
and
\[
    \underline\sigma_t(A):=\inf_{u\in A}\sigma_t(u),
    \qquad
    \overline\sigma_t(B):=\sup_{v\in B}\sigma_t(v).
\]

\begin{lemma}[Setwise activation for plug-in mean and UCB]
\label{lem:mean_ucb_setwise_activation}
For any disjoint \(A,B\subseteq\X\), the plug-in mean score satisfies
\begin{equation}
\label{eq:mean_setwise_margin_bound}
    \widehat\Delta_t^{\rm mean}(A,B)
    \ge
    \Delta_f(A,B)-2e_t(A\cup B).
\end{equation}
The UCB score satisfies
\begin{equation}
\label{eq:ucb_setwise_margin_bound}
    \widehat\Delta_t^{\rm UCB}(A,B)
    \ge
    \Delta_f(A,B)
    -2e_t(A\cup B)
    +b_t\{\underline\sigma_t(A)-\overline\sigma_t(B)\}.
\end{equation}
In particular,
\begin{equation}
\label{eq:ucb_setwise_margin_simple}
    \widehat\Delta_t^{\rm UCB}(A,B)
    \ge
    \Delta_f(A,B)
    -2e_t(A\cup B)
    -b_t\overline\sigma_t(B).
\end{equation}
Consequently, the plug-in mean score ranks every point in \(A\) above every
point in \(B\) whenever
\[
    \Delta_f(A,B)>2e_t(A\cup B),
\]
and the UCB score does so whenever
\[
    \Delta_f(A,B)
    >
    2e_t(A\cup B)
    +b_t\{\overline\sigma_t(B)-\underline\sigma_t(A)\}.
\]
\end{lemma}

\begin{proof}
For the mean score, for any \(u\in A\) and \(v\in B\),
\[
    \mu_t(u)-\mu_t(v)
    =
    f(u)-f(v)
    +\{\mu_t(u)-f(u)\}
    -\{\mu_t(v)-f(v)\}.
\]
Taking the infimum over \(u\in A\) and the supremum over \(v\in B\) gives
Eq.~\eqref{eq:mean_setwise_margin_bound}.

For UCB,
\[
\begin{aligned}
    \widehat a_t^{\rm UCB}(u)-\widehat a_t^{\rm UCB}(v)
    &=
    \mu_t(u)-\mu_t(v)
    +b_t\{\sigma_t(u)-\sigma_t(v)\}  \\
    &\ge
    f(u)-f(v)
    -2e_t(A\cup B)
    +b_t\{\underline\sigma_t(A)-\overline\sigma_t(B)\}.
\end{aligned}
\]
Taking the setwise infimum proves Eq.~\eqref{eq:ucb_setwise_margin_bound}.
Since \(\underline\sigma_t(A)\ge0\), Eq.~\eqref{eq:ucb_setwise_margin_simple}
follows.
\end{proof}

\begin{remark}[A confidence-band version]
\label{rem:ucb_confidence_band_version}
If one has a valid confidence event
\[
    |\mu_t(x)-f(x)|\le c_t\sigma_t(x)
    \qquad
    \text{for all }x\in A\cup B,
\]
then the UCB margin also obeys
\[
    \widehat\Delta_t^{\rm UCB}(A,B)
    \ge
    \Delta_f(A,B)
    +(b_t-c_t)\underline\sigma_t(A)
    -(b_t+c_t)\overline\sigma_t(B).
\]
Thus, if \(b_t\ge c_t\), a sufficient condition for setwise activation is
\[
    \Delta_f(A,B)>(b_t+c_t)\overline\sigma_t(B).
\]
This form highlights the usual UCB intuition: uncertainty on promising points
does not hurt the ranking, while large uncertainty on clearly inferior points
can still delay activation.
\end{remark}

\subsubsection{Exact-optimizer certificate}

Let
\[
    \Delta_f^\star
    :=
    f^\star-\max_{x\notin\Xstar}f(x)
\]
when \(\Xstar\ne\X\).  This is positive because \(\X\) is finite.

\begin{corollary}[Exact setwise activation]
\label{cor:mean_ucb_exact_activation}
Assume \(\Xstar\ne\X\).  Let
\[
    e_t^\star
    :=
    e_t(\X)
    =
    \|\mu_t-f\|_\infty.
\]
The plug-in mean score satisfies
\[
    \min_{x^\star\in\Xstar}\mu_t(x^\star)
    >
    \max_{x\notin\Xstar}\mu_t(x)
\]
whenever
\[
    \Delta_f^\star>2e_t^\star.
\]
The UCB score satisfies
\[
    \min_{x^\star\in\Xstar}\widehat a_t^{\rm UCB}(x^\star)
    >
    \max_{x\notin\Xstar}\widehat a_t^{\rm UCB}(x)
\]
whenever
\[
    \Delta_f^\star
    >
    2e_t^\star
    +
    b_t\Bigl\{
        \sup_{x\notin\Xstar}\sigma_t(x)
        -
        \inf_{x^\star\in\Xstar}\sigma_t(x^\star)
    \Bigr\}.
\]
A simpler sufficient condition is
\[
    \Delta_f^\star
    >
    2e_t^\star
    +
    b_t\sup_{x\notin\Xstar}\sigma_t(x).
\]
Therefore, under these conditions, if the terminal candidate pool contains an
optimizer, top-\(K\) selection with \(K\ge1\) evaluates an optimizer.
\end{corollary}

\begin{proof}
Apply Lemma~\ref{lem:mean_ucb_setwise_activation} with
\(A=\Xstar\) and \(B=\X\setminus\Xstar\).  Positive setwise margin implies that
every optimizer has larger learned score than every non-optimizer.  Hence, if
the terminal candidate pool contains an optimizer, the optimizer is ranked above
all non-optimizers and must be included in the top-\(K\) set for any \(K\ge1\).
\end{proof}

\subsubsection{Near-optimal set certificate}

For \(0\le\gamma<\eta\), recall
\[
    \U_\gamma:=\{x\in\X:f^\star-f(x)\le\gamma\}.
\]
Then every \(u\in\U_\gamma\) and every \(v\notin\U_\eta\) satisfy
\[
    f(u)-f(v)\ge \eta-\gamma.
\]
Thus
\begin{equation}
\label{eq:near_opt_objective_margin}
    \Delta_f(\U_\gamma,\X\setminus\U_\eta)
    \ge
    \eta-\gamma.
\end{equation}

\begin{corollary}[Near-optimal setwise activation]
\label{cor:mean_ucb_near_opt_activation}
Fix \(0\le\gamma<\eta\).  The plug-in mean score ranks every point in
\(\U_\gamma\) above every point outside \(\U_\eta\) whenever
\[
    \eta-\gamma
    >
    2e_t(\U_\gamma\cup\U_\eta^c).
\]
The UCB score ranks every point in \(\U_\gamma\) above every point outside
\(\U_\eta\) whenever
\[
    \eta-\gamma
    >
    2e_t(\U_\gamma\cup\U_\eta^c)
    +
    b_t
    \Bigl\{
        \overline\sigma_t(\U_\eta^c)
        -
        \underline\sigma_t(\U_\gamma)
    \Bigr\}.
\]
A simpler sufficient condition is
\[
    \eta-\gamma
    >
    2e_t(\U_\gamma\cup\U_\eta^c)
    +
    b_t\overline\sigma_t(\U_\eta^c).
\]
Consequently, if \(\D_T\cap\U_\gamma\ne\varnothing\) and the corresponding
setwise margin is positive, then the final recommendation
\[
    \widehat x_T\in\argmax_{x\in\D_T}\widehat a_T(x)
\]
belongs to \(\U_\eta\).
\end{corollary}

\begin{proof}
Apply Lemma~\ref{lem:mean_ucb_setwise_activation} with
\(A=\U_\gamma\) and \(B=\X\setminus\U_\eta\), and use
Eq.~\eqref{eq:near_opt_objective_margin}.  If \(\D_T\cap\U_\gamma\ne\varnothing\),
then some evaluated \(\gamma\)-optimal point has larger learned score than every
evaluated point outside \(\U_\eta\).  Therefore a learned-score maximizer over
\(\D_T\) cannot lie outside \(\U_\eta\).
\end{proof}

\begin{remark}[Threshold progress versus acceptable recommendation]
\label{rem:threshold_progress_vs_recommendation}
Corollary~\ref{cor:mean_ucb_near_opt_activation} is sufficient for an
\(\eta\)-acceptable final recommendation.  If one wants to guarantee that
top-\(K\) selection evaluates a point in the stricter progress set
\(\U_\gamma\) whenever the terminal pool contains one, then the rejected set
should be \(B=\X\setminus\U_\gamma\), not merely \(B=\X\setminus\U_\eta\).
This requires a positive finite-domain boundary gap
\[
    \Delta_f(\U_\gamma,\X\setminus\U_\gamma)>0.
\]
If this boundary gap is zero or too small, one can use a slacked progress
certificate, for example ranking \(\U_{\gamma'}\) above
\(\X\setminus\U_\gamma\) with \(0\le\gamma'<\gamma\).
\end{remark}

\subsubsection{Gibbs target mass from setwise activation}

The preceding certificates also imply the target-mass lower bounds needed by
the search term.  The following proposition is the setwise analogue of
Proposition~\ref{prop:gibbs_target_mass_app}.

\begin{proposition}[Setwise Gibbs mass certificate]
\label{prop:setwise_gibbs_mass_certificate}
Let \(A\subseteq\X\) satisfy \(p_0(A)>0\).  Consider the Gibbs target
\[
    \widehat q_t(x)
    \propto
    p_0(x)\exp\{\beta_t\widehat a_t(x)\}.
\]
If
\begin{equation}
\label{eq:setwise_gibbs_margin_condition}
    \inf_{x\in A}\widehat a_t(x)
    -
    \sup_{y\notin A}\widehat a_t(y)
    \ge
    \delta_{t,A}
    >
    0,
\end{equation}
then
\begin{equation}
\label{eq:setwise_gibbs_mass_bound}
    \widehat q_t(A)
    \ge
    \left[
    1+
    \frac{1-p_0(A)}{p_0(A)}
    \exp\{-\beta_t\delta_{t,A}\}
    \right]^{-1}.
\end{equation}
\end{proposition}

\begin{proof}
Let
\[
    m_A:=\inf_{x\in A}\widehat a_t(x),
    \qquad
    M_{A^c}:=\sup_{y\notin A}\widehat a_t(y).
\]
By assumption, \(M_{A^c}\le m_A-\delta_{t,A}\).  Therefore
\[
    \sum_{y\notin A}p_0(y)e^{\beta_t\widehat a_t(y)}
    \le
    (1-p_0(A))e^{\beta_t(m_A-\delta_{t,A})},
\]
whereas
\[
    \sum_{x\in A}p_0(x)e^{\beta_t\widehat a_t(x)}
    \ge
    p_0(A)e^{\beta_t m_A}.
\]
Dividing the outside mass by the inside mass gives
\[
    \frac{\widehat q_t(A^c)}{\widehat q_t(A)}
    \le
    \frac{1-p_0(A)}{p_0(A)}e^{-\beta_t\delta_{t,A}},
\]
which is equivalent to Eq.~\eqref{eq:setwise_gibbs_mass_bound}.
\end{proof}

Combining Proposition~\ref{prop:setwise_gibbs_mass_certificate} with
Corollary~\ref{cor:mean_ucb_exact_activation} gives explicit mass lower bounds
for the exact optimizer set.  For the plug-in mean score, on the event
\[
    \delta_{t,\Xstar}^{\rm mean}
    :=
    \Delta_f^\star-2\|\mu_t-f\|_\infty
    >0,
\]
we have
\[
    \widehat q_t^{\rm mean}(\Xstar)
    \ge
    \left[
    1+
    \frac{1-p_0^\star}{p_0^\star}
    \exp\{-\beta_t\delta_{t,\Xstar}^{\rm mean}\}
    \right]^{-1}.
\]
For UCB, on the event
\[
    \delta_{t,\Xstar}^{\rm UCB}
    :=
    \Delta_f^\star
    -2\|\mu_t-f\|_\infty
    +b_t
    \left\{
        \inf_{x^\star\in\Xstar}\sigma_t(x^\star)
        -
        \sup_{x\notin\Xstar}\sigma_t(x)
    \right\}
    >0,
\]
we have
\[
    \widehat q_t^{\rm UCB}(\Xstar)
    \ge
    \left[
    1+
    \frac{1-p_0^\star}{p_0^\star}
    \exp\{-\beta_t\delta_{t,\Xstar}^{\rm UCB}\}
    \right]^{-1}.
\]

\subsubsection{Relation to the global score-learning term}

The main regret decomposition can also be applied directly with the calibrated
ideal score \(a_t=f\).  For the plug-in mean score,
\begin{equation}
\label{eq:mean_global_score_error}
    \|\widehat a_t^{\rm mean}-a_t\|_\infty
    =
    \|\mu_t-f\|_\infty.
\end{equation}
For the UCB score,
\begin{equation}
\label{eq:ucb_global_score_error}
    \|\widehat a_t^{\rm UCB}-a_t\|_\infty
    \le
    \|\mu_t-f\|_\infty
    +
    b_t\|\sigma_t\|_\infty.
\end{equation}
Therefore, if both the prediction error and the UCB uncertainty envelope vanish,
the UCB score becomes a vanishing perturbation of the calibrated ideal score
\(f\).  In this case the same regret decomposition applies with the score
learning term controlled by the right-hand side of
Eq.~\eqref{eq:ucb_global_score_error}.

The setwise bounds above are often less conservative than the global
\(L_\infty\) bound: for exact search it is enough that uncertainty is controlled
on the non-optimal region, and for near-optimal recommendation it is enough that
uncertainty is controlled outside the current acceptable set.  Thus UCB can be
used in the framework without changing the main calibration assumption, provided
one verifies either a global perturbation bound such as
Eq.~\eqref{eq:ucb_global_score_error} or the relevant setwise activation and
mass certificates.

\section{Terminal Mass, Target Mass, and Miss Probability}\label{app:terminal_mass}
\begin{lemma}[Target-relative RMS discrepancy implies a mass event]
\label{prop:rms_to_mass_event_main}
Under Assumption~\ref{ass:weak_terminal_mass_app}, let $\good_{t,A}$ be an event on which the frozen Gibbs target satisfies
\[
    \widehat q_t(A)\ge \ell_{t,A}
\]
for a deterministic $\ell_{t,A}\in[0,1]$. Then, for any $\omega_{t,A}>0$,
\begin{equation}
\label{eq:rms_to_mass_event_app}
    \Pbb\left(
    \good_{t,A}
    \cap
    \left\{
    \nu_t(A)<[\ell_{t,A}-\omega_{t,A}]_+
    \right\}
    \,\middle|\,\F_{t-1}
    \right)
    \le
    \frac{d_{t,A}^2}{\omega_{t,A}^2}.
\end{equation}
\end{lemma}

\begin{proof}
On $\good_{t,A}$, $\widehat q_t(A)\ge\ell_{t,A}$. If also
$\nu_t(A)<[\ell_{t,A}-\omega_{t,A}]_+$, then
\[
    \{\widehat q_t(A)-\nu_t(A)\}_+>\omega_{t,A}.
\]
Conditional Markov's inequality and Assumption~\ref{ass:weak_terminal_mass_app} prove the claim.
\end{proof}

\subsection{Target mass under a Gibbs tilt}
\label{app:target_mass_bounds}

\begin{proposition}[Gibbs target mass and top-$K$ certificate]
\label{prop:gibbs_target_mass_app}
Assume Assumptions~\ref{ass:finite_alignment_app} and~\ref{ass:calibrated_score_app}, and let $\beta_t\ge0$. If $\Xstar=\X$, then
\[
    \widehat q_t(\Xstar)=1
\]
and top-$K_t$ selection trivially evaluates an optimizer whenever $C_t\neq\varnothing$.

Otherwise, suppose $\Xstar\ne\X$ and define
\begin{equation}
\label{eq:objective_gap_app}
    \Delta_f^\star
    :=
    f^\star-\max_{x\notin\Xstar}f(x)>0 .
\end{equation}
Let the frozen Gibbs target be
\[
    \widehat q_t(x)
    \propto
    p_0(x)\exp\{\beta_t\widehat a_t(x)\}.
\]
On the good score event
\begin{equation}
\label{eq:gibbs_good_event_app}
    \good_t^\star
    :=
    \left\{
    \|\widehat a_t-a_t\|_\infty
    \le
    \frac{\Delta_f^\star}{4L_a}
    \right\},
\end{equation}
the following hold.

First, every optimizer has larger estimated score than every non-optimizer:
\begin{equation}
\label{eq:gibbs_score_gap_app}
    \min_{x^\star\in\Xstar}\widehat a_t(x^\star)
    -
    \max_{x\notin\Xstar}\widehat a_t(x)
    \ge
    \frac{\Delta_f^\star}{2L_a}.
\end{equation}
Consequently, if $C_t\cap\Xstar\neq\varnothing$, then top-$K_t$ selection with $K_t\ge1$ evaluates an optimizer.

Second, the Gibbs target mass satisfies
\begin{equation}
\label{eq:qstar_lower_app}
    \widehat q_t(\Xstar)
    \ge
    \underline q_\star^{\rm Gibbs}(\beta_t)
    :=
    \left[
    1+
    \frac{1-p_0^\star}{p_0^\star}
    \exp\!\left\{
    -\frac{\beta_t\Delta_f^\star}{2L_a}
    \right\}
    \right]^{-1}.
\end{equation}
The lower bound is nondecreasing in $\beta_t$. 
\end{proposition}

\begin{proof}
The case $\Xstar=\X$ is immediate. Assume $\Xstar\ne\X$.
For any $x^\star\in\Xstar$ and $x\notin\Xstar$, Assumption~\ref{ass:calibrated_score_app} gives
\[
    a_t(x^\star)-a_t(x)
    \ge
    \frac{f(x^\star)-f(x)}{L_a}
    \ge
    \frac{\Delta_f^\star}{L_a}.
\]
On $\good_t^\star$, the estimated score gap is therefore at least
$\Delta_f^\star/(2L_a)$, proving Eq.~\eqref{eq:gibbs_score_gap_app}. The top-$K_t$ claim follows because any optimizer in the terminal pool has larger estimated score than every non-optimizer.

For the mass bound, write
\[
    m_\star:=\min_{x^\star\in\Xstar}\widehat a_t(x^\star),
    \qquad
    M_0:=\max_{x\notin\Xstar}\widehat a_t(x).
\]
On $\good_t^\star$, $M_0\le m_\star-\Delta_f^\star/(2L_a)$. Hence
\[
    \sum_{x\notin\Xstar}p_0(x)e^{\beta_t\widehat a_t(x)}
    \le
    (1-p_0^\star)e^{\beta_t(m_\star-\Delta_f^\star/(2L_a))},
\]
while
\[
    \sum_{x^\star\in\Xstar}p_0(x^\star)e^{\beta_t\widehat a_t(x^\star)}
    \ge
    p_0^\star e^{\beta_t m_\star}.
\]
Dividing the non-optimal Gibbs weight by the optimal Gibbs weight gives Eq.~\eqref{eq:qstar_lower_app}. Monotonicity in $\beta_t$ is immediate.
\end{proof}

\subsection{Formal end-to-end theorem}\label{app:main_theory_proofs}
\paragraph{Score activation time.}
The search bound uses a ranking certificate for the learned score.  In the
exact-optimizer case, let
\[
    \Delta_f^\star
    :=
    f^\star-\max_{x\notin\Xstar} f(x) >0
\]
when \(\Xstar\ne\X\).  Define the good score event
\[
    \good_t^\star
    :=
    \left\{
    \|\widehat a_t-a_t\|_\infty
    \le
    \frac{\Delta_f^\star}{4L_a}
    \right\}.
\]
On this event, every optimizer has larger learned score than every
non-optimizer, and the Gibbs target assigns a certified mass to \(\Xstar\);
see Proposition~\ref{prop:gibbs_target_mass_app}.  The \emph{score activation
time} is the first round after which the deterministic score-learning envelope
is below this margin:
\begin{align}
\label{eq:activation_time}
    T_{\rm act}
    :=
    \inf\left\{
    t\ge1:
    \bar\varepsilon_s
    \le
    \frac{\Delta_f^\star}{4L_a}
    \ \text{for all } s\ge t
    \right\},
    \qquad
    \bar\varepsilon_t
    =
    C t^{-\vartheta_a}(\log t)^{c_a}.
\end{align}
Thus, for all \(t\ge T_{\rm act}\), Theorem~\ref{thm:generic_score_learning_rate}
gives
\[
    u_t
    :=
    \mathbb P\{(\good_t^\star)^c\}
    \le
    C t^{-1-\kappa_a}.
\]
Before \(T_{\rm act}\), the learned score may still be useful empirically, but
the proof does not rely on its ranking.  Those early rounds can be treated as
burn-in; dropping them only weakens the bound.

\begin{theorem}[Formal end-to-end miss bound]
\label{thm:formal_end_to_end_miss}
Assume Assumptions~\ref{ass:finite_alignment_app} and~\ref{ass:weak_terminal_mass_app} with $A=\Xstar$. At round $t$, suppose the algorithm evaluates $J_{t,a}$ prior-refresh candidates drawn conditionally i.i.d. from $p_0$, independently of the terminal candidate pool conditional on the past and frozen sampler information. It also evaluates the top-$K_t$ terminal candidates from $C_t$, with $K_t\ge1$.

For each round $t$, let $\good_t$ be an event such that, on $\good_t$,
\[
    \widehat q_t(\Xstar)\ge \ell_{t,\Xstar},
\]
and top-$K_t$ selection evaluates an optimizer whenever
$C_t\cap\Xstar\neq\varnothing$.

Let
\[
    u_t:=\Pbb(\good_t^c),
    \qquad
    m_t^\star:=[\ell_{t,\Xstar}-\omega_{t,\Xstar}]_+,
\]
and define
\begin{equation}
\label{eq:formal_alpha_star_app}
    \alpha_t^\star
    :=
    J_{t,a}\Lambda(p_0^\star)+N_t\Lambda(m_t^\star),
    \qquad
    \Lambda(u):=-\log(1-u).
\end{equation}
Then
\begin{align}
\label{eq:formal_miss_sum_app}
    \Pbb(\D_T\cap\Xstar=\varnothing)
    &\le
    \exp\!\left(-\sum_{t=1}^T\alpha_t^\star\right)
    +
    \sum_{t=1}^T
    \exp\!\left(-\sum_{s=t+1}^T\alpha_s^\star\right)
    \left(u_t+\frac{d_{t,\Xstar}^2}{\omega_{t,\Xstar}^2}\right).
\end{align}
Consequently, if $\good_t$ is chosen as in Proposition~\ref{prop:gibbs_target_mass_app}, and if after a score-activation time $T_{\rm act}$ the envelope in Theorem~\ref{thm:generic_score_learning_rate} satisfies
\[
    \bar\varepsilon_t\le \frac{\Delta_f^\star}{4L_a},
    \qquad t\ge T_{\rm act},
\]
then $u_t=\Pbb(\good_t^c)$ is controlled by the high-probability score-learning tail for all $t\ge T_{\rm act}$.
Combining with Theorem~\ref{thm:formal_regret_decomp} gives
\begin{equation}
\label{eq:formal_generic_regret_app}
    r_T
    \le
    O\!\left(T^{-\vartheta_a}(\log T)^{c_a}\right)
    +R_f\,\Pbb(\D_T\cap\Xstar=\varnothing),
\end{equation}
with $\vartheta_a=\bar s_a/(2\bar s_a+d_a)$.
\end{theorem}

\begin{proof}
Let $B_t:=\{\D_t\cap\Xstar=\varnothing\}$. On $B_{t-1}$, $\good_t$, and the sampler mass event
\[
    \nu_t(\Xstar)\ge m_t^\star,
\]
the prior-refresh block misses $\Xstar$ with probability
$(1-p_0^\star)^{J_{t,a}}$, and the terminal candidate pool misses $\Xstar$ with probability at most
$(1-m_t^\star)^{N_t}$. If the terminal pool hits $\Xstar$, top-$K_t$ evaluates an optimizer by the definition of $\good_t$. Therefore
\[
    \Pbb(B_t\mid B_{t-1},\good_t,\nu_t(\Xstar)\ge m_t^\star)
    \le
    \exp(-\alpha_t^\star).
\]
On the complement of the good event or the mass event we use the trivial bound $1$. Lemma~\ref{prop:rms_to_mass_event_main} gives
\[
    \Pbb\!\left(
    \good_t\cap\{\nu_t(\Xstar)<m_t^\star\}
    \right)
    \le
    \frac{d_{t,\Xstar}^2}{\omega_{t,\Xstar}^2}.
\]
Thus
\[
    \Pbb(B_t)
    \le
    \exp(-\alpha_t^\star)\Pbb(B_{t-1})
    +u_t+\frac{d_{t,\Xstar}^2}{\omega_{t,\Xstar}^2}.
\]
Unrolling the recursion proves Eq.~\eqref{eq:formal_miss_sum_app}. The regret statement follows from Theorem~\ref{thm:formal_regret_decomp} and Theorem~\ref{thm:generic_score_learning_rate}.
\end{proof}

\begin{lemma}[Geometric smoothing]\label{lem:geometric_smoothing_app}
Let $\rho\in(0,1)$ and let $e_t=O(t^{-\nu}(\log t)^c)$ for some $\nu>0$. Then
\begin{equation}
\label{eq:geometric_smoothing_app}
    \sum_{t=1}^T \rho^{T-t}e_t
    =
    O(T^{-\nu}(\log T)^c).
\end{equation}
\end{lemma}

\begin{proof}
Split the sum at $T/2$. The early part is exponentially small because $\rho^{T-t}\le\rho^{T/2}$. For the late part, $e_t\le C T^{-\nu}(\log T)^c$ and $\sum_{k=0}^\infty\rho^k=(1-\rho)^{-1}$.
\end{proof}

\begin{corollary}[Formal clean rate]\label{cor:formal_clean_rate}
Suppose that after $T_0$, $J_{t,a}\ge J_{a,\min}\ge1$, $N_t\ge N_{\min}\ge1$, $\ell_{t,\Xstar}\ge\ell_\star>0$, and $u_t\le C_ut^{-1-\kappa_a}$. If
\begin{equation}
\label{eq:formal_sampler_floor_app}
    \frac{d_{t,\Xstar}^2}{(\ell_\star/2)^2}
    \le
    C_{\rm sam}t^{-1-\kappa_{\rm sam}},
\end{equation}
then with $\omega_{t,\Xstar}=\ell_\star/2$ and
\begin{equation}
\label{eq:lambda_bo_formal_app}
    \lambda_{\rm BO}:=J_{a,\min}\Lambda(p_0^\star)+N_{\min}\Lambda(\ell_\star/2),
\end{equation}
we have
\begin{equation}
\label{eq:formal_clean_miss_app}
    \Pbb(\D_T\cap\Xstar=\varnothing)
    \le
    C_{\rm BO}e^{-\lambda_{\rm BO}(T-T_0)_+}
    +O(T^{-1-\kappa_a})
    +O(T^{-1-\kappa_{\rm sam}}).
\end{equation}
Consequently,
\begin{equation}
\label{eq:formal_clean_regret_app}
    r_T
    =
    O\!\left(
    T^{-\vartheta_a}(\log T)^{c_a}
    +T^{-1-\kappa_a}
    +T^{-1-\kappa_{\rm sam}}
    +e^{-\lambda_{\rm BO}T}
    \right),
    \qquad
    \vartheta_a=\frac{\bar s_a}{2\bar s_a+d_a}.
\end{equation}
\end{corollary}

\begin{proof}
For $t\ge T_0$, $m_t^\star\ge\ell_\star/2$ and hence $\alpha_t^\star\ge\lambda_{\rm BO}$. Apply Theorem~\ref{thm:formal_end_to_end_miss}; the weighted sums of $u_t$ and the sampler floor are controlled by Lemma~\ref{lem:geometric_smoothing_app}. The regret bound follows from Theorem~\ref{thm:formal_regret_decomp} and Theorem~\ref{thm:generic_score_learning_rate}.
\end{proof}

\begin{theorem}[Prior-refresh safety theorem]
\label{thm:safety}
Assume Assumptions~\ref{ass:finite_alignment_app},
\ref{ass:calibrated_score_app}, \ref{ass:score_reservoir_app},
\ref{ass:besov_score_app}, and \ref{ass:relu_erm_app}.  Suppose that
\(J_{t,a}\ge1\) prior-refresh labels are evaluated at each round and that the
score learner satisfies Theorem~\ref{thm:generic_score_learning_rate}.  If the
final recommendation is
\[
\widehat x_T\in\argmax_{x\in\mathcal D_T}\widehat a_T(x),
\]
then, for exact optimization with \(\gamma_T=0\),
\[
r_T
\le
2L_a\,\mathbb E\|\widehat a_T-a_T\|_\infty
+
R_f\exp\!\left(-p_0^\star\sum_{t=1}^T J_{t,a}\right).
\]
Consequently, if \(n_{T,a}\gtrsim T\), then
\[
r_T
=
O\!\left(
T^{-\vartheta_a}(\log T)^{c_a}
+
\exp\!\left(-p_0^\star\sum_{t=1}^T J_{t,a}\right)
\right).
\]
\end{theorem}

\begin{proof}
Apply Theorem~\ref{thm:formal_regret_decomp} with \(\gamma_T=0\), so that
\(\mathcal U_{\gamma_T}=\mathcal X^\star\).  The score term is controlled by
Theorem~\ref{thm:generic_score_learning_rate}.  The prior-refresh channel
contributes \(\sum_{t=1}^T J_{t,a}\) independent draws from \(p_0\), hence
\[
\Pr(\mathcal D^{\rm explr}_{T,a}\cap\mathcal X^\star=\varnothing)
=
(1-p_0^\star)^{\sum_{t=1}^T J_{t,a}}
\le
\exp\!\left(-p_0^\star\sum_{t=1}^T J_{t,a}\right).
\]
Since \(\mathcal D^{\rm explr}_{T,a}\subseteq\mathcal D_T\), this also upper
bounds
\(\Pr(\mathcal D_T\cap\mathcal X^\star=\varnothing)\).  Substitution into
Theorem~\ref{thm:formal_regret_decomp} proves the claim.
\end{proof}

\section{Sampler-Specific Instantiation}
\label{app:sampler_specific_instantiation}

This appendix proves the sampler certificates used in
Table~\ref{tab:sampler_floor_summary}. The common interface is that any
terminal TV or KL control implies the weak terminal mass condition in
Assumption~\ref{ass:weak_terminal_mass_app}.

\subsection{A terminal-divergence interface}

\begin{lemma}[Sampler mass error from terminal divergence]
\label{lem:terminal_div_to_mass_error}
Fix a round \(t\). Let \(\widehat q_t\in\Delta(\X)\) be the frozen target law
and let \(\nu_t\in\Delta(\X)\) be the terminal law of the sampler. For every
\(A\subseteq\X\),
\begin{equation}
\label{eq:terminal_div_mass_tv}
  \{\widehat q_t(A)-\nu_t(A)\}_+^2
  \le
  \TV(\widehat q_t,\nu_t)^2 .
\end{equation}
Consequently, if for a deterministic envelope \(\delta_t^2\),
\begin{equation}
\label{eq:terminal_tv_certificate_app}
  \E\!\left[
  \TV(\widehat q_t,\nu_t)^2
  \mid\F_{t-1}
  \right]
  \le
  \delta_t^2 ,
\end{equation}
then Assumption~\ref{ass:weak_terminal_mass_app} holds for every
\(A\subseteq\X\) with \(d_{t,A}^2=\delta_t^2\).

Moreover, if for a deterministic envelope \(\varepsilon_t\) either
\begin{equation}
\label{eq:forward_kl_certificate_app}
  \E\!\left[
  \KL(\widehat q_t\|\nu_t)
  \mid\F_{t-1}
  \right]
  \le
  \varepsilon_t
\end{equation}
or
\begin{equation}
\label{eq:reverse_kl_certificate_app}
  \E\!\left[
  \KL(\nu_t\|\widehat q_t)
  \mid\F_{t-1}
  \right]
  \le
  \varepsilon_t ,
\end{equation}
then Assumption~\ref{ass:weak_terminal_mass_app} holds for every
\(A\subseteq\X\) with
\begin{equation}
\label{eq:kl_to_mass_floor_app}
  d_{t,A}^2=\frac{\varepsilon_t}{2}.
\end{equation}
If the displayed envelopes are random but \(\F_{t-1}\)-measurable, any
deterministic upper envelope may be used as \(d_{t,A}^2\).
\end{lemma}

\begin{proof}
For any \(A\subseteq\X\),
\[
  |\widehat q_t(A)-\nu_t(A)|
  \le
  \sup_{B\subseteq\X}
  |\widehat q_t(B)-\nu_t(B)|
  =
  \TV(\widehat q_t,\nu_t).
\]
Since \(\{z\}_+\le |z|\), this proves
Eq.~\eqref{eq:terminal_div_mass_tv}. Taking conditional expectations gives
Eq.~\eqref{eq:terminal_tv_certificate_app} and hence
Assumption~\ref{ass:weak_terminal_mass_app} with
\(d_{t,A}^2=\delta_t^2\).

For the KL statements, Pinsker's inequality gives, using natural logarithms,
\[
    \TV(\widehat q_t,\nu_t)^2
    \le
    \frac12\KL(\widehat q_t\|\nu_t)
\]
and also
\[
    \TV(\widehat q_t,\nu_t)^2
    \le
    \frac12\KL(\nu_t\|\widehat q_t).
\]
Combining either inequality with Eq.~\eqref{eq:terminal_div_mass_tv} and taking
conditional expectations proves Eq.~\eqref{eq:kl_to_mass_floor_app}. If a KL is
infinite because of support mismatch, the implication is still valid but
vacuous; finite certificates therefore implicitly enforce the required support
condition.
\end{proof}

\begin{proposition}[Complete certificate from prior-proposal resampling]
\label{prop:certified_prior_resampling}
Fix a round $t$ and condition on the frozen score $\widehat a_t$.  Let
\[
\widehat q_t(x)
=
\frac{p_0(x)\exp\{\beta_t\widehat a_t(x)\}}
{\sum_{z\in\mathcal X}p_0(z)\exp\{\beta_t\widehat a_t(z)\}} .
\]
Assume that the frozen score has bounded oscillation,
\[
\sup_{x\in\mathcal X}\widehat a_t(x)
-
\inf_{x\in\mathcal X}\widehat a_t(x)
\le A_{\rm osc},
\qquad
0\le \beta_t\le \beta_{\max}.
\]
Draw proposal particles
$Z_{t,1},\ldots,Z_{t,M_t}\stackrel{\rm i.i.d.}{\sim}p_0$ and define
\[
w_{t,i}:=\exp\{\beta_t\widehat a_t(Z_{t,i})\},
\qquad
\nu_t^{M}(B)
:=
\frac{\sum_{i=1}^{M_t}w_{t,i}\mathbbm 1\{Z_{t,i}\in B\}}
{\sum_{i=1}^{M_t}w_{t,i}},
\quad B\subseteq\mathcal X .
\]
Return the terminal candidate pool conditionally i.i.d. from the random law
$\nu_t^{M}$.  Then, for every $A\subseteq\mathcal X$,
\[
\mathbb E\!\left[
\{\widehat q_t(A)-\nu_t^{M}(A)\}_+^2
\,\middle|\,
\mathcal F_{t-1},\widehat a_t
\right]
\le
\frac{\exp(2\beta_{\max}A_{\rm osc})}{M_t}.
\]
Thus Assumption~\ref{ass:weak_terminal_mass_app} holds with
\[
d_{t,A}^2
=
\frac{\exp(2\beta_{\max}A_{\rm osc})}{M_t}
\qquad
\text{for all } A\subseteq\mathcal X.
\]
In particular, choosing
$M_t\gtrsim \exp(2\beta_{\max}A_{\rm osc})\,t^{1+\kappa_{\rm sam}}$
gives a sampler floor of order $t^{-1-\kappa_{\rm sam}}$ in
Corollary~\ref{cor:formal_clean_rate}.
\end{proposition}

\begin{proof}
Write
\[
w(x):=\exp\{\beta_t\widehat a_t(x)\},
\qquad
Z:=\mathbb E_{X\sim p_0}[w(X)],
\qquad
\theta:=\widehat q_t(A).
\]
Let
\[
S_M:=\frac1{M_t}\sum_{i=1}^{M_t}w(Z_{t,i}),
\qquad
T_M:=\frac1{M_t}\sum_{i=1}^{M_t}w(Z_{t,i})\mathbbm 1\{Z_{t,i}\in A\}.
\]
Then $\nu_t^M(A)=T_M/S_M$ and $\theta=\mathbb E[T_M]/Z$.
Let
\[
w_{\min}:=\inf_x w(x),
\qquad
w_{\max}:=\sup_x w(x).
\]
The bounded-oscillation assumption gives
$w_{\max}/w_{\min}\le \exp(\beta_{\max}A_{\rm osc})$ and
$S_M\ge w_{\min}$ almost surely.  Moreover,
\[
T_M-\theta S_M
=
\frac1{M_t}\sum_{i=1}^{M_t}
w(Z_{t,i})\{\mathbbm 1\{Z_{t,i}\in A\}-\theta\}
\]
is an average of centered i.i.d. random variables, because
$\mathbb E_{p_0}[w(X)\{\mathbbm 1\{X\in A\}-\theta\}]=0$.
Therefore
\[
\mathbb E\{ \nu_t^M(A)-\theta\}^2
=
\mathbb E\left[
\frac{(T_M-\theta S_M)^2}{S_M^2}
\right]
\le
\frac{1}{w_{\min}^2}
\mathbb E (T_M-\theta S_M)^2 .
\]
Independence gives
\[
\mathbb E (T_M-\theta S_M)^2
=
\frac1{M_t}
\operatorname{Var}_{p_0}
\!\left(
w(X)\{\mathbbm 1\{X\in A\}-\theta\}
\right)
\le
\frac{w_{\max}^2}{M_t}.
\]
Thus
\[
\mathbb E\{ \nu_t^M(A)-\widehat q_t(A)\}^2
\le
\frac{1}{M_t}\left(\frac{w_{\max}}{w_{\min}}\right)^2
\le
\frac{\exp(2\beta_{\max}A_{\rm osc})}{M_t}.
\]
Since $\{u\}_+^2\le u^2$, the displayed bound follows.
\end{proof}

\begin{corollary}[Proposal-corrected guided sampler]
\label{cor:proposal_corrected_guided_sampler}
Fix a round $t$ and condition on the frozen score.  Let $g_t\in\Delta(\mathcal X)$
be any proposal distribution, for example the terminal law of a DiBO/GenBO-style
fine-tuned diffusion sampler, and assume $g_t(x)>0$ whenever
$p_0(x)\exp\{\beta_t\widehat a_t(x)\}>0$.  Define the correction weight
\[
W_t(x)
:=
\frac{p_0(x)\exp\{\beta_t\widehat a_t(x)\}}{g_t(x)} .
\]
Draw $Z_{t,1},\ldots,Z_{t,M_t}\stackrel{\rm i.i.d.}{\sim}g_t$, form the
self-normalized law
\[
\nu_t^{M,g}(B)
:=
\frac{\sum_{i=1}^{M_t}W_t(Z_{t,i})\mathbbm 1\{Z_{t,i}\in B\}}
{\sum_{i=1}^{M_t}W_t(Z_{t,i})},
\]
and return the terminal candidate pool conditionally i.i.d. from
$\nu_t^{M,g}$.  If
\[
0<W_{\min}\le W_t(x)\le W_{\max}<\infty
\qquad
\text{for all }x\in\mathcal X,
\]
then for every $A\subseteq\mathcal X$,
\[
\mathbb E\!\left[
\{\widehat q_t(A)-\nu_t^{M,g}(A)\}_+^2
\,\middle|\,
\mathcal F_{t-1},\widehat a_t
\right]
\le
\frac{1}{M_t}\left(\frac{W_{\max}}{W_{\min}}\right)^2 .
\]
\end{corollary}

\begin{proof}
The proof is identical to Proposition~\ref{prop:certified_prior_resampling},
with $p_0$ replaced by $g_t$ and $w$ replaced by $W_t$.
\end{proof}

\begin{remark}[How this fixes DiBO/GenBO]
If the terminal density $g_t$ or a path-space proposal density is available, the
proposal-corrected wrapper converts an arbitrary guided sampler into a certified
sampler for the same Gibbs target $\widehat q_t$.  If this density is not
available, Proposition~\ref{prop:certified_prior_resampling} still gives a
complete proof using the unconditional prior sampler $p_0$.  Thus unchanged
DiBO/GenBO remain covered by the certificate reductions in
Remark~\ref{rem:scope_dibo_genbo}, while the corrected wrapper gives a fully
proved sampler instance.
\end{remark}

\subsection{GenBO-style weighted scoring}

The following proposition makes explicit the target-matching condition needed
for GenBO-style scoring objectives. Let \(R_t:\X\to(0,\infty)\) be the positive
tilt used by the sampler and define
\begin{equation}
\label{eq:genbo_positive_tilt_target_app}
    Z_t
    :=
    \sum_{x\in\X}p_0(x)R_t(x),
    \qquad
    \widehat q_t(x)
    :=
    \frac{p_0(x)R_t(x)}{Z_t}.
\end{equation}
For the Gibbs target in the main text, \(R_t(x)=\exp\{\beta_t\widehat a_t(x)\}\).
For EI-weighted GenBO, one should instead take
\(R_t(x)\approx \widehat{\mathrm{EI}}_t(x)+\varepsilon_{\rm EI}\), equivalently
\(\widehat a_t=\log R_t/\beta_t\), if one wants the same proof to apply.

\begin{proposition}[GenBO log-score sampler floor]
\label{prop:genbo_logscore_sampler_floor}
Fix a round \(t\) and condition on \(\F_{t-1}\). Let \(\rho_t\in\Delta(\X)\)
be a proposal distribution satisfying \(\rho_t(x)>0\) whenever
\(p_0(x)R_t(x)>0\). Define the normalized importance weight
\begin{equation}
\label{eq:genbo_normalized_weight_app}
    W_t(x)
    :=
    \frac{p_0(x)R_t(x)}{Z_t\rho_t(x)} .
\end{equation}
Consider the population negative log-score objective
\begin{equation}
\label{eq:genbo_population_logscore_app}
    \mathcal L_t(q)
    :=
    -\E_{X\sim\rho_t}
    \left[
        W_t(X)\log q(X)
    \right],
    \qquad q\in\Delta(\X).
\end{equation}
Let \(\nu_t\) be the terminal distribution of the fine-tuned generator. If
\begin{equation}
\label{eq:genbo_excess_logscore_certificate_app}
  \E\!\left[
  \mathcal L_t(\nu_t)-\mathcal L_t(\widehat q_t)
  \mid \F_{t-1}
  \right]
  \le
  \mathcal E_t^{\rm GenBO}+b_{t,{\rm opt}}^2 ,
\end{equation}
then for every \(A\subseteq\X\),
\begin{equation}
\label{eq:genbo_logscore_mass_floor_app}
  \E\!\left[
  \{\widehat q_t(A)-\nu_t(A)\}_+^2
  \mid \F_{t-1}
  \right]
  \le
  \frac12
  \left(
  \mathcal E_t^{\rm GenBO}+b_{t,{\rm opt}}^2
  \right).
\end{equation}
\end{proposition}

\begin{proof}
By the definition of \(W_t\),
\[
    \rho_t(x)W_t(x)
    =
    \frac{p_0(x)R_t(x)}{Z_t}
    =
    \widehat q_t(x).
\]
Therefore
\[
    \mathcal L_t(q)
    =
    -\sum_{x\in\X}\widehat q_t(x)\log q(x).
\]
In particular,
\[
\begin{aligned}
    \mathcal L_t(q)-\mathcal L_t(\widehat q_t)
    &=
    \sum_{x\in\X}
    \widehat q_t(x)
    \log\frac{\widehat q_t(x)}{q(x)}
    \\
    &=
    \KL(\widehat q_t\|q),
\end{aligned}
\]
with the usual convention that the expression is \(+\infty\) if
\(q(x)=0\) for some \(x\) with \(\widehat q_t(x)>0\). Applying this identity
with \(q=\nu_t\), taking conditional expectations, and then using
Lemma~\ref{lem:terminal_div_to_mass_error} proves
Eq.~\eqref{eq:genbo_logscore_mass_floor_app}.
\end{proof}

\begin{proposition}[GenBO proper-score sampler floor]
\label{prop:genbo_proper_score_sampler_floor}
Fix a round \(t\) and condition on \(\F_{t-1}\). Suppose the population
GenBO-style scoring objective induces a nonnegative scoring divergence
\(D_S\) satisfying
\begin{equation}
\label{eq:proper_score_excess_identity_app}
    \mathcal L_t^S(q)-\mathcal L_t^S(\widehat q_t)
    =
    D_S(\widehat q_t,q),
\end{equation}
and suppose \(D_S\) is TV-strongly proper on the relevant model class:
there exists \(c_S>0\) such that
\begin{equation}
\label{eq:tv_strongly_proper_app}
    D_S(\widehat q_t,q)
    \ge
    c_S\,\TV(\widehat q_t,q)^2
    \qquad
    \text{for all }q\in\Delta(\X).
\end{equation}
If the terminal law \(\nu_t\) satisfies
\begin{equation}
\label{eq:proper_score_certificate_app}
  \E\!\left[
  \mathcal L_t^S(\nu_t)-\mathcal L_t^S(\widehat q_t)
  \mid \F_{t-1}
  \right]
  \le
  \mathcal E_t^{S}+b_{t,{\rm opt}}^2 ,
\end{equation}
then for every \(A\subseteq\X\),
\begin{equation}
\label{eq:proper_score_mass_floor_app}
  \E\!\left[
  \{\widehat q_t(A)-\nu_t(A)\}_+^2
  \mid \F_{t-1}
  \right]
  \le
  c_S^{-1}
  \left(
  \mathcal E_t^{S}+b_{t,{\rm opt}}^2
  \right).
\end{equation}
\end{proposition}

\begin{proof}
By Eq.~\eqref{eq:proper_score_excess_identity_app} and
Eq.~\eqref{eq:tv_strongly_proper_app},
\[
    \TV(\widehat q_t,\nu_t)^2
    \le
    c_S^{-1}
    \{\mathcal L_t^S(\nu_t)-\mathcal L_t^S(\widehat q_t)\}.
\]
Taking conditional expectations and using
Eq.~\eqref{eq:proper_score_certificate_app} gives a terminal TV certificate.
The claim then follows from Lemma~\ref{lem:terminal_div_to_mass_error}. For
the log score, \(D_S(\widehat q_t,q)=\KL(\widehat q_t\|q)\), and Pinsker's
inequality gives \(c_S=2\).
\end{proof}

\begin{proposition}[GenBO diffusion-score sampler floor via distribution learning]
\label{prop:genbo_oko_sampler_floor}
Fix a round \(t\) and condition on \(\F_{t-1}\). Suppose the GenBO fine-tuning
problem admits a roundwise diffusion distribution-learning reduction as follows.
There is a state space \(\mathcal Z_t\), possibly \(\mathcal Z_t=\X\), a
measurable decoding map \(\Pi_t:\mathcal Z_t\to\X\), a target distribution
\(\bar q_t\in\Delta(\mathcal Z_t)\), and a learned terminal distribution
\(\bar\nu_t\in\Delta(\mathcal Z_t)\) such that
\begin{equation}
\label{eq:decoder_pushforward_app}
    \widehat q_t=(\Pi_t)_\#\bar q_t,
    \qquad
    \nu_t=(\Pi_t)_\#\bar\nu_t .
\end{equation}
Assume that a diffusion distribution-learning theorem gives
\begin{equation}
\label{eq:oko_style_tv_certificate_app}
  \E\!\left[
  \TV(\bar\nu_t,\bar q_t)^2
  \mid \F_{t-1}
  \right]
  \le
  r_t+b_{t,{\rm opt}}^2 .
\end{equation}
For an Oko--Akiyama--Suzuki-style Besov density theorem, one may take
\[
    r_t
    =
    \widetilde O\!\left(
    M_{t,{\rm eff}}^{-s/(2s+d_x)}
    \right),
\]
provided the roundwise target density belongs uniformly to the required Besov
class, the forward marginals satisfy the required score-to-distribution
regularity conditions, the off-policy or weighted samples have effective sample
size \(M_{t,{\rm eff}}\), and the optimization/discretization residual is
absorbed into \(b_{t,{\rm opt}}^2\). If the imported theorem gives
\(\E[\TV(\bar\nu_t,\bar q_t)\mid\F_{t-1}]\le r_t\) instead, then
Eq.~\eqref{eq:oko_style_tv_certificate_app} still holds with the same \(r_t\),
since \(\TV\le1\).

Then for every \(A\subseteq\X\),
\begin{equation}
\label{eq:genbo_oko_mass_floor_app}
  d_{t,A}^2
  \le
  r_t+b_{t,{\rm opt}}^2 .
\end{equation}
In particular, under the Besov rate above,
\begin{equation}
\label{eq:genbo_oko_mass_floor_rate_app}
  d_{t,A}^2
  \le
  \widetilde O\!\left(
  M_{t,{\rm eff}}^{-s/(2s+d_x)}
  \right)
  +b_{t,{\rm opt}}^2 .
\end{equation}
If instead one has a score-to-KL recovery inequality of the form
\[
    \E[\KL(\widehat q_t\|\nu_t)\mid\F_{t-1}]
    \le
    C_{\rm rec}\mathcal E_t^{\rm score}+b_{t,{\rm opt}}^2 ,
\]
then
\[
    d_{t,A}^2
    \le
    \frac12
    \{C_{\rm rec}\mathcal E_t^{\rm score}+b_{t,{\rm opt}}^2\}.
\]
\end{proposition}

\begin{proof}
Total variation contracts under measurable pushforwards. Therefore
Eq.~\eqref{eq:decoder_pushforward_app} gives
\[
    \TV(\nu_t,\widehat q_t)
    =
    \TV((\Pi_t)_\#\bar\nu_t,(\Pi_t)_\#\bar q_t)
    \le
    \TV(\bar\nu_t,\bar q_t).
\]
Taking conditional expectations and using
Eq.~\eqref{eq:oko_style_tv_certificate_app} yields
\[
    \E[
    \TV(\nu_t,\widehat q_t)^2
    \mid\F_{t-1}
    ]
    \le
    r_t+b_{t,{\rm opt}}^2 .
\]
Lemma~\ref{lem:terminal_div_to_mass_error} then proves
Eq.~\eqref{eq:genbo_oko_mass_floor_app}. The Besov-rate display is obtained by
substituting the imported diffusion distribution-learning rate into \(r_t\).
The final score-to-KL statement follows directly from the KL part of
Lemma~\ref{lem:terminal_div_to_mass_error}.
\end{proof}

\subsection{DiBO/RTB-style path-law certificates}

\begin{proposition}[DiBO/RTB sampler floor from path-space KL]
\label{prop:dibo_rtb_sampler_floor}
Fix a round \(t\) and condition on \(\F_{t-1}\). Let \(\Omega_t\) be the
reverse-diffusion path space and let
\(\Pi_t:\Omega_t\to\X\) be the terminal-state or decoded-state map. Let
\(P_t^\star\) be the acquisition-tilted posterior path law and let
\(P_{\theta_t}\) be the learned RTB path law. Assume their terminal marginals
satisfy
\begin{equation}
\label{eq:path_pushforwards_app}
    \widehat q_t=(\Pi_t)_\#P_t^\star,
    \qquad
    \nu_t=(\Pi_t)_\#P_{\theta_t}.
\end{equation}
If the RTB training analysis certifies
\begin{equation}
\label{eq:rtb_forward_kl_certificate_app}
  \E\!\left[
  \KL(P_t^\star\|P_{\theta_t})
  \mid \F_{t-1}
  \right]
  \le
  \mathcal E_t^{\rm RTB}+b_{t,A,{\rm dec}}^2 ,
\end{equation}
then for every \(A\subseteq\X\),
\begin{equation}
\label{eq:rtb_forward_mass_floor_app}
  \E\!\left[
  \{\widehat q_t(A)-\nu_t(A)\}_+^2
  \mid \F_{t-1}
  \right]
  \le
  \frac12
  \left(
  \mathcal E_t^{\rm RTB}+b_{t,A,{\rm dec}}^2
  \right).
\end{equation}
The same conclusion holds if the certified path KL has the reverse orientation,
namely if
\begin{equation}
\label{eq:rtb_reverse_kl_certificate_app}
  \E\!\left[
  \KL(P_{\theta_t}\|P_t^\star)
  \mid \F_{t-1}
  \right]
  \le
  \mathcal E_t^{\rm RTB}+b_{t,A,{\rm dec}}^2 .
\end{equation}
\end{proposition}

\begin{proof}
The map \(\Pi_t\) is measurable. By data processing for KL and
Eq.~\eqref{eq:path_pushforwards_app},
\[
    \KL(\widehat q_t\|\nu_t)
    =
    \KL((\Pi_t)_\#P_t^\star\|(\Pi_t)_\#P_{\theta_t})
    \le
    \KL(P_t^\star\|P_{\theta_t}).
\]
Combining this with Eq.~\eqref{eq:rtb_forward_kl_certificate_app} and
Lemma~\ref{lem:terminal_div_to_mass_error} proves
Eq.~\eqref{eq:rtb_forward_mass_floor_app}.

If the reverse path-KL certificate
Eq.~\eqref{eq:rtb_reverse_kl_certificate_app} holds instead, then data
processing gives
\[
    \KL(\nu_t\|\widehat q_t)
    =
    \KL((\Pi_t)_\#P_{\theta_t}\|(\Pi_t)_\#P_t^\star)
    \le
    \KL(P_{\theta_t}\|P_t^\star),
\]
and the reverse-KL part of Lemma~\ref{lem:terminal_div_to_mass_error} gives the
same mass bound.
\end{proof}

\begin{remark}[RTB residual versus KL]
\label{rem:rtb_residual_vs_kl}
Proposition~\ref{prop:dibo_rtb_sampler_floor} requires
\(\mathcal E_t^{\rm RTB}\) to be a path-space KL certificate, or a certified
upper bound on such a KL. If the available quantity is instead a raw squared
log-ratio residual, for example
\[
    \mathcal R_t^{\rm RTB}
    :=
    \E_{P_{\theta_t}}
    \left[
    \left\{
    \log\frac{P_{\theta_t}(\tau)}{P_t^\star(\tau)}
    \right\}^2
    \right],
\]
then Cauchy--Schwarz only gives
\[
    \KL(P_{\theta_t}\|P_t^\star)
    \le
    \sqrt{\mathcal R_t^{\rm RTB}}.
\]
In that case the sampler floor becomes
\[
    d_{t,A}^2
    \lesssim
    \sqrt{\mathcal R_t^{\rm RTB}}
    +
    b_{t,A,{\rm dec}}^2,
\]
rather than a linear bound in \(\mathcal R_t^{\rm RTB}\).
\end{remark}

\subsection{Theory--experiment interface}
\label{sec:theory-experiment-interface}

The preceding results are best understood as a \emph{certificate interface}
rather than as a theorem tied to one particular diffusion-training loss. A
guided-diffusion BO method instantiates the theory once it supplies three
objects for each search-relevant set $A \subseteq \mathcal X$: an acquisition
or ranking certificate, a target-mass lower bound $\widehat q_t(A)\ge
\ell_{t,A}$, and a terminal sampler certificate
\[
    \mathbb E\!\left[
        \{\widehat q_t(A)-\nu_t(A)\}_+^2
        \,\middle|\, \mathcal F_{t-1}
    \right]
    \le d_{t,A}^2 .
\]
The relevant set $A$ depends on the regime: $A=\mathcal X^\star$ for exact
no-regret, $A=U_{a\eta}$ for multiplicative threshold progress, and
$A=C_{T,\eta}$ for local active score learning. This interface is deliberately
modular: the regret proof uses only $(\ell_{t,A},d_{t,A})$, not the internal
details of the sampler.

\paragraph{What is formally certified.}
The formal no-regret theorem assumes a finite aligned domain,
$p_0(\mathcal X^\star)>0$, a calibrated or setwise-calibrated acquisition
score, and a terminal sampler certificate of the form above. The
proposal-corrected resampling construction in Proposition~F.2 gives a complete
end-to-end sampler certificate for the Gibbs target
$\widehat q_t(x)\propto p_0(x)\exp\{\beta_t\widehat a_t(x)\}$. The DiBO and
GenBO rows in Table~2 are certificate reductions: they show which additional
finite-sample TV/KL/path-KL/proper-scoring guarantees would be sufficient to
convert the corresponding training procedure into a formal regret guarantee.

\paragraph{What the experiments test.}
The experiments use unmodified DiBO/GenBO-style guided samplers to test the
mechanism predicted by the theory: increasing terminal mass on the relevant
threshold set should increase the finite-pool hit exponent
$N\Lambda(\nu_t(A))$. Thus the empirical curves are intended to validate the
mass-lift explanation of the three regimes in Figure~2. A separate terminal
sampler certificate, or the proposal-corrected wrapper, turns this diagnostic
mechanism into the formal theorem.

\paragraph{Experimental regret and threshold mass.}
The theoretical regret in Eq.~(1) is defined with respect to
$f^\star=\max_{x\in\mathcal X}f(x)$. In the experiments, exhaustive
enumeration of $\mathcal X$ is not available, so we report normalized target
regret
\[
    R_t^{\rm exp}
    :=
    1-\max_{x\in D_t}\bar f(x),
    \qquad
    \bar f(x)\in[0,1].
\]
When the normalized target value $\bar f=1$ is attainable in the accessible
domain, this coincides with simple regret after normalization. Otherwise it
should be interpreted as target-gap regret. For empirical mass diagnostics we
write
\[
    S_\tau := \{x:\bar f(x)\ge \tau\},
    \qquad
    \widehat \nu_t(S_\tau)
    :=
    \frac1m\sum_{i=1}^m {\bf 1}\{\bar f(Z_{t,i})\ge \tau\},
    \quad Z_{t,i}\sim\nu_t .
\]
In the plots we use $\tau=0.8$. If $f^\star_{\rm exp}=1$, then
$S_\tau$ corresponds to the theoretical near-optimal set $U_{1-\tau}$.

\paragraph{EI versus log-EI targets.}
For positive utilities such as EI, two targets should be distinguished. A
Gibbs target with EI as the reward is
\[
    q(x)\propto p_0(x)\exp\{\beta\,{\rm EI}_t(x)\}.
\]
The threshold-mass argument in Eq.~(10), however, uses the EI-weighted target
\[
    q^{\rm EI}(x)\propto p_0(x)\bigl({\rm EI}_t(x)+\varepsilon_{\rm EI}\bigr),
\]
equivalently a Gibbs target with log-utility score
\[
    a_t(x)=\beta^{-1}\log\bigl({\rm EI}_t(x)+\varepsilon_{\rm EI}\bigr).
\]
Thus EI-based mass-lift certificates in this paper refer to the log-utility
or EI-weighted form. A sampler using $\exp\{\beta\,{\rm EI}_t(x)\}$ can still
be analyzed through the generic Gibbs-mass certificate in Proposition~D.15,
but not through the specific threshold identity in Eq.~(10).

\paragraph{Score learning and prior refresh.}
The global safety/no-regret theorem uses a labeled prior-refresh reservoir to
obtain a clean nonadaptive score-learning bound. The reported experiments
train on the adaptively collected top-$K$ labels, which is the practical
protocol used by current GDBO methods. This is covered conceptually by the
local threshold and active-learning decomposition in Appendix~G: the relevant
quantity is the labeled design mass assigned to the critical band
$C_{T,\eta}$, rather than a global $L^\infty$ score-learning constant.

\paragraph{Finite domains and crystals.}
The formal theorem assumes a finite accessible domain
$\mathcal X=\operatorname{supp}(p_0)$. This is natural for molecular graphs
under fixed atom and bond constraints. For crystal generation, the implemented
generator operates at finite numerical precision, while the underlying
physical design space has continuous degrees of freedom. The crystal
experiments should therefore be read as empirical evidence for the mass-lift
mechanism at the implemented resolution; Appendix~J.1 states the corresponding
continuous-domain limitation.

\paragraph{Cost of the certified wrapper.}
The proposal-corrected wrapper is a correctness certificate, not a claim that
importance resampling is always the most efficient implementation. Its
worst-case proposal budget scales as
\[
    M_t \gtrsim \exp(2\beta_{\max}A_{\rm osc})\,t^{1+\kappa_{\rm sam}},
\]
so it can be conservative when large $\beta$ is needed. Its role is to show
that the certificate interface is non-vacuous and to provide a fully proved
sampler instance; practical DiBO/GenBO samplers can replace it whenever they
come with terminal TV/KL/path-KL certificates.

\begin{table}[t]
\centering
\small
\caption{
How the algorithms used in the paper instantiate the certificate interface.
``Certified'' means that the row supplies a formal route to the regret theorem;
``diagnostic'' means that the experiment tests the mass-lift mechanism predicted
by the theorem.
}
\label{tab:theory-experiment-interface}
\begin{tabular}{p{0.18\linewidth} p{0.25\linewidth} p{0.25\linewidth} p{0.22\linewidth}}
\hline
Method & Acquisition route & Sampler route & Role in the paper \\
\hline
Random search
&
Prior sampling; hit rate controlled by $p_0(A)$
&
Exact samples from $p_0$
&
Certified baseline and reference mass profile
\\

DiBO + UCB / mean
&
Setwise perturbation of the ideal score $a=f$; see Appendix~D.4
&
Conditional RTB/path-KL or terminal-TV/KL route; see Table~2
&
Empirical mass-lift diagnostic for guided diffusion
\\

DiBO / GenBO + EI
&
EI-weighted or log-EI threshold certificate; see Eq.~(10) and Appendix~G.3
&
Conditional weighted-scoring or diffusion-learning route; see Appendix~F.2
&
Empirical test of threshold-progress mass lift
\\

Proposal-corrected wrapper
&
Any score satisfying the acquisition certificate
&
Self-normalized resampling certificate in Proposition~F.2 or Corollary~F.3
&
Fully certified sampler instance
\\
\hline
\end{tabular}
\end{table}

\section{A Unified Threshold View of Mass Lift, Local Sharpness, and Active Learning}
\label{app:unified_threshold_acceleration}

This appendix gives a threshold-level refinement of the mass-lift explanation in the main text.  It has three goals.  First, it clarifies why a no-filtering random-search baseline can look polynomial on a log--log regret plot, even though its hit probability for every fixed near-optimal set is exponential.  Second, it shows that EI-style thresholding gives a natural mechanism for exponential-looking acceleration: a guided sampler can give constant probability of a multiplicative threshold improvement, while random search has a progress probability that vanishes with the threshold gap.  Third, it connects the same local geometry to the active-learning term, yielding the polynomial acceleration regime once the exponential search phase reaches the score-learning floor.

Throughout this appendix let
\begin{equation}
\label{eq:unified_gap_def}
    \rho(x):=f^\star-f(x),
    \qquad
    \U_\gamma:=\{x\in\X:\rho(x)\le \gamma\},
    \qquad
    M_0(\gamma):=p_0(\U_\gamma).
\end{equation}
Thus \(M_0\) is the prior near-optimal mass profile.  It is fixed in time.  The resolution \(\gamma\) changes with the best threshold reached by the algorithm.

\subsection{No-filtering random search is governed by the inverse mass profile}
\label{app:random_inverse_mass_profile}

Consider the matched random-search baseline that evaluates \(B\) fresh samples from \(p_0\) per round and performs no Best-of-\(N\) filtering.  After \(T\) rounds it has \(n=BT\) evaluated samples
\[
    X_1,\ldots,X_n\stackrel{\rm i.i.d.}{\sim}p_0.
\]
Let
\begin{equation}
\label{eq:random_true_best_gap}
    \Gamma_n:=\min_{1\le i\le n}\rho(X_i)
    =f^\star-\max_{1\le i\le n} f(X_i)
\end{equation}
be the true best-observed regret.  The following identity is exact and requires no regularity assumption.

\begin{proposition}[Random search and the inverse mass profile]
\label{prop:random_inverse_mass_profile}
For every \(\gamma\ge0\),
\begin{equation}
\label{eq:random_tail_identity}
    \Pbb\{\Gamma_n>\gamma\}
    =
    \{1-M_0(\gamma)\}^n
    =
    \exp\{-n\Lambda(M_0(\gamma))\}.
\end{equation}
Consequently,
\begin{equation}
\label{eq:random_expectation_identity}
    \E\Gamma_n
    =
    \int_0^{R_f}\{1-M_0(\gamma)\}^n\,\dd\gamma,
\end{equation}
and, for every \(\gamma\ge0\),
\begin{equation}
\label{eq:random_inverse_mass_bound}
    \E\Gamma_n
    \le
    \gamma
    +R_f\exp\{-n\Lambda(M_0(\gamma))\}.
\end{equation}
If the final random-search recommendation is instead
\(\widehat x_n\in\argmax_{x\in\{X_1,\ldots,X_n\}}\widehat f_n(x)\), then
\begin{equation}
\label{eq:random_surrogate_recommendation_bound}
    \E[f^\star-f(\widehat x_n)]
    \le
    \inf_{\gamma\ge0}
    \left\{
        \gamma
        +2\E\|\widehat f_n-f\|_{\infty,\{X_1,\ldots,X_n\}}
        +R_f\exp\{-n\Lambda(M_0(\gamma))\}
    \right\}.
\end{equation}
\end{proposition}

\begin{proof}
The event \(\{\Gamma_n>\gamma\}\) is exactly the event that all \(n\) samples miss \(\U_\gamma\).  Since the samples are i.i.d. from \(p_0\),
\[
    \Pbb\{\Gamma_n>\gamma\}
    =\{1-p_0(\U_\gamma)\}^n,
\]
which proves Eq.~\eqref{eq:random_tail_identity}.  The expectation identity follows from the layer-cake formula for the nonnegative bounded random variable \(\Gamma_n\).  Splitting the integral at \(\gamma\) gives
\[
    \E\Gamma_n
    \le
    \gamma+R_f\Pbb\{\Gamma_n>\gamma\},
\]
which proves Eq.~\eqref{eq:random_inverse_mass_bound}.

For the surrogate recommendation, let \(x_n^\star\in\argmin_{1\le i\le n}\rho(X_i)\) be the true best observed point and set
\(\epsilon_n:=\|\widehat f_n-f\|_{\infty,\{X_1,\ldots,X_n\}}\).  Since \(\widehat x_n\) maximizes \(\widehat f_n\) over the evaluated set,
\[
    \widehat f_n(\widehat x_n)\ge \widehat f_n(x_n^\star).
\]
Hence
\[
    f(x_n^\star)-f(\widehat x_n)
    \le
    2\epsilon_n,
\]
and therefore
\[
    f^\star-f(\widehat x_n)
    \le
    \Gamma_n+2\epsilon_n.
\]
Taking expectations and using Eq.~\eqref{eq:random_inverse_mass_bound} proves Eq.~\eqref{eq:random_surrogate_recommendation_bound}.
\end{proof}

Proposition~\ref{prop:random_inverse_mass_profile} is the rate-free form of the random-search baseline.  For every fixed \(\gamma\), the hit probability is exponential in \(n\).  A polynomial-looking expected-regret curve arises because the relevant level \(\gamma\) decreases with \(n\), so the curve is controlled by the inverse of \(M_0\), not by a time-varying prior mass.

\subsection{Local sharpness and prior small-ball dimension imply polynomial-looking random search}
\label{app:local_geometry_random_polynomial}

The preceding statement needs no rate assumption.  A polynomial rate appears once the near-optimal mass profile is induced by local geometry.  This avoids directly assuming \(M_0(\gamma)\asymp \gamma^\alpha\).

Let \(\Psi:\X\to\mathcal Z\) be a representation space with metric \(d_\Psi\), and write
\begin{equation}
\label{eq:dist_to_opt_set}
    \dist_\Psi(x,\X^\star)
    :=
    \inf_{x^\star\in\X^\star}d_\Psi(\Psi(x),\Psi(x^\star)).
\end{equation}
The following assumption is meant to hold on the experimentally relevant resolution window.  In a finite domain, below the last nonzero objective gap the exact atom \(p_0(\X^\star)\) takes over, and the final asymptotic regime is exponential again.

\begin{assumption}[Local sharpness and prior small-ball regularity]
\label{ass:local_geometry_unified}
There are constants \(0\le r_{\min}<r_0\), \(c_\rho,C_\rho,c_p,C_p>0\), an objective sharpness exponent \(\chi_f>0\), and a prior local dimension \(d_0>0\) such that the following hold whenever the displayed radii lie in \([r_{\min},r_0]\):
\begin{align}
\label{eq:objective_local_sharpness}
    c_\rho\,\dist_\Psi(x,\X^\star)^{\chi_f}
    \le
    \rho(x)
    \le
    C_\rho\,\dist_\Psi(x,\X^\star)^{\chi_f},
\end{align}
and
\begin{align}
\label{eq:prior_small_ball_regular}
    c_p r^{d_0}
    \le
    p_0\{x:\dist_\Psi(x,\X^\star)\le r\}
    \le
    C_p r^{d_0}.
\end{align}
\end{assumption}

\begin{proposition}[Local geometry implies a near-optimal mass profile]
\label{prop:local_geometry_mass_profile}
Under Assumption~\ref{ass:local_geometry_unified}, define
\begin{equation}
\label{eq:alpha0_def_unified}
    \alpha_0:=\frac{d_0}{\chi_f}.
\end{equation}
For every \(\gamma\) such that both
\((\gamma/C_\rho)^{1/\chi_f}\) and \((\gamma/c_\rho)^{1/\chi_f}\) lie in \([r_{\min},r_0]\),
\begin{equation}
\label{eq:mass_profile_from_geometry}
    c_M \gamma^{\alpha_0}
    \le
    M_0(\gamma)
    \le
    C_M \gamma^{\alpha_0},
\end{equation}
where
\begin{equation}
\label{eq:mass_profile_constants}
    c_M:=c_p C_\rho^{-d_0/\chi_f},
    \qquad
    C_M:=C_p c_\rho^{-d_0/\chi_f}.
\end{equation}
Moreover, for every fixed \(a\in(0,1)\),
\begin{equation}
\label{eq:reverse_doubling_from_geometry}
    \frac{M_0(a\gamma)}{M_0(\gamma)}
    \ge
    c_D(a)
    :=
    \frac{c_p}{C_p}
    \left(\frac{c_\rho}{C_\rho}\right)^{d_0/\chi_f}
    a^{d_0/\chi_f},
\end{equation}
whenever the relevant radii lie in the local window.
\end{proposition}

\begin{proof}
If \(\dist_\Psi(x,\X^\star)\le (\gamma/C_\rho)^{1/\chi_f}\), then Eq.~\eqref{eq:objective_local_sharpness} gives \(\rho(x)\le\gamma\).  Hence
\[
    \{x:\dist_\Psi(x,\X^\star)\le (\gamma/C_\rho)^{1/\chi_f}\}
    \subseteq
    \U_\gamma.
\]
Applying the lower small-ball bound in Eq.~\eqref{eq:prior_small_ball_regular} gives the lower bound in Eq.~\eqref{eq:mass_profile_from_geometry}.  Conversely, if \(x\in\U_\gamma\), then
\[
    c_\rho\dist_\Psi(x,\X^\star)^{\chi_f}\le \rho(x)\le\gamma,
\]
so \(\dist_\Psi(x,\X^\star)\le(\gamma/c_\rho)^{1/\chi_f}\).  Applying the upper small-ball bound gives the upper bound in Eq.~\eqref{eq:mass_profile_from_geometry}.  Dividing the lower bound for \(M_0(a\gamma)\) by the upper bound for \(M_0(\gamma)\) proves Eq.~\eqref{eq:reverse_doubling_from_geometry}.
\end{proof}

\begin{corollary}[Polynomial-looking no-filtering random search]
\label{cor:random_polynomial_from_geometry}
Assume the local window is large enough that \(n^{-1/\alpha_0}\) lies in the corresponding \(\gamma\)-range.  Then the true best-observed random-search regret satisfies, up to constants depending only on the local geometry,
\begin{equation}
\label{eq:random_polynomial_from_geometry}
    \E\Gamma_n
    \asymp
    n^{-1/\alpha_0}
    =
    n^{-\chi_f/d_0}.
\end{equation}
More precisely, the upper bound follows from the lower mass estimate in Eq.~\eqref{eq:mass_profile_from_geometry}; the matching lower bound follows from the upper mass estimate as long as the chosen level remains in the local window.
\end{corollary}

\begin{proof}
For the upper bound, Eq.~\eqref{eq:random_expectation_identity} and \(\Lambda(u)\ge u\) give
\[
    \E\Gamma_n
    \le
    \int_0^{R_f}\exp\{-nM_0(\gamma)\}\,\dd\gamma.
\]
On the local window, \(M_0(\gamma)\ge c_M\gamma^{\alpha_0}\).  The contribution over the local window is bounded by
\[
    \int_0^\infty \exp\{-nc_M\gamma^{\alpha_0}\}\,\dd\gamma
    =
    C n^{-1/\alpha_0},
\]
using the change of variables \(u=nc_M\gamma^{\alpha_0}\).  The contribution above the local window is exponentially small for the budget range under consideration.

For the lower bound, take \(\gamma_n=c n^{-1/\alpha_0}\) with \(c>0\) small enough that \(nC_M\gamma_n^{\alpha_0}\le 1/2\).  Then \(M_0(\gamma_n)\le C_M\gamma_n^{\alpha_0}\le 1/(2n)\).  Hence
\[
    \Pbb\{\Gamma_n>\gamma_n\}
    =
    \{1-M_0(\gamma_n)\}^n
    \ge
    \left(1-\frac{1}{2n}\right)^n
    \ge c_0
\]
for a universal constant \(c_0>0\).  Therefore
\[
    \E\Gamma_n\ge \gamma_n\Pbb\{\Gamma_n>\gamma_n\}
    \ge c' n^{-1/\alpha_0}.
\]
\end{proof}

Thus the power law is not an independent modeling assumption.  It follows from two interpretable local facts: the objective gap scales like a power of distance to the optimum, and the prior has a local small-ball dimension near the optimum.

\subsection{EI thresholds and constant-probability multiplicative progress}
\label{app:ei_threshold_progress}

The threshold view gives a sharper explanation of the exponential-looking GDBO phase.  Suppose the current improvement threshold is
\begin{equation}
\label{eq:threshold_gap_eta}
    \tau=f^\star-\eta,
\end{equation}
so that \(\eta\) is the current threshold gap.  The set of points improving over the threshold is \(\U_\eta\).  For \(a\in(0,1)\), hitting \(\U_{a\eta}\) gives a multiplicative improvement of the threshold gap from \(\eta\) to at most \(a\eta\).

The following statements are written for the clean deterministic-improvement
utility
\[
I_\eta(x)=(f(x)-(f^\star-\eta))_+=(\eta-\rho(x))_+ .
\]
This is the noiseless population EI utility.  It is not globally calibrated in the sense of Assumption~\ref{ass:calibrated_score_app}; see Remark~\ref{rem:noiseless_ei_not_global}.  What it does provide is the setwise threshold certificate in Proposition~\ref{prop:noiseless_ei_setwise}, which is exactly the certificate needed for multiplicative progress on $\mathcal U_{a\eta}$.  With additive noise satisfying Assumption~\ref{ass:ei_noise_band_app}, the global EI calibration in Proposition~\ref{prop:ei_calibrated_app} can be used instead.

Define the ideal EI-tilted target
\begin{equation}
\label{eq:ideal_ei_target_unified}
    q_\eta^{\rm EI}(x)
    :=
    \frac{p_0(x)(\eta-\rho(x))_+}
    {Z_0(\eta)},
    \qquad
    Z_0(\eta):=\E_{X\sim p_0}[(\eta-\rho(X))_+].
\end{equation}
The progress mass for multiplicative factor \(a\in(0,1)\) is
\begin{equation}
\label{eq:ei_progress_mass_def}
    m_{\rm EI}(\eta,a)
    :=
    q_\eta^{\rm EI}(\U_{a\eta}).
\end{equation}

\begin{proposition}[EI progress mass from the mass profile]
\label{prop:ei_progress_mass}
For every \(\eta>0\) with \(Z_0(\eta)>0\),
\begin{equation}
\label{eq:ei_denominator_integral}
    Z_0(\eta)
    =
    \int_0^\eta M_0(u)\,\dd u.
\end{equation}
Moreover,
\begin{equation}
\label{eq:ei_progress_exact_formula}
    m_{\rm EI}(\eta,a)
    =
    \frac{
        \E_{p_0}[(\eta-\rho(X))_+\mathbbm{1}\{\rho(X)\le a\eta\}]
    }{
        \E_{p_0}[(\eta-\rho(X))_+]
    }
    =
    \frac{
        \int_{[0,a\eta]}(\eta-u)\,\dd M_0(u)
    }{
        \int_{[0,\eta]}(\eta-u)\,\dd M_0(u)
    }.
\end{equation}
In particular,
\begin{equation}
\label{eq:ei_progress_reverse_doubling_lower}
    m_{\rm EI}(\eta,a)
    \ge
    (1-a)\frac{M_0(a\eta)}{M_0(\eta)}.
\end{equation}
Consequently, under the reverse-doubling lower bound in Eq.~\eqref{eq:reverse_doubling_from_geometry},
\begin{equation}
\label{eq:ei_progress_constant_lower}
    m_{\rm EI}(\eta,a)
    \ge
    m_a
    :=
    (1-a)c_D(a)
    >0
\end{equation}
throughout the local window.
\end{proposition}

\begin{proof}
For any nonnegative random variable \(R\),
\[
    \E[(\eta-R)_+]
    =
    \int_0^\eta \Pbb(R\le u)\,\dd u.
\]
Applying this identity to \(R=\rho(X)\) proves Eq.~\eqref{eq:ei_denominator_integral}.  The numerator in Eq.~\eqref{eq:ei_progress_exact_formula} is the expectation of the function \((\eta-u)\mathbbm{1}\{u\le a\eta\}\) under the distribution function \(M_0(u)=\Pbb(\rho(X)\le u)\), which gives the Stieltjes integral formula.

For the lower bound, on the event \(\rho(X)\le a\eta\),
\[
    (\eta-\rho(X))_+\ge (1-a)\eta.
\]
Thus the numerator is at least \((1-a)\eta M_0(a\eta)\).  The denominator is at most \(\eta M_0(\eta)\), because \((\eta-\rho)_+\le\eta\mathbbm{1}\{\rho\le\eta\}\).  Dividing proves Eq.~\eqref{eq:ei_progress_reverse_doubling_lower}.  Combining with Eq.~\eqref{eq:reverse_doubling_from_geometry} proves Eq.~\eqref{eq:ei_progress_constant_lower}.
\end{proof}

This proposition is the threshold version of mass lift.  Random search hits \(\U_{a\eta}\) with probability \(M_0(a\eta)\), which decreases as the threshold approaches the optimum.  In contrast, the ideal EI-tilted target assigns at least constant mass to \(\U_{a\eta}\) under local reverse-doubling.  Thus EI guidance can make multiplicative threshold improvement a constant-probability event.

\subsection{Threshold contraction theorem}
\label{app:threshold_contraction_theorem}

We next state the contraction argument independently of the particular sampler.  Let
\begin{equation}
\label{eq:best_gap_process}
    \eta_t:=\min_{x\in\D_t}\rho(x)
\end{equation}
be the true best-so-far threshold gap.  The process \(\eta_t\) is nonincreasing.  Fix a target multiplicative factor \(a\in(0,1)\), a floor \(\eta_{\min}\ge0\), and define
\begin{equation}
\label{eq:threshold_excess_floor}
    \zeta_t:=(\eta_t-\eta_{\min})_+.
\end{equation}

\begin{theorem}[Constant progress probability implies exponential threshold contraction]
\label{thm:threshold_contraction}
Suppose that after a burn-in time \(T_0\), whenever \(\eta_{t-1}>\eta_{\min}\), the next evaluated batch has conditional probability at least \(p_{\rm prog}>0\) of containing a point in \(\U_{a\eta_{t-1}}\):
\begin{equation}
\label{eq:conditional_progress_probability}
    \Pbb\{\D_t\cap\U_{a\eta_{t-1}}\ne\varnothing\mid\F_{t-1}\}
    \ge
    p_{\rm prog}.
\end{equation}
Then for every \(T\ge T_0\),
\begin{equation}
\label{eq:threshold_contraction_bound}
    \E(\eta_T-\eta_{\min})_+
    \le
    (\eta_{T_0}-\eta_{\min})_+
    \{1-p_{\rm prog}(1-a)\}^{T-T_0}.
\end{equation}
In particular,
\begin{equation}
\label{eq:threshold_contraction_exponential}
    \E(\eta_T-\eta_{\min})_+
    \le
    (\eta_{T_0}-\eta_{\min})_+
    \exp\{-p_{\rm prog}(1-a)(T-T_0)\}.
\end{equation}
\end{theorem}

\begin{proof}
On the progress event \(\D_t\cap\U_{a\eta_{t-1}}\ne\varnothing\), the new best gap satisfies \(\eta_t\le a\eta_{t-1}\).  Otherwise, the best-so-far gap cannot increase, so \(\eta_t\le\eta_{t-1}\).  If \(\eta_{t-1}>\eta_{\min}\), then
\[
    (a\eta_{t-1}-\eta_{\min})_+
    \le
    a(\eta_{t-1}-\eta_{\min}),
\]
because \(a\eta_{\min}\le\eta_{\min}\).  Therefore
\[
    \E[\zeta_t\mid\F_{t-1}]
    \le
    p_{\rm prog}\,a\zeta_{t-1}
    +(1-p_{\rm prog})\zeta_{t-1}
    =
    \{1-p_{\rm prog}(1-a)\}\zeta_{t-1}.
\]
If \(\eta_{t-1}\le\eta_{\min}\), then \(\zeta_{t-1}=0\) and the same inequality is trivial.  Iterating from \(T_0\) proves Eq.~\eqref{eq:threshold_contraction_bound}, and \(1-u\le e^{-u}\) gives Eq.~\eqref{eq:threshold_contraction_exponential}.
\end{proof}

For the GDBO template, the conditional progress probability can be read from the terminal mass.  Suppose the round evaluates \(J\) prior-refresh samples and top-\(K\) points selected from \(N\) terminal guided candidates.  If, on the score-activation event, top-\(K\) selection evaluates a progress point whenever the terminal pool contains one, and if
\begin{equation}
\label{eq:guided_progress_mass_lower}
    \nu_t(\U_{a\eta_{t-1}})\ge m_{-}
\end{equation}
with \(m_->0\), then one may take
\begin{equation}
\label{eq:p_prog_bo}
    p_{\rm prog}
    \ge
    1-
    \{1-M_0(a\eta_{t-1})\}^{J}
    \{1-m_{-}\}^{N}.
\end{equation}
In particular, if \(m_-\) is bounded below uniformly over the threshold window, then \(p_{\rm prog}\) is bounded below uniformly.  For ideal EI guidance, Proposition~\ref{prop:ei_progress_mass} provides such a bound with \(m_-=m_a\), up to sampler and score-learning errors.

The corresponding threshold-level search exponents are
\begin{equation}
\label{eq:threshold_search_exponents}
    \lambda_{\rm rand}(\eta,a)
    :=
    B\Lambda(M_0(a\eta)),
    \qquad
    \lambda_{\rm BO}(\eta,a)
    :=
    J\Lambda(M_0(a\eta))+N\Lambda(m(\eta,a)),
\end{equation}
where \(m(\eta,a)\) is the certified guided mass on \(\U_{a\eta}\).  The threshold-level exponent gain is
\begin{equation}
\label{eq:threshold_mass_lift_gain}
    \Delta(\eta,a)
    :=
    \lambda_{\rm BO}(\eta,a)-\lambda_{\rm rand}(\eta,a)
    =
    N\Lambda(m(\eta,a))-K\Lambda(M_0(a\eta)),
    \qquad B=J+K.
\end{equation}
This expression separates the no-filtering random baseline from the guided candidate pool.  The true mass-lift mechanism is \(m(\eta,a)>M_0(a\eta)\); when \(N>K\), there is additionally a candidate-screening amplification from drawing more guided candidates than are evaluated.

\subsection{Active score learning as the polynomial floor}
\label{app:unified_active_learning_floor}

The preceding theorem explains the exponential-looking search phase.  It cannot continue indefinitely because the recommendation is ultimately limited by score-learning and sampler errors.  We now state a local active-learning theorem that uses the same threshold language.  This is a modular replacement for a global \(L_\infty\) score-learning term.

Let \(a_T\) be the ideal score used to rank the final evaluated set, and let \(\widehat a_T\) be its learned version.  For an accuracy level \(\eta>0\), define
\begin{equation}
\label{eq:active_min_score_unified}
    a_{T,\eta/2}^{\min}
    :=
    \min_{u\in\U_{\eta/2}}a_T(u),
\end{equation}
and the critical score band outside \(\U_\eta\),
\begin{equation}
\label{eq:critical_band_unified}
    \mathcal C_{T,\eta}
    :=
    \left\{
        x\notin\U_\eta:
        0< a_{T,\eta/2}^{\min}-a_T(x)
        \le c_{\mathcal C}\eta^{1/\chi_a}
    \right\},
\end{equation}
where \(c_{\mathcal C}>0\) and \(\chi_a>0\) are fixed.  The exponent \(\chi_a\) is a score-margin sharpness exponent.  For direct objective scores or standard calibrated EI scores, one may take \(\chi_a=1\).  Larger \(\chi_a\) describes a locally sharper score representation in which score-scale errors translate into regret as \(\eta\asymp \epsilon^{\chi_a}\).

\begin{assumption}[Local score margin]
\label{ass:local_score_margin_unified}
There exist constants \(\eta_0>0\), \(c_\Delta>0\), and \(\chi_a>0\) such that, for all \(0<\eta\le\eta_0\),
\begin{equation}
\label{eq:local_score_margin_unified}
    \inf_{u\in\U_{\eta/2},\,x\in\mathcal C_{T,\eta}}
    \{a_T(u)-a_T(x)\}
    \ge
    c_\Delta\eta^{1/\chi_a}.
\end{equation}
The convention \(\inf\varnothing=+\infty\) is used.
\end{assumption}

\begin{assumption}[Active local score learning]
\label{ass:active_local_score_learning_unified}
Let \(Q_T^{\rm act}\) be the average labeled design distribution used to train \(\widehat a_T\).  There are constants \(c_{\rm act}>0\), \(\beta_{\rm act}\ge0\), \(C_{\rm loc}>0\), and
\begin{equation}
\label{eq:theta_loc_unified}
    \vartheta_{\rm loc}:=\frac{s_{\rm loc}}{2s_{\rm loc}+d_{\rm loc}}
\end{equation}
such that the active mass event
\begin{equation}
\label{eq:active_mass_event_unified}
    \mathsf M_T(\eta)
    :=
    \left\{
        Q_T^{\rm act}(\mathcal C_{T,\eta})
        \ge
        c_{\rm act}\eta^{\beta_{\rm act}}
    \right\}
\end{equation}
has failure probability at most \(v_T^{\rm act}(\eta)\), and the local learning event
\begin{equation}
\label{eq:local_learning_event_unified}
    \mathsf L_T(\eta)
    :=
    \left\{
        \|\widehat a_T-a_T\|_{\infty,\U_{\eta/2}\cup\mathcal C_{T,\eta}}
        \le
        C_{\rm loc}\{TQ_T^{\rm act}(\mathcal C_{T,\eta})\}^{-\vartheta_{\rm loc}}
        (\log T)^{c_{\rm loc}}
    \right\}
\end{equation}
satisfies
\begin{equation}
\label{eq:local_learning_failure_unified}
    \Pbb\{\mathsf L_T(\eta)^c\cap\mathsf M_T(\eta)\}
    \le
    u_T^{\rm loc}(\eta).
\end{equation}
\end{assumption}

\begin{assumption}[Coarse screening outside the critical band]
\label{ass:coarse_screening_unified}
Define
\begin{equation}
\label{eq:coarse_screening_event_unified}
    \mathsf B_T(\eta)
    :=
    \left\{
        \D_T\cap\U_{\eta/2}=\varnothing
    \right\}
    \cup
    \left\{
        \max_{x\in\D_T\setminus(\U_\eta\cup\mathcal C_{T,\eta})}
        \widehat a_T(x)
        <
        \max_{u\in\D_T\cap\U_{\eta/2}}
        \widehat a_T(u)
    \right\}.
\end{equation}
There is an envelope \(b_T(\eta)\) such that
\begin{equation}
\label{eq:coarse_screening_failure_unified}
    \Pbb\{\mathsf B_T(\eta)^c\}
    \le
    b_T(\eta).
\end{equation}
\end{assumption}

\begin{theorem}[Local active-learning floor]
\label{thm:local_active_learning_floor_unified}
Suppose Assumptions~\ref{ass:local_score_margin_unified}, \ref{ass:active_local_score_learning_unified}, and \ref{ass:coarse_screening_unified} hold.  Let
\begin{equation}
\label{eq:search_miss_eta_unified}
    S_T(\eta):=\Pbb\{\D_T\cap\U_{\eta/2}=\varnothing\}.
\end{equation}
If \(T\ge3\), \(0<\eta\le\eta_0\), and
\begin{equation}
\label{eq:active_balance_unified}
    4C_{\rm loc}
    \{Tc_{\rm act}\eta^{\beta_{\rm act}}\}^{-\vartheta_{\rm loc}}
    (\log T)^{c_{\rm loc}}
    \le
    c_\Delta\eta^{1/\chi_a},
\end{equation}
then every final recommendation
\(\widehat x_T\in\argmax_{x\in\D_T}\widehat a_T(x)\) satisfies
\begin{equation}
\label{eq:active_floor_regret_raw}
    r_T
    \le
    \eta
    +R_f S_T(\eta)
    +R_f b_T(\eta)
    +R_f u_T^{\rm loc}(\eta)
    +R_f v_T^{\rm act}(\eta).
\end{equation}
Consequently, with
\begin{equation}
\label{eq:eta_active_unified_choice}
    \eta_T
    =
    C_\eta
    T^{-\frac{\chi_a\vartheta_{\rm loc}}
             {1+\chi_a\beta_{\rm act}\vartheta_{\rm loc}}}
    (\log T)^{\frac{\chi_a c_{\rm loc}}
             {1+\chi_a\beta_{\rm act}\vartheta_{\rm loc}}},
\end{equation}
where \(C_\eta\) is sufficiently large,
\begin{equation}
\label{eq:active_floor_rate_unified}
    r_T
    \le
    \widetilde O\!\left(
        T^{-\frac{\chi_a\vartheta_{\rm loc}}
             {1+\chi_a\beta_{\rm act}\vartheta_{\rm loc}}}
    \right)
    +R_f S_T(\eta_T)
    +R_f b_T(\eta_T)
    +R_f u_T^{\rm loc}(\eta_T)
    +R_f v_T^{\rm act}(\eta_T).
\end{equation}
In the strongest active-mass case \(\beta_{\rm act}=0\), the local learning floor is
\begin{equation}
\label{eq:active_floor_constant_mass_unified}
    \widetilde O\!\left(T^{-\chi_a\vartheta_{\rm loc}}\right).
\end{equation}
\end{theorem}

\begin{proof}
On \(\D_T\cap\U_{\eta/2}\ne\varnothing\), choose
\[
    u_\eta\in\argmax_{u\in\D_T\cap\U_{\eta/2}}\widehat a_T(u).
\]
On \(\mathsf M_T(\eta)\cap\mathsf L_T(\eta)\), the local score error
\[
    \epsilon_{T,\eta}
    :=
    \|\widehat a_T-a_T\|_{\infty,\U_{\eta/2}\cup\mathcal C_{T,\eta}}
\]
satisfies
\[
    \epsilon_{T,\eta}
    \le
    C_{\rm loc}\{Tc_{\rm act}\eta^{\beta_{\rm act}}\}^{-\vartheta_{\rm loc}}
    (\log T)^{c_{\rm loc}}.
\]
By Eq.~\eqref{eq:active_balance_unified}, \(2\epsilon_{T,\eta}\le c_\Delta\eta^{1/\chi_a}/2\).  For every \(x\in\D_T\cap\mathcal C_{T,\eta}\), Assumption~\ref{ass:local_score_margin_unified} gives
\[
    a_T(u_\eta)-a_T(x)
    \ge
    c_\Delta\eta^{1/\chi_a}.
\]
Therefore
\[
    \widehat a_T(u_\eta)-\widehat a_T(x)
    \ge
    a_T(u_\eta)-a_T(x)-2\epsilon_{T,\eta}
    >0.
\]
Thus \(u_\eta\) beats every evaluated point in the critical band.  On the coarse screening event \(\mathsf B_T(\eta)\), it also beats every evaluated point outside \(\U_\eta\cup\mathcal C_{T,\eta}\).  Hence no maximizer of \(\widehat a_T\) over \(\D_T\) can lie outside \(\U_\eta\).  The final regret is therefore at most \(\eta\) on the event
\[
    \{\D_T\cap\U_{\eta/2}\ne\varnothing\}
    \cap\mathsf M_T(\eta)
    \cap\mathsf L_T(\eta)
    \cap\mathsf B_T(\eta).
\]
On the complement, use the deterministic bound \(f^\star-f(\widehat x_T)\le R_f\).  The union bound gives Eq.~\eqref{eq:active_floor_regret_raw}.

It remains to solve Eq.~\eqref{eq:active_balance_unified}.  Up to constants, it is
\[
    T^{-\vartheta_{\rm loc}}
    \eta^{-\beta_{\rm act}\vartheta_{\rm loc}}
    (\log T)^{c_{\rm loc}}
    \lesssim
    \eta^{1/\chi_a}.
\]
Equivalently,
\[
    \eta^{1/\chi_a+\beta_{\rm act}\vartheta_{\rm loc}}
    \gtrsim
    T^{-\vartheta_{\rm loc}}(\log T)^{c_{\rm loc}}.
\]
The choice in Eq.~\eqref{eq:eta_active_unified_choice} satisfies this inequality for sufficiently large \(C_\eta\).  Substituting it into Eq.~\eqref{eq:active_floor_regret_raw} proves Eq.~\eqref{eq:active_floor_rate_unified}.
\end{proof}

\subsection{How the same local geometry supplies active mass}
\label{app:same_geometry_active_mass}

The active-learning exponent depends on \(\beta_{\rm act}\), which measures how much labeled design mass is available in the local critical region.  The same local geometry used in Proposition~\ref{prop:local_geometry_mass_profile} explains why guidance can make \(\beta_{\rm act}\) much smaller than the passive exponent.

For fixed \(0<a<b\le1\), define the threshold shell
\begin{equation}
\label{eq:threshold_shell_def}
    \mathcal A_{a,b}(\eta)
    :=
    \U_{b\eta}\setminus\U_{a\eta}.
\end{equation}
Under Assumption~\ref{ass:local_geometry_unified},
\begin{equation}
\label{eq:passive_shell_mass}
    p_0(\mathcal A_{a,b}(\eta))
    \le
    M_0(b\eta)
    \lesssim
    \eta^{d_0/\chi_f}.
\end{equation}
Thus a passive design has critical-region mass exponent \(\beta_{\rm rand}=d_0/\chi_f\) for any region contained in a fixed fractional shell.

For the ideal EI target, the corresponding shell mass is often constant.  Indeed, if the annulus has nonnegligible relative prior mass, i.e.
\begin{equation}
\label{eq:annular_lower_mass}
    M_0(b\eta)-M_0(a\eta)
    \ge
    c_{a,b}M_0(\eta)
\end{equation}
throughout the threshold window, then
\begin{equation}
\label{eq:ei_shell_constant_mass}
    q_\eta^{\rm EI}(\mathcal A_{a,b}(\eta))
    \ge
    (1-b)c_{a,b}.
\end{equation}
The annular lower bound follows from the two-sided power profile in Eq.~\eqref{eq:mass_profile_from_geometry} for any fixed \(0<a<b<1\), after adjusting constants.

\begin{proof}
For the passive bound, use Eq.~\eqref{eq:mass_profile_from_geometry}.  For the EI shell bound, on \(\mathcal A_{a,b}(\eta)\) we have \(\rho(x)\le b\eta\), so \((\eta-\rho(x))_+\ge(1-b)\eta\).  Hence the EI numerator over the shell is at least
\[
    (1-b)\eta\{M_0(b\eta)-M_0(a\eta)\}
    \ge
    (1-b)c_{a,b}\eta M_0(\eta).
\]
The denominator is at most \(\eta M_0(\eta)\).  Dividing proves Eq.~\eqref{eq:ei_shell_constant_mass}.
\end{proof}

Therefore, when the local critical band is contained in a fixed fractional threshold shell, guided threshold sampling can give \(\beta_{\rm act}=0\), while passive random sampling gives \(\beta_{\rm rand}=d_0/\chi_f\).  Substituting \(\beta_{\rm act}=0\) in Theorem~\ref{thm:local_active_learning_floor_unified} gives the active local floor \(\widetilde O(T^{-\chi_a\vartheta_{\rm loc}})\).  A passive local learner using only prior samples would instead have the slower exponent
\begin{equation}
\label{eq:passive_local_floor_comparison}
    \frac{\chi_a\vartheta_{\rm loc}}
         {1+\chi_a(d_0/\chi_f)\vartheta_{\rm loc}},
\end{equation}
provided the same local learning theorem were available under the passive design.  This is the polynomial acceleration regime.

\subsection{Unified finite-budget picture}
\label{app:unified_finite_budget_picture}

Combining the pieces gives the following interpretation.

\paragraph{Random search.}
The no-filtering random baseline is controlled by the inverse mass profile:
\begin{equation}
\label{eq:unified_random_summary}
    r_T^{\rm rand}
    \lesssim
    \inf_{\gamma\ge0}
    \left\{
        \gamma
        +2\E\|\widehat f_T-f\|_{\infty,\D_T}
        +R_f\exp[-BT\Lambda(M_0(\gamma))]
    \right\}.
\end{equation}
Under local sharpness and prior small-ball dimension, this gives the polynomial-looking law
\begin{equation}
\label{eq:unified_random_summary_poly}
    r_T^{\rm rand}
    \approx
    T^{-\chi_f/d_0}
    +2\E\|\widehat f_T-f\|_{\infty,\D_T}
\end{equation}
over the local resolution window.  The exact finite-domain exponential regime is recovered only after the resolution is below the last nonzero objective gap, where \(M_0(\gamma)=p_0(\X^\star)\).

\paragraph{GDBO search phase.}
At threshold gap \(\eta\), random search hits the multiplicative progress set \(\U_{a\eta}\) with probability \(M_0(a\eta)\), whereas EI-guided sampling can assign constant mass to \(\U_{a\eta}\).  Thus, after score activation and sampler accuracy are sufficient, Theorem~\ref{thm:threshold_contraction} gives
\begin{equation}
\label{eq:unified_exponential_phase_summary}
    \E(\eta_T-\eta_{\min})_+
    \lesssim
    \exp\{-c_{\rm prog}T\},
    \qquad
    c_{\rm prog}:=p_{\rm prog}(1-a).
\end{equation}
This is the exponential-looking acceleration window.

\paragraph{GDBO learning floor.}
The exponential phase stops when score-learning, activation, or sampler errors dominate.  Under local active learning,
\begin{equation}
\label{eq:unified_learning_floor_summary}
    r_T^{\rm BO}
    \lesssim
    \widetilde O\!\left(
        T^{-\frac{\chi_a\vartheta_{\rm loc}}
             {1+\chi_a\beta_{\rm act}\vartheta_{\rm loc}}}
    \right)
    +\text{search miss}
    +\text{sampler and activation errors}.
\end{equation}
When guidance supplies constant local critical mass, \(\beta_{\rm act}=0\), so the floor is \(\widetilde O(T^{-\chi_a\vartheta_{\rm loc}})\).  This is the polynomial acceleration regime.  If score learning remains the bottleneck and active sampling does not improve its exponent, then GDBO may only improve constants, matching the constant-factor regime observed empirically.

\section{Practical Diagnostics}\label{app:diagnostics}
\paragraph{Mass estimation.}
Given an audit sample
$Z_{t,1:m}\stackrel{\rm i.i.d.}{\sim}\nu_t$, define
\begin{equation}
\label{eq:empirical_mass_estimator_main}
    \widehat\nu_t(\U_\gamma)
    :=
    \frac1m\sum_{i=1}^m
    \mathbbm 1\{Z_{t,i}\in\U_\gamma\}.
\end{equation}
The induced pool-hit probability is $1-(1-\widehat\nu_t(\U_\gamma))^{N_t}$.

\paragraph{Phase fitting.}
A useful diagnostic model for a seed-averaged regret curve is a three-piece fit: an activation segment, a semi-log search segment $a-bt$, and a late log-log learning segment $c-\vartheta\log t$. Breakpoints can be selected by BIC with a minimum segment length. The estimated middle slope should correlate with the empirical search exponent $N_t\Lambda(\widehat\nu_t(\U_\gamma))$.

\section{Experiment}\label{sec:apd_experiments}

This section provides additional details on the experiments presented in Section~\ref{sec:experiments}. The code is available at \url{https://anonymous.4open.science/r/Diffusion-BO-E260}.

\textbf{Default setup.}
Across all methods, we use 32 initial data points and query 32 new samples per round. We use a candidate pool of 128 samples. In DiBO, candidates are first deduplicated, then ranked by the surrogate acquisition score using UCB, and the top 32 candidates are selected for oracle evaluation. In GenBO, candidates are deduplicated and the first 32 unique samples are selected. Each experiment is run for 20 rounds, giving a total query budget of 672 oracle evaluations, including the initial batch. All methods are evaluated under the same query budget, and invalid samples are counted toward the budget.

\subsection{Material Generation}

\textbf{Objectives.} 
We evaluate targeted material generation on two property objectives: (i) bulk modulus with target value $300$ GPa, and (ii) magnetic density greater than $0.2\,\mathring{\mathrm{A}}^{-3}$. Rewards are defined relative to the target, with values closer to 1 indicating better target attainment. We follow the reward calculation used by \citep{chen2025matinvent}. For each property, the reward is mapped to $[0,1]$ according to the target mode:
\[
r(y)=
\begin{cases}
\operatorname{clip}_{[0,1]}\!\left(\dfrac{y-m}{M-m}\right), & \text{ascending target}, \\[8pt]
\operatorname{clip}_{[0,1]}\!\left(\dfrac{M-y}{M-m}\right), & \text{descending target}, \\[10pt]
\operatorname{clip}_{[0,1]}\!\left(1-\dfrac{|y-t|}{s}\right),\quad s=\max(|t-m|,|M-t|), & \text{numeric target } t,
\end{cases}
\]
where $y$ is the predicted raw property value and $[m,M]$, with $M>m$, denotes the property range used for normalization. Following Chen et al. \citep{chen2025matinvent}, we use $[0,250]$ GPa for bulk modulus and $[0,0.25]\,\mathring{\mathrm{A}}^{-3}$ for magnetic density.

\textbf{Diffusion prior.}
We use the unconditional pretrained MatterGen model as the diffusion prior \citep{zeni2025mattergen}.

\textbf{Fine-tuning diffusion.}
For diffusion fine-tuning, we use the default MatterGen training pipeline with 20 epochs and learning rate $5\times 10^{-6}$. For DiBO, the RTB loss is computed over reverse-diffusion trajectories. Since tracking all 1000 MatterGen denoising steps is prohibitively expensive, we approximate the trajectory objective using 64 approximately equally spaced reverse-time states.

\textbf{Oracle.}
Following Chen et al.\citep{chen2025matinvent}, we use ALIGNN as the property oracle with the same setup \citep{choudhary2021atomistic}.

\textbf{Surrogate.}
For DiBO \citep{yun2025posterior}, we use an ensemble of five Crystal Graph Convolutional Neural Network (CGCNN) models \citep{xie2018crystal}. Each CGCNN has 3 layers with hidden dimension 128 and is trained for 100 epochs with learning rate $10^{-3}$, weight decay $10^{-6}$, and dropout $0.1$.

\textbf{VAE-based methods.}
For VAE-random and VAE-BO, we train a crystal variational autoencoder based on CDVAE \citep{xie2022crystal} on the MP-20 dataset \citep{jain2013commentary}, following the original experimental setup and code.

\textbf{Stability metric.}
We evaluate stability using the stable, unique, and novel (SUN) metric from MatterGen \citep{zeni2025mattergen}. We use the MP2020 Materials Project dataset as the reference set, together with MP2020 energy corrections. A generated structure is considered \emph{stable} if its energy above the convex hull, computed with respect to the reference dataset over the relevant chemical system and its subsystems, is at most $0.1$ eV/atom.

\subsection{Molecule Generation}

\textbf{Objectives.} We evaluate targeted molecular generation on highest occupied molecular orbital (HOMO) energy, with target value $\mathrm{HOMO}=-4.0~\mathrm{eV}$. The reward is defined as $r(y) = -|y-t|$,
where $y$ is the oracle prediction and $t$ is the target value.

\textbf{Diffusion prior.}
We use DiGress, a discrete graph diffusion model \citep{vignac2023digress}, pretrained on QM9 \citep{wu2018moleculenet}, which contains molecules with up to 9 heavy atoms. We follow the original DiGress codebase and training setup. For DiBO, the RTB objective is trajectory-based; since tracking all 500 DiGress denoising steps is computationally expensive, we approximate it using 64 equally spaced reverse-time states.

\textbf{Fine-tuning diffusion.}
For diffusion fine-tuning, we use the default DiGress training pipeline with 20 epochs and learning rate $2\times 10^{-4}$.

\textbf{Oracle.}
We use GFN2-xTB as the molecular property oracle \citep{bannwarth2019gfn2xtb}. GFN2-xTB is a semiempirical quantum chemistry method based on extended tight binding, providing efficient approximations of electronic structure and molecular properties at substantially lower cost than DFT or Hartree--Fock.

\textbf{Surrogate.}
For DiBO \citep{yun2025posterior}, we use an ensemble of five Graph Isomorphism Network (GIN) models \citep{xu2018how}. Each GIN has 2 layers with hidden dimension 64 and mean pooling, and is trained for 80 epochs with learning rate $10^{-3}$, weight decay $10^{-6}$, and dropout $0.1$.

\textbf{VAE-based methods.}
For VAE-random, VAE-BO, and COWBOYS \citep{moss2025return}, we use a chemical VAE \citep{gomez2018automatic} trained on the QM9 dataset \citep{wu2018moleculenet}, following the original setup.

\subsection{Additional experiments}
\begin{figure}
    \centering
    \includegraphics[width=1\linewidth]{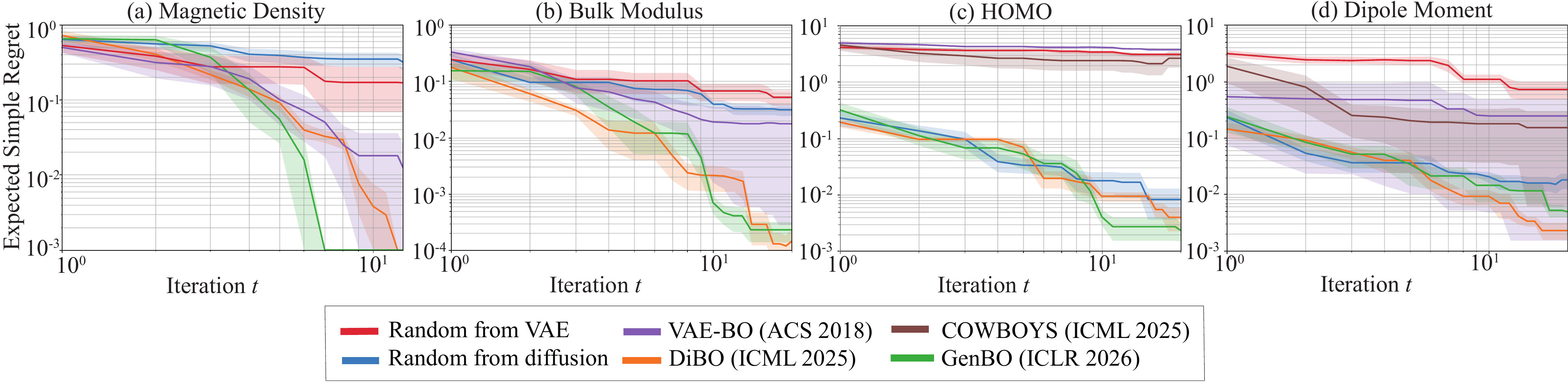}
    \caption{Comparison with VAE baselines}
    \label{fig:baselines}
\end{figure}

Additional experiments for the comparison with VAE baselines are shown in Fig.~\ref{fig:baselines}.

\section{Limitation and broader impact}\label{app:continuous_warning}
\subsection{Limitations}
\paragraph{Continuous domain}
\begin{proposition}[Fixed best-of-$N$ cannot generally certify continuous no-regret]\label{prop:continuous_pool_warning_app}
Assume $\X\subset\R^d$ and that, for all sufficiently small $\gamma$ and large $t$, $q_t(\U_\gamma)\le C_{\rm cont}\gamma^{d/\alpha}$. If the candidate-pool miss probability is required to satisfy
$\Pbb(C_t\cap\U_{\gamma_t}=\varnothing)\le c\gamma_t$ for some $\gamma_t\downarrow0$, then necessarily
\begin{equation}
\label{eq:continuous_pool_warning_app}
    N_t\gtrsim \gamma_t^{-d/\alpha}\log\frac1{\gamma_t}.
\end{equation}
\end{proposition}

\begin{proof}
The miss requirement gives $(1-q_t(\U_{\gamma_t}))^{N_t}\le c\gamma_t$, so $N_t\Lambda(q_t(\U_{\gamma_t}))\ge\log(1/(c\gamma_t))$. For small mass, $\Lambda(u)\le2u$, and $q_t(\U_{\gamma_t})\le C_{\rm cont}\gamma_t^{d/\alpha}$ gives the claim.
\end{proof}

\paragraph{Sampler-certification limitation.}
Our regret theorem requires terminal sampler certificates: after the acquisition score is frozen, the realized sampler must be close to the intended Gibbs target in set mass, TV, KL, or a path-space divergence that contracts to the terminal law.  The DiBO/GenBO instantiations in Table~\ref{tab:sampler_floor_summary} are therefore conditional reductions, not claims that the original finite-time training losses automatically imply regret.  A complete end-to-end sampler proof is obtained for the proposal-corrected resampling wrapper in Proposition~\ref{prop:certified_prior_resampling}.  Proving analogous finite-sample terminal KL/TV certificates for unmodified diffusion fine-tuning objectives remains an important direction.

\paragraph{Global score constants.}
The global score-learning bound uses a finite-support
$L_2(P_\Psi)\to L_\infty$ conversion.  Although valid on finite domains, its constant depends on $p_{\min}(P_\Psi)^{-1/2}$ and can be very large in astronomically large structured spaces.  We therefore use this global bound as a conservative safety/no-regret certificate.  The finite-budget acceleration claims are instead explained by terminal mass lift on search-relevant sets and by the local threshold results in Appendix~\ref{app:unified_threshold_acceleration}.

\subsection{Broader impact}
We theoretically analyzed the existing family of optimization algorithms. This may accelerate the algorithm development in this domain, yet the consequence of its acceleration is beyond our scope.


\end{document}